\documentclass{article}

\PassOptionsToPackage{numbers, sort, compress}{natbib}

\usepackage[preprint]{neurips_2026}

\usepackage[dvipsnames,table,xcdraw]{xcolor} 
\usepackage[most]{tcolorbox}
\usepackage[utf8]{inputenc} 
\usepackage[T1]{fontenc}    
\usepackage{hyperref}       
\usepackage{url}            
\usepackage{booktabs}       
\usepackage{amsfonts}       
\usepackage{nicefrac}       
\usepackage{microtype}      
\usepackage{amsmath}
\usepackage{amssymb}
\usepackage{mathtools}
\usepackage{amsthm}
\usepackage{thmtools,thm-restate}
\usepackage{multirow}
\usepackage[capitalize,noabbrev]{cleveref}
\usepackage{subcaption}
\usepackage{stackengine,scalerel}

\theoremstyle{plain}
\newtheorem{theorem}{Theorem}[section]

\newtheorem{lemma}[theorem]{Lemma}

\theoremstyle{definition}
\newtheorem{definition}[theorem]{Definition}

\theoremstyle{remark}
\newtheorem{remark}[theorem]{Remark}

\newcommand\DDDag{%
  \sbox0{\ddag}\stretchrel*{%
  \stackengine{-.6\ht0}{\ddag}{\ddag}{O}{c}{F}{F}{S}}{\ddag}%
}

\usepackage[textsize=tiny]{todonotes}

\usepackage{amsmath,amsfonts,bm}

\newcommand{\apropto}{\mathrel{
  \vcenter{
    \offinterlineskip
    \halign{\hfil$##$\hfil\cr
      \propto\cr
      \noalign{\kern1pt}
      \sim\cr
    }
  }
}}

\newcommand{\sg}{\operatorname{sg}}
\newcommand{\vmu}{\boldsymbol{\mu}}

\def\eqref#1{equation~\ref{#1}}

\def\1{\bm{1}}

\def\veps{{\pmb{\epsilon}}}

\def\rmI{{\mathbf{I}}}

\def\vmu{{\bm{\mu}}}

\def\va{{\bm{a}}}

\def\vg{{\bm{g}}}

\def\vs{{\bm{s}}}

\def\vu{{\bm{u}}}
\def\vv{{\bm{v}}}
\def\vw{{\bm{w}}}
\def\vx{{\bm{x}}}
\def\vy{{\bm{y}}}
\def\vz{{\bm{z}}}

\DeclareMathAlphabet{\mathsfit}{\encodingdefault}{\sfdefault}{m}{sl}
\SetMathAlphabet{\mathsfit}{bold}{\encodingdefault}{\sfdefault}{bx}{n}

\def\gD{{\mathcal{D}}}

\def\gN{{\mathcal{N}}}

\def\sP{{\mathbb{P}}}

\def\sR{{\mathbb{R}}}

\newcommand{\R}{\mathbb{R}}
\newcommand{\Z}{\mathbb{Z}}

\def\eqref#1{(\ref{#1})}

\title{Reward Score Matching: Unifying Reward-based Fine-tuning for Flow and Diffusion Models}

\author{
  Jeongjae Lee$^{*}$ \quad
  Jinho Chang$^{*}$ \quad
  Jeongsol Kim$^{\dagger}$ \quad
  Jong Chul Ye$^{\dagger}$ \\
  Graduate School of AI\\ KAIST, Korea\\
  \texttt{\{jaylee2000, jinhojsk515, jeongsol, jong.ye\}@kaist.ac.kr} \\
  $^{*}$First authors \qquad $^{\dagger}$Corresponding authors
}

\begin{document}

\maketitle

\vspace{-1em}

\begin{abstract}
\vspace{-0.5em}
    Reward-based fine-tuning steers a pretrained diffusion or flow-based generative model toward higher-reward samples while remaining close to the pretrained model.
    Although existing methods are derived from different perspectives, we show that many can be written under a common framework, which we call \emph{reward score matching} (RSM).
    Under this view, alignment becomes score matching against a value-guided target, and the main differences across methods reduce to the construction of the {\em value-guidance estimator} and the {\em effective optimization strength} across timesteps.
    This unification clarifies the bias--variance--compute tradeoffs of existing designs, and distinguishes core optimization components from auxiliary mechanisms that add complexity without clear benefit.
    Guided by this perspective, we develop simpler, more efficient redesigns across representative differentiable and black-box reward alignment tasks. Overall, RSM turns a seemingly fragmented collection of reward-based fine-tuning methods into a smaller, more interpretable, and more actionable design space.
    \vspace{-1em}
\end{abstract}
\begin{center}
\textbf{Code:} \url{https://github.com/jaylee2000/rsm} \\
\end{center}

\begin{table*}[!t]
\vspace{-1em}
\centering
\small
\caption{
\textbf{Various methods admit special cases of RSM.}
The common RSM loss $\mathcal{L}(\theta)$ is given by
$\mathbb{E}
\left[ C_1(t_i) \Big(
\|\vs^\theta_{t_i} -
(\vs^\text{ref}_{t_i}
+\gamma(t_i)\textcolor{ForestGreen}{\hat{\mathbf{\Psi}}_{t_i}})\|^2
+C_2(t_i)\|\vs^\theta_{t_i}-\vs^{\theta^\dagger}_{t_i}\|^2
\Big)\right]$.
The value gradient estimator $\textcolor{ForestGreen}{\hat{\mathbf{\Psi}}_{t_i}}$
is specified by estimator type, lookahead depth $j$, and branching width $K_i$.
The \emph{Normalized Influence Metric} $h(t_i)$ summarizes guidance strength.
}
\label{tab:unified-table}

\resizebox{0.98\linewidth}{!}{
\vspace{-1.5em}
\begin{tabular}{cccccc}
\toprule
\textbf{Method}
& \(\textcolor{ForestGreen}{\hat{\mathbf{\Psi}}_{t_i}(j, K_i)}\)
& \(\gamma(t_i)\)
& \(C_1(t_i)\)
& \(C_2(t_i)\)
& \(h(t_i)^\dagger\) \\
\midrule

\multicolumn{6}{c}{\textit{First-order estimators}} \\
\midrule

VGG-Flow$^{\ddagger}$
& $\hat{\mathbf{\Psi}}^{\mathrm{CS},1}_{t_i}[\mathbf J_{\mathrm{Tw}}]$
& \({(1-t_i)^2}\delta(t_i)\)
& \(\frac{1}{d\delta(t_i)^2}\)
& \(0\)
& \(\frac{1}{\alpha d}(1-t_i)^2\) \\

\addlinespace

ReFL
& $\hat{\mathbf{\Psi}}^{\mathrm{CS},1}_{t_i}[\mathbf I]$
& \(\frac{1}{2\sqrt{\bar\alpha_{t_i}}}\)
& \(\frac{\alpha}{\delta(t_i)^2}\)
& \(0\)
& \(\frac{1}{2\sqrt{\bar\alpha_{t_i}}\delta(t_i)}\) \\

\addlinespace

SQDF
& $\hat{\mathbf{\Psi}}^{\mathrm{LA},1}_{t_i}(i-1,1)$
& \(\tilde\gamma_{\mathrm{SQDF}}^{Nt_i}\)
& \(\frac{\alpha}{2}\frac{\Omega(t_i)^2}{\sigma_{t_i}^2}\)
& \(0\)
& \(\tilde\gamma_{\mathrm{SQDF}}^{Nt_i}
   \frac{\sigma_{t_i}w(t_i)}{2}\) \\

\addlinespace

Residual \(\nabla\)-DB$^{\ddagger}$
& $\hat{\mathbf{\Psi}}^{\mathrm{LA},1}_{t_i}(i-1,1)$
& \(\bar\alpha_{t_{i-1}}\)
& \(\frac{1}{d}\frac{\Omega(t_i)^2}{\sigma_{t_i}^4}\)
& \(\frac{(1-\bar\alpha_{t_i})\sigma_{t_i}^4}
        {\Omega(t_i)^2}\frac{w_R}{w_F}\)
& \(\frac{1}{\alpha d}\bar\alpha_{t_{i-1}}
   \frac{w(t_i)}{\sigma_{t_i}}\) \\

\addlinespace

DRaFT
& $\hat{\mathbf{\Psi}}^{\mathrm{LA},1}_{t_i}(0,1)$
& \(\frac{\delta(t_i)^2\Omega(t_i)^2}{2\sigma_{t_i}^2}\)
& \(\frac{\alpha}{\delta(t_i)^2}\)
& \(0\)
& \(\frac{\delta(t_i)\Omega(t_i)}{2}\) \\

\addlinespace

DRaFT-$K$
& $\mathbf{1}[i\le K]\,
   \hat{\mathbf{\Psi}}^{\mathrm{LA},1}_{t_i}(0,1)$
& \(\frac{\delta(t_i)^2\Omega(t_i)^2}{2\sigma_{t_i}^2}\)
& \(\frac{\alpha}{\delta(t_i)^2}\)
& \(0\)
& \(\mathbf{1}[i\le K]\frac{\delta(t_i)\Omega(t_i)}{2}\) \\

\addlinespace

Adjoint Matching
& $\hat{\mathbf{\Psi}}^{\mathrm{Adj}}_{t_i}
   \triangleq -\frac{1}{\alpha}\mathbf a_{t_i}$
& \(1\)
& \(\frac{2}{\delta(t_i)}\)
& \(0\)
& --- \\

\midrule
\multicolumn{6}{c}{\textit{Zeroth-order estimators}} \\
\midrule

REINFORCE + KL reg.
& $\hat{\mathbf{\Psi}}^{\mathrm{LA},0}_{t_i}(0,1)$
& \(1\)
& \(\frac{\alpha}{2}\frac{\Omega(t_i)^2}{\sigma_{t_i}^2}\)
& \(0\)
& \(\frac{\sigma_{t_i}w(t_i)}{2}\) \\

\addlinespace

PPO / GRPO / PCPO-base $^*$
& $\hat{\mathbf{\Psi}}^{\mathrm{LA},0}_{t_i}(0,1)$
& \(1\)
& \(\frac{\alpha}{2}\frac{\Omega(t_i)^2}{\sigma_{t_i}^2}\)
& \(\frac{r(\vx_0)}{\alpha}\)
& \(\frac{\sigma_{t_i}w(t_i)}{2}\) \\

\addlinespace

PCPO-reweight (diffusion)
& $\hat{\mathbf{\Psi}}^{\mathrm{LA},0}_{t_i}(0,1)$
& \(1\)
& \(\frac{\alpha}{2}\frac{\Omega(t_i)^2}{\sigma_{t_i}^2}\)
& \(\frac{r(\vx_0)}{\alpha}\)
& \(\frac{\sigma'_{t_i}w'(t_i)}{2}\) \\

\addlinespace

PCPO-reweight (flow)
& $\hat{\mathbf{\Psi}}^{\mathrm{LA},0}_{t_i}(0,1)$
& \(\frac{w'(t_i)}{w(t_i)}\)
& \(\frac{\alpha}{2}\frac{\Omega(t_i)^2}{\sigma_{t_i}^2}\)
& \(\gamma(t_i)\frac{r(\vx_0)}{\alpha}\)
& \(\frac{\sigma_{t_i}w'(t_i)}{2}\) \\

\addlinespace

Branch-GRPO$^{\DDDag}$
& $\hat{\mathbf{\Psi}}^{\mathrm{LA},0}_{t_i}(0,b_i)$
& \(1\)
& \(\frac{\alpha}{2}\frac{\Omega(t_i)^2}{\sigma_{t_i}^2}\)
& \(\frac{r(\vx_0)}{\alpha}\)
& \(\frac{\sigma_{t_i}w(t_i)}{2}\) \\

\addlinespace

TempFlow-GRPO
& $\hat{\mathbf{\Psi}}^{\mathrm{LA},0}_{t_i}(0,6)$
& \(\frac94\sigma_{t_i}\)
& \(\frac{\alpha}{2}\frac{\Omega(t_i)^2}{\sigma_{t_i}^2}\)
& \(\gamma(t_i)\frac{r(\vx_0)}{\alpha}\)
& \(\frac{9\sigma_{t_i}^2w(t_i)}{8}\) \\

\addlinespace

GRPO-Guard
& $\hat{\mathbf{\Psi}}^{\mathrm{LA},0}_{t_i}(0,1)$
& \(\frac{\sigma_{t_i}\Omega(t_i)}{\Delta t_i}\)
& \(\frac{\alpha}{2}\frac{\Omega(t_i)^2}{\sigma_{t_i}^2}\)
& \(0\)
& \(\frac{\sigma_{t_i}^3w(t_i)^2}
        {2(\Delta t_i)\delta(t_i)}\) \\

\bottomrule

\multicolumn{6}{l}{\footnotesize
\(\dagger\) $w(t_i)=\Omega(t_i)\delta(t_i)/\sigma_{t_i}$
(see Appendix~\ref{appendix:weights}).} \\

\multicolumn{6}{l}{\footnotesize
\(\ddagger\) The reduction becomes explicit after pruning redundant auxiliary components
(see Appendix~\ref{subsec:experiment-invalidate-auxiliary}).} \\

\multicolumn{6}{l}{\footnotesize
\(\DDDag\) For Branch-GRPO, $b_i=2$ if $i\in\mathcal B$ and $b_i=1$ otherwise.} \\

\multicolumn{6}{l}{\footnotesize
* DDPO~\cite{black2023training}, DPOK~\cite{fan2023dpok},
Flow-GRPO~\cite{liu2025flowgrpo}, DanceGRPO~\cite{xue2025dancegrpo},
and CPS~\cite{wang2025coefficientspreservingsamplingreinforcementlearning} are instances of PPO / GRPO / PCPO-base.} \\

\end{tabular}
}
\vspace{-2em}
\end{table*}

\vspace{-1em}
\section{Introduction}
\label{sec:intro}

Reward-based fine-tuning steers a pretrained diffusion or flow-based generative model toward higher-reward samples without deviating too far from the pretrained model.
Recent methods approach this problem from seemingly different perspectives, including entropy-regularized reinforcement learning (Soft RL)~\cite{black2023training,fan2023dpok,xue2025dancegrpo,liu2025flowgrpo,lee2025pcpo,he2025tempflowgrpotimingmattersgrpo,li2025mixgrpounlockingflowbasedgrpo}, optimal control~\cite{liu2025vggflow,domingo-enrich2025adjoint}, direct reward backpropagation~\cite{clark2024draft,ImageReward}, and GFlowNets~\cite{zhang2025dagdb,liu2025resnabladb,liu2025nablar2d3}.
As a result, the literature appears fragmented: across different objectives, derivations, and stabilization heuristics, it is difficult to distinguish fundamental differences from less consequential ones.

We show that these apparently different methods admit a common $L_2$ score-matching objective, which we term \emph{Reward Score Matching} (RSM; see Table~\ref{tab:unified-table}).
RSM casts reward fine-tuning as regression toward a value-guided target score: the reference score plus optimal value guidance.
Under this view, existing methods differ primarily along three design axes: how they estimate value guidance, how strongly they apply it across timesteps, and how they realize the trust region around the old policy.

The central challenge is that optimal value guidance is intractable in high dimensions.
RSM reveals that gradient-based methods construct first-order estimates of optimal value guidance, whereas gradient-free methods construct zeroth-order estimates.
These estimators exhibit different bias--variance--compute tradeoffs.
First-order estimators exhibit low variance but are severely biased at low-SNR timesteps, whereas zeroth-order estimators are unbiased yet increasingly noisy at high-SNR timesteps.
Furthermore, the temporal weighting used by existing methods can be understood as compensating for the timestep-dependent failure modes of their estimators.

This perspective reframes the design of various methods as a common optimization problem: at which timesteps should guidance be applied, and how should computation be allocated to make that guidance reliable?
It also reveals that some seemingly essential mechanisms reduce to coefficient choices or removable auxiliary components.
Guided by this analysis, we redesign representative differentiable and black-box reward fine-tuning methods.
For zeroth-order methods, we allocate estimator budget according to timestep-dependent estimator quality and make trust region realizations comparable across timesteps.
For first-order methods, we replace biased local guidance with reward gradients evaluated on clean samples, and reallocate optimization strength toward previously suppressed low-SNR timesteps.

Across four representative baselines, our simple redesigns improve reward optimization efficiency while maintaining competitive tradeoffs between reward and distribution shift.
In particular, our zeroth-order redesign reaches a GenEval~\cite{ghosh2023geneval} score of $0.97$ with an estimated \(5\times\) reduction in training cost relative to the reported TempFlow-GRPO baseline~\cite{he2025tempflowgrpotimingmattersgrpo}.
Together, these results establish RSM not only as a common language for existing methods, but also as a practical framework for designing simpler and more efficient reward fine-tuning algorithms.
\vspace{-1em}

\section{Preliminaries}
\label{sec:preliminaries}

\subsection{Flow-based Models}
\label{subsec:flow-based-models}
Let $p_0(\vx_0)$ and $p_1(\vx_1)$ denote the target and source distributions.
Flow-based models transport $\vx_1 \sim p_1$ to $\vx_0 \sim p_0$ along a probability path $p_t(\vx_t)$ induced by a flow map $\psi_t:\sR^d \to \sR^d$ on $t\in[0,1]$, with $\psi_1(\vx_1)=\vx_1$ and $\psi_0(\vx_1)=\vx_0$.
The corresponding trajectory $\vx_t=\psi_t(\vx_1)$ is governed by the flow ODE
\begin{equation}
    d\vx_t = \vv_t(\vx_t)dt,
    \label{eq:flow-ode}
\end{equation}
where $\vv_t(\vx)=\dot\psi_t(\psi_t^{-1}(\vx))$ is the velocity field.
Prior work~\cite{song2021scorebased,albergo2025stochasticinterpolants,domingo-enrich2025adjoint,liu2025flowgrpo} identified a reverse-time flow SDE that preserves the marginal distribution as:
\begin{equation}
    d\vx_t = \Big(\vv_t(\vx_t) - \frac{\tilde\sigma_t^2}{2} \nabla \log p_t(\vx_t) \Big)dt + \tilde\sigma_t d\vw.
\end{equation}
We denote the diffusion coefficient in the SDE by $\tilde \sigma_t$, and its discretized counterpart by $\sigma_t=\tilde\sigma_t \sqrt{\Delta t}$.
In conditional flow matching~\cite{lipman2023flow}, one defines a conditional flow $\psi_t(\cdot \mid \vx_0):\sR^d \to \sR^d$ with conditional velocity field $\vv_t(\vx \mid \vx_0)$, and marginal velocity
$\vv_t(\vx) = \int \vv_t(\vx \mid \vx_0)\, p(\vx_0 \mid \vx_t)\, d\vx_0$.

\subsection{Affine Conditional Flow}
\label{subsec:affine-cond-flow}
A widely used family of conditional flows is the affine conditional flow, defined as
\begin{equation}
    \psi_t(\vx_1|\vx_0) = a_t\vx_0 + b_t\vx_1 \triangleq \vx_t
    \label{eq:affine_flow}
\end{equation}
where $a_t\in\sR$ and $b_t\in\sR$ denote time-dependent variables.
From the definition of marginal velocity field and $p_1 = \gN(0,\rmI)$, the flow-ODE constructed with the affine flow is \footnote{For brevity, we write $\nabla_{\vx_t} f(\vx_t)$ as shorthand notation of $\nabla_{\vx} f(\vx) |_{\vx = \vx_t}$ for the remainder of this paper.}
\begin{equation}
\label{eq:affine-flow-expand}
    d\vx_t = \frac{\dot a_t}{a_t} \vx_t dt + \left( \frac{\dot a_t b^2_t-a_t\dot b_t b_t}{a_t} \right) \nabla_{\vx_t}\log p_t(\vx_t) dt.
\end{equation}
We establish notation generally using the affine conditional flow, since
it subsumes major generative paradigms including diffusion models (see Appendix~\ref{appendix:general-flow-de}).

For brevity, we denote the score function as $\vs_t\equiv \vs(\vx_t, t)\triangleq\nabla_{\vx_t}\log p_t(\vx_t)$ for the remainder of the paper.
Then, the transition density of the discrete-time affine flow SDE takes the general form:
\begin{align}
    p(\vx_{t_{i-1}}| \vx_{t_i}) = \gN(\kappa(t_i)\vx_{t_i} + \Omega(t_i) \vs_{t_i}, \sigma(t_i)^2 \rmI),
    \label{eq:reverse_transition}
\end{align}
where $\kappa(t_i)$ is fully determined by the definition of affine flow, and $\Omega(t_i), \sigma(t_i)$ are also affected by the stochastic noise schedule.
Proof is in Appendix~\ref{appendix:discretize}
\footnote{The network may instead parameterize a time-dependent affine transform of the score, absorbed into $\delta(t_i)$; see Appendix~\ref{appendix:delta-derivations}.
See also Appendix~\ref{appendix:omega-derivations} for derivations of $\Omega(t_i)$ for common samplers.}.

\subsection{Training-based Reward Alignment}
\label{subsec:reward-alignment-setup}

Soft RL formulates alignment as an entropy-regularized Markov decision process over discretized reverse-time transitions~\cite{Bellman1957markovian,black2023training,fan2023dpok,uehara2024understandingreinforcementlearningbasedfinetuning}.
With timesteps $0=t_0<t_1<\dots<t_N=1$, the state is $\vx_t\in\mathbb{R}^d$, the time-indexed policy is a learnable reverse transition $p^\theta_{t_i}(\vx_{t_{i-1}}\mid \vx_{t_i})$, and the reward $r$ is given only at the terminal state.
This yields the objective
\begin{equation}
    \max_\theta \mathbb{E}_{\tau\sim p^\theta} [r(\vx_0) - \alpha \sum_i \mathcal{D}_\text{KL}(p^\theta_{t_i}(\cdot|\vx_{t_i}) \| p^\text{ref}_{t_i}(\cdot|\vx_{t_i}))].
    \label{eq:soft_rl}
\end{equation}
The optimal marginal probability density is expressed with the \emph{value function} \(V_t(\vx_t)\) related to the expected return,
\begin{equation}
\label{eq:optimal-marginal-distribution}
  p_{t}^\star(\vx_{t})
  = \frac{ p_{t}^{\text{ref}}(\vx_{t})}{Z}
  \exp\Big(\frac{V_{t}(\vx_{t})}{\alpha}\Big),
\end{equation}
where $Z$ is a normalizing constant independent of $t$~\cite{uehara2024understandingreinforcementlearningbasedfinetuning}.

Soft RL subsumes Residual $\nabla$-DB~\cite{liu2025resnabladb}, SQDF~\cite{kang2026sqdf}, policy gradients, etc., and provides the canonical reward-tilted target underlying RSM's formulation.
Other reward-based fine-tuning methods arise from seemingly different derivations: VGG-Flow~\cite{liu2025vggflow} directly regresses the model velocity field toward a reward-guided target, DRaFT~\cite{clark2024draft} and ReFL~\cite{ImageReward} perform direct backpropagation through the reward, and Adjoint Matching~\cite{domingo-enrich2025adjoint} propagates reward gradients via the adjoint method.
Several recent methods further introduce timestep-dependent weighting~\cite{he2025tempflowgrpotimingmattersgrpo,lee2025pcpo,wang2025grpoguard,choi2026rethinking}, analogous to diffusion pretraining~\cite{ho2020denoising}.
Despite these differences in derivation and heuristic temporal reweighting, we show that they admit the same RSM form, extending the score-matching view induced by Soft RL beyond methods originally derived from the Soft RL framework
\footnote{DiffusionNFT~\cite{zheng2025diffusionnft} is closely related through reward-weighted regression, but does not reduce exactly to the same formulation; see Appendix~\ref{appendix:temp-diffusionnft}.}.
\section{Reward Score Matching: A Unified Framework}
\label{sec:theory}

\subsection{Optimal Guidance and Common Objective}
\label{subsec:rsm}
Eq.~\eqref{eq:optimal-marginal-distribution} implies that the optimal intermediate marginal $p_t^\star$ differs from the reference marginal $p_t^{\mathrm{ref}}$ by an exponential tilt of the value function $V_t$.
Accordingly, the optimal score takes the form
\begin{equation}
\label{eq:optimal-decomposition}
\vs_t^\star(\vx_t)=\vs_t^{\mathrm{ref}}(\vx_t)+\frac{1}{\alpha} \nabla_{\vx_t} V_t,
\end{equation}
which leads to the following first order
solver with the optimal value guidance:
\begin{restatable}[Optimal value guidance]{proposition}{gensoftval}
\label{cor:gen_soft_value}
Consider a first-order solver with transitional kernel:
\begin{equation*}
p^{\textcolor{blue}{\Psi}}_{t_i}(\vx_{t_{i-1}}\mid \vx_{t_i})
=\mathcal{N}\!\Big(\kappa(t_i)\vx_{t_i}+{\Omega(t_i)}\big(\vs_{t_i}^\text{ref}(\vx_{t_i})+
{\mathbf{\Psi}_{t_i}}\big),\;\sigma_{t_i}^2\mathbf{I}\Big).
\end{equation*}
The \textcolor{red}{\emph{optimal value guidance}} $\mathbf{\Psi}_{t_i}$ that matches the optimal transitional kernel $p^\star_{t_i}(\cdot\mid \vx_{t_i})$ is
\begin{equation}
\label{eq:alternate-optimal-psi}
\textcolor{red}{
\mathbf{\Psi}^\star_{t_i}}(\vx_{t_i})
\triangleq
\frac{1}{\alpha}\nabla_{\vx_{t_i}}V_{t_i}(\vx_{t_i})
=
\frac{\sigma_{t_i}^2}{\alpha \Omega(t_i)}\;
\mathbb{E}_{\vx_{t_{i-1}}\sim p^\star_{t_i}(\cdot\mid \vx_{t_i})}
\!\left[\nabla_{\vx_{t_{i-1}}}V_{t_{i-1}}(\vx_{t_{i-1}})\right].
\end{equation}
\end{restatable}

The {{optimal value guidance}} is generally intractable due to the integration with respect to the posterior.
Therefore, practical methods use an \textcolor{ForestGreen}
{\emph{optimal value guidance estimate}} $\textcolor{ForestGreen}{\hat{\mathbf{\Psi}}_{t_i}(\vx_{t_i})} \approx \textcolor{red}{\mathbf{\Psi}_{t_i}^\star(\vx_{t_i})}$ (see Appendix~\ref{appendix:hard-soft-value} for details) and often apply a heuristic temporal weight $\gamma(t_i)$, yielding the \textcolor{blue}{\emph{effective guidance}} $\textcolor{blue}{\mathbf{\Psi}_{t_i}(\vx_{t_i})} \triangleq \gamma(t_i)\textcolor{ForestGreen}{\hat{\mathbf{\Psi}}_{t_i}(\vx_{t_i})}$.

Since score networks are trained successfully by regression to target scores, the score function is parameterized by a score network and trained to satisfy $\vs^\theta_t \approx \vs_t$.
In particular, Theorem~\ref{theorem:rsm} shows that existing methods can be unified by the regression problem of $\vs_t^\theta$ toward the optimal value guidance estimate with $\textcolor{blue}{\mathbf{\Psi}_{t_i}}$, yielding a common RSM objective.
See Appendix~\ref{appendix:deriving-psi} for proofs.

\begin{restatable}[Reward Score Matching]{theorem}{rsm}
\label{theorem:rsm}
    The methods summarized in Table~\ref{tab:unified-table}, including Soft RL methods,
    VGG-Flow, DRaFT, DRaFT-$K$, ReFL, and Adjoint Matching, admit the common loss
    \begin{equation}
    \begin{aligned}
        \mathcal{L}(\theta) = \mathbb{E}_{t_i, \vx_{t_i}, \veps}
        &\bigg[ C_1(t_i) \Big( \|\vs^\theta_{t_i} - (\vs^\text{ref}_{t_i} + \textcolor{blue}{\mathbf{\Psi}_{t_i}})\|^2
        +C_2(t_i) \|\vs^\theta_{t_i} - \vs^{\theta^\dagger}_{t_i}  \|^2 \Big) \bigg]
    \end{aligned}
    \label{eq:master_equation}
    \end{equation}
    where $C_1, C_2$ are weights, and $\vs^\theta, \vs^{\theta^\dagger}, \vs^{\text{ref}}$ are the current, previous, and reference score functions
    \footnote{When $C_2(t)=0$, Eq.~\eqref{eq:master_equation} reduces to entropy-regularized optimal control (Appendix~\ref{appendix:optimal_control}).}.
\end{restatable}

\begin{remark}
In the first on-policy update\footnote{Throughout this paper, we use the relaxed notion of \emph{on-policy} common in online RL algorithms, where data is collected under the old policy \(\theta^\dagger\) and reused for one or more updates (e.g., DDPO~\cite{black2023training}). When we need to refer specifically to the case \(\theta=\theta^\dagger\), we distinguish it explicitly, e.g., by referring to it as the \emph{first on-policy update} or \emph{strictly on-policy}.}, the old-policy anchor is inactive.
In presence of PPO clipping, the update is suppressed by setting $\textcolor{blue}{\mathbf{\Psi}_{t_i}} \leftarrow \mathbf{0}$ and $C_2(t_i)\leftarrow 0$.
\end{remark}

The gradient of Eq.~\eqref{eq:master_equation} reveals update dynamics.
We derive a \emph{Canonical Gradient Form}:
\begin{align}
\nabla_\theta \mathcal{L}(\theta)
&=
2\mathbb{E}[\mathbf{J}_\theta(\vs_{t_i}^\theta)^\top \mathcal{G}(\vx_{t_i})]
\label{eq:gradient}
\\[-1.5em]
\mathcal{G}(\vx_{t_i})
&=
C_1(t_i)\Big(
\textcolor{blue}{-\mathbf{\Psi}_{t_i}(\vx_{t_i})}
+ \overbrace{(\vs^\theta_{t_i} - \vs^\text{ref}_{t_i})}^{\mathclap{\text{KL Reg.}}}
+ C_2(t_i)\overbrace{(\vs^\theta_{t_i} - \vs^{\theta^\dagger}_{t_i})}^{\mathclap{\text{Trust Region}}}
\Big),
\label{eq:canonical-gradient-form}
\end{align}
where $\mathbf{J}_\theta$ denotes the Jacobian.
This decomposition highlights three competing forces: \emph{\textcolor{blue}{Effective Guidance}} steers generation toward high rewards, while \emph{KL Regularization} and the \emph{Trust Region} term prevent model collapse by anchoring the model to the reference and previous policies, respectively.
Section~\ref{subsec:rsm} is written in the entropy-regularized setting for generality, but Eq.~\eqref{eq:canonical-gradient-form} also reconciles the $\alpha=0$ case for SQDF~\cite{kang2026sqdf} and policy-gradient methods; see Appendix~\ref{appendix:deriving-psi} for details.

Eq.~\eqref{eq:canonical-gradient-form} also clarifies the hierarchy of RSM's design dimensions.
The primary source of variation between methods lies in the way $\textcolor{ForestGreen}{\hat{\mathbf{\Psi}}_{t_i}}$ is constructed.
By contrast, temporal weighting and trust-region realization act only as secondary coefficient-level modifiers via $\gamma(t_i), C_1(t_i), C_2(t_i)$ and clipping. 
We therefore focus first on estimator design, and defer these coefficient-level choices to Section~\ref{sec:practical-design}.

\vspace{-1em}
\subsection{\texorpdfstring{Estimating Optimal Value Guidance $\textcolor{ForestGreen}{\hat{\mathbf{\Psi}}_{t_i}}$}{Estimating Optimal Value Guidance}}
\label{sec:value_estimate}

This section, which is one of the most important contributions of this paper, shows that the seemingly different forms of reward-based finetuning methods can be unified as the zeroth and first order estimate of the current or lookahead estimator of the optimal value guidance estimate.
Specifically, Eq.~\eqref{eq:alternate-optimal-psi} suggests two primary practical families of value-guidance estimators: \emph{current-state} and \emph{lookahead} \footnote{Adjoint Matching~\cite{domingo-enrich2025adjoint} provides an alternative first-order construction that propagates reward gradients via the adjoint method; see Appendix~\ref{appendix:adjoint_matching}.}.

\textbf{Current-state Estimator.}
The first equality in Eq.~\eqref{eq:alternate-optimal-psi}
suggests a \emph{current-state} estimator, available only in the first-order setting;
VGG-Flow~\cite{liu2025vggflow} and ReFL~\cite{ImageReward} fall in this family.
Specifically, since we do not have access to the value function $V$ but only to the terminal reward $r$, we approximate the value gradient via a Tweedie estimate.

\begin{definition}[Current-state Estimator]
\label{def:current_state_estimator}
For a local pullback matrix $A_{t_i}$, the first-order current-state estimator is
\begin{align}
\textcolor{ForestGreen}{\hat{\mathbf{\Psi}}_{t_i}^{\mathrm{CS},1}}[A_{t_i}]
&=
\frac{1}{\alpha}
A_{t_i}^\top
\nabla_{\hat\vx_{0|t_i}} r(\hat\vx_{0|t_i})
\;\approx\;
\frac{1}{\alpha}\nabla_{\vx_{t_i}} V_{t_i}(\vx_{t_i}).
\label{eq:cs-estimator}
\end{align}
\end{definition}
VGG-Flow uses the full Tweedie pullback $A_{t_i} = \mathbf J_{\mathrm{Tw}} \triangleq
\frac{\partial \hat\vx_{0|t_i}}{\partial \vx_{t_i}}
$,
whereas ReFL uses the identity pullback $A_{t_i}=\mathbf I$, with the remaining scalar Jacobian factor absorbed into $\gamma(t_i)$.
Current-state estimators avoid stochastic rollouts, but incur substantial bias at low-SNR timesteps, as will be discussed later.
They are therefore best viewed as simpler but potentially more biased alternatives to the lookahead estimators introduced below, which subsume most existing methods and expose the main design tradeoffs.

\textbf{Lookahead Estimator.}
Lookahead estimators roll out from $\vx_{t_i}$ to $\vx_{t_j}$ with $j<i$ and extract reward information there.
First-order estimators require differentiable rewards, whereas zeroth-order estimators apply to arbitrary black-box rewards.

\begin{definition}[Lookahead Estimators]
\label{def:lookahead_estimators}
For \emph{lookahead depth} $j<i$ and branching width $K_i$, the first-order and zeroth-order lookahead estimators of the existing reward-alignment methods are
\begin{alignat}{2}
\textcolor{ForestGreen}{\hat{\mathbf{\Psi}}_{t_i}^{\mathrm{LA},1}(j, K_i)}
&=
\frac{\sigma_{t_i}^2}{\alpha\Omega(t_i)}
\frac{1}{K_i}\sum_{k=1}^{K_i}
\nabla_{\vx_{t_{i-1}}^{(k)}} r(\hat\vx_{0|t_j}^{(k)})
&\;\approx\;&
\frac{\sigma_{t_i}^2}{\alpha\Omega(t_i)}
\mathbb{E}_{\vx_{t_{i-1}:t_j}}
\Big[
\nabla_{\vx_{t_{i-1}}} r(\hat\vx_{0|t_j})
\Big], \label{eq:fo-estimator}\\[-0.3em]
\textcolor{ForestGreen}{\hat{\mathbf{\Psi}}_{t_i}^{\mathrm{LA},0}(j, K_i)}
&=
\frac{\sigma_{t_i}}{\alpha \Omega(t_i)}
\frac{1}{K_i}\sum_{k=1}^{K_i}
r(\hat\vx_{0|t_j}^{(k)})\,\veps_{t_i}^{(k)}
&\;\approx\;&
\frac{\sigma_{t_i}}{\alpha \Omega(t_i)}
\mathbb{E}_{\vx_{t_{i-1}:t_j}}
\Big[
r(\hat\vx_{0|t_j})\,\veps_{t_i}
\Big], \label{eq:zo-estimator}
\end{alignat}
where $\hat\vx_{0|t_j}= \mathbb{E} [\vx_0\mid \vx_{t_j}] \triangleq  \mathbb{E}_{\vx_{0}\sim p_{t_j}(\vx_0\mid \vx_{t_j})}
\!\left[ \vx_0 \right]$ is the Tweedie estimate of the posterior mean.
\end{definition}

Existing local first-order lookahead methods (Residual $\nabla$-DB~\cite{liu2025resnabladb}, SQDF~\cite{kang2026sqdf})
use one-step lookahead ($j=i-1$), and therefore inherit substantial Tweedie bias from $\hat{\vx}_{0|t_j}$ at low-SNR timesteps, akin to current-state methods ($j=i$).

By contrast, DRaFT-$K$~\cite{clark2024draft} uses full-lookahead guidance ($j=0$) over the final $K$ reverse steps, while truncating gradient propagation beyond them. Equivalently, it applies full-lookahead guidance at those $K$ timesteps and sets the effective guidance to zero at all remaining timesteps\footnote{Untruncated DRaFT applies full-lookahead guidance ($j=0$) at every timestep, but is impractical because backpropagating through the entire Jacobian chain incurs prohibitive memory costs.}.

Several local first-order methods (Residual $\nabla$-DB, VGG-Flow) attempt to correct the bias of their local value-gradient approximation by learning the Jensen gap via a small neural network $g_\phi$.
Proposition~\ref{prop:consistent_first_order} shows that, if this residual parameterization were exact, both local current-state and one-step first-order estimators would be unbiased; see Appendix~\ref{appendix:formalize-first-order} for the proof.
\begin{restatable}[Conditional Consistency of First-Order Estimators]{proposition}{consistentfirstorder}
\label{prop:consistent_first_order}
Let \(j \in \{i,\, i-1\}\). For
\begin{equation}
\label{eq:fo-decomposition-prop}
\nabla_{\vx_{t_j}} V_{t_j}(\vx_{t_j})
=
\gamma_{t_i}\nabla_{\vx_{t_j}} r(\hat \vx_{0 \mid t_j})
+
g_\phi(\vx_{t_j}),
\end{equation}
define the residual-corrected first-order estimator:
\begin{equation}
\label{eq:fo-consistent-estimator-compact}
\textcolor{ForestGreen}{\hat{\mathbf{\Psi}}_{t_i}^{\mathrm{FO, res}}(j)}
\triangleq
\left(
\frac{1}{\alpha}\iota[j=i]
+
\frac{\sigma_{t_i}^2}{\alpha \Omega(t_i)}\iota[j=i-1]
\right)
\Big(
\gamma_{t_i}\nabla_{\vx_{t_j}} r(\hat \vx_{0 \mid t_j})
+
g_\phi(\vx_{t_j})
\Big),
\end{equation}
where \(\iota[\,\cdot\,]\) denotes the indicator function, and for \(j=i-1\) we take
\(\vx_{t_{i-1}} \sim p^\star_{t_i}(\cdot \mid \vx_{t_i})\),
while for \(j=i\) we simply set \(\vx_{t_j}=\vx_{t_i}\).
Then,
$\mathbb{E}\!\left[
\textcolor{ForestGreen}{\hat{\mathbf{\Psi}}_{t_i}^{\mathrm{FO, res}}(j)}
\;\middle|\;
\vx_{t_i}
\right]
$ is an unbiased estimate of the value gradient.
\end{restatable}

However, as we later verify in Appendix~\ref{subsec:experiment-invalidate-auxiliary}, auxiliary corrections $g_\phi$ turn out to be effectively negligible in the existing methods, so the resulting local current-state and one-step first-order estimators remain biased.

Existing zeroth-order (policy-gradient) methods typically use full rollout ($j=0$)\footnote{In principle, first-order consistency can also be achieved by full rollout, i.e. by selecting $j=0$ and differentiating the terminal reward through the entire denoising trajectory. However, this requires the full Jacobian chain and is therefore computationally impractical and potentially unstable in high dimension.
}, which is an unbiased value gradient estimate.
The following theorem, following from Stein's identity~\cite{Stein1981}, completes the theoretical unification of zeroth- and first-order methods.

\begin{restatable}[Consistency of Policy Gradient]{theorem}{policygradientzeroth}
\label{thm:stein-lemma}
For the full-rollout $j=0$, the lookahead zeroth-order estimator of the existing reward-alignment methods is an unbiased estimator of the value gradient:
\begin{equation}
\begin{aligned}
    \textcolor{ForestGreen}{\hat{\mathbf{\Psi}}_{t_i}^{\mathrm{LA},0}}
    &\approx \frac{\sigma_{t_i}}{\alpha \Omega(t_i)} \mathbb{E}_{\vx_{t_{i-1:0}}\sim p_{t_{i:1}}^\star,\veps_{t_{i:1}} \overset{\mathrm{iid}}{\sim} \mathcal{N}(\mathbf{0}, \mathbf{I})}[r(\vx_0)\veps_{t_i}] = \mathbb{E}_{\vx_{t_{i-1}}\sim p_{t_{i}}^\star}[\frac{\sigma_{t_i}^2}{\alpha \Omega(t_i)} \nabla_{\vx_{t_{i-1}}} V_{t_{i-1}}(\vx_{t_{i-1}})].
\end{aligned}
\end{equation}
\end{restatable}

See Appendix~\ref{appendix:formalize-stein} for the proof.
Lookahead estimators introduce two design dimensions for variance reduction.
\emph{Branching} specifies per-step branch widths $K_i$, with $K_i=1$ denoting a point estimate at $t_i$; larger $K_i$ reduce variance at higher compute cost~\cite{li2025branchgrpostableefficientgrpo,he2025tempflowgrpotimingmattersgrpo}.
\emph{Stochasticity localization} changes the underlying stochastic control problem by keeping only selected reverse steps stochastic~\cite{li2025mixgrpounlockingflowbasedgrpo,he2025tempflowgrpotimingmattersgrpo}.

\section{Practical Design Choices of RSM from Optimization Perspectives}
\label{sec:practical-design}

Section~\ref{sec:theory} established exact reductions of many existing methods to a common RSM objective.
We now shift to the practical optimization view: under finite compute and noisy gradient estimates, performance is governed by three coupled dimensions: \emph{estimator design}, \emph{temporal weighting}, and \emph{trust-region realization}.
These determine the quality of value guidance and how compute, optimization mass, and regularization are allocated across timesteps.

\textbf{Estimator design} determines the quality of $\textcolor{ForestGreen}{\hat{\mathbf{\Psi}}_{t_i}}$.
Its main knobs are lookahead depth $j$, branching width $\{K_i\}$, and stochasticity localization.
Deeper lookahead reduces bias at higher compute cost, while larger branching widths reduce finite-sample variance at higher cost.
Stochasticity localization can likewise improve estimator quality, but unlike branching, it also changes the underlying stochastic control problem.
Reward centering, e.g., subtracting the sample mean, may reduce variance for free without changing the underlying gradients.
Under a fixed compute budget, two allocation problems emerge: how to distribute the total compute budget \emph{across timesteps}, and how to use compute \emph{within each timestep} to obtain a useful guidance estimate.

An important example is the zeroth-order estimator $\textcolor{ForestGreen}{\hat{\mathbf{\Psi}}_{t_i}^{\text{LA}, 0}}$, whose conditional variance under independent descendant rollouts scales as:
\begin{equation}
\label{eq:branch-covariance}
\mathrm{Cov}\!\left(\textcolor{ForestGreen}{\hat{\mathbf{\Psi}}_{t_i}^{\text{LA}, 0}}\mid \vx_{t_i}\right)
=
\frac{1}{K_i}
\left(\frac{\sigma_{t_i}}{\alpha\Omega(t_i)}\right)^2
\mathrm{Cov}\!\left(r(\vx_0)\veps_{t_i}\mid \vx_{t_i}\right)
\apropto
\frac{1}{K_i}
\left(\frac{\sigma_{t_i}}{\alpha\Omega(t_i)}\right)^2.
\end{equation}
This scaling shows that, under a fixed branching width ${K_i}$, some timesteps are intrinsically noisier than others and may therefore require larger $K_i$ to reduce variance, at the cost of additional compute.

\begin{figure}[!t]
\centering
\vspace{-1.5em}
\includegraphics[width=0.85\linewidth]{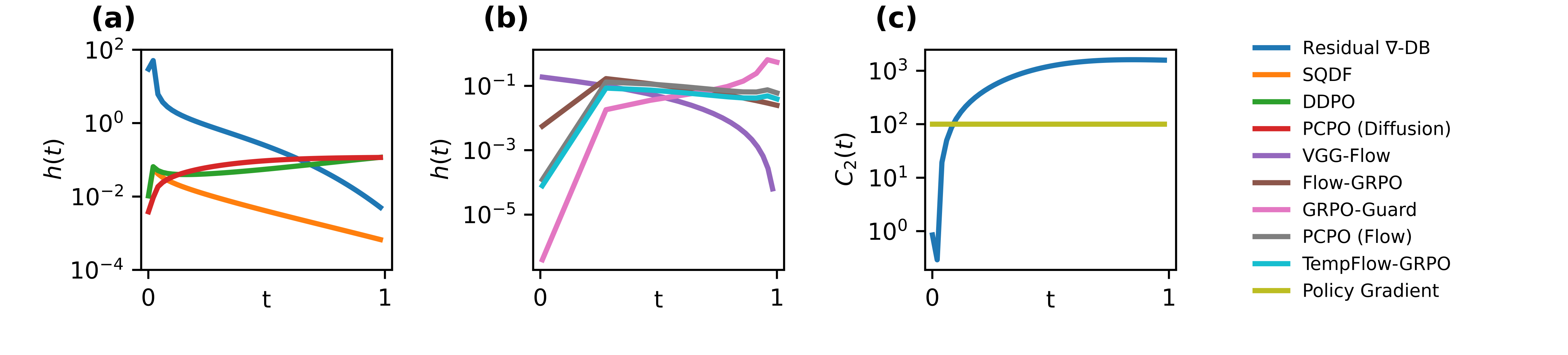}
\vspace{-1em}
\caption{\textbf{Temporal Optimization Strength.} (a, b) Many first-order methods reduce value guidance at low-SNR timesteps; improved zeroth-order methods reduce value guidance at high-SNR timesteps: (a) diffusion, (b) flow-matching. (c) Residual $\nabla$-DB enforces stronger trust-region constraints for low-SNR timesteps. Policy Gradient's $C_2(t)$ used constants $r(\vx_0)=1, \alpha=10^{-2}$.}
\label{fig:h}
\end{figure}

\textbf{Temporal weighting.}
Another key question is how strongly the reward signal enters the update at each timestep.
We seek a common coefficient measuring its timestep-wise optimization strength.
For current-state estimators, Eq.~\eqref{eq:cs-estimator} shows that $\textcolor{ForestGreen}{\hat{\mathbf{\Psi}}_{t_i}}$ is proportional to a reward gradient estimate at the current state.
For lookahead estimators, Theorem~\ref{thm:stein-lemma} together with Eqs.~\eqref{eq:fo-estimator}--\eqref{eq:zo-estimator} yields:
\begin{equation}
\label{eq:r}
\widehat{\nabla r(\vx_{t_i})}
=
\frac{1}{\sigma_{t_i}}
\mathbb{E}_{\vx_{t_{i-1:0}}\sim p_{t_{i:1}}^\star,\veps_{t_{i:1}} \overset{\mathrm{iid}}{\sim} \mathcal{N}(\mathbf{0}, \mathbf{I})}
\big[
r(\vx_0)\veps_{t_i}
\big]
=
\mathbb{E}_{\vx_{t_{i-1:0}}\sim p_{t_{i:1}}^\star}
\big[
\nabla_{\vx_{t_{i}}} r(\vx_0)
\big].
\end{equation}
Thus, all cases admit a common reward-gradient signal.
Using the parameterization-conversion factor $\delta(t_i)$ (Appendix~\ref{appendix:delta-derivations}), defined by
$-\delta(t)(\vg^\theta_t-\vg^\text{ref}_t)=\vs^\theta_t-\vs^\text{ref}_t$
for $\vg \in \{\veps, \vv\}$,
Eq.~\eqref{eq:gradient} becomes
\begin{align}
\nabla_\theta \mathcal{L}(\theta)
&=
2\mathbb{E}[\mathbf{J}_\theta(\vg_{t_i}^\theta)^\top \tilde{\mathcal{G}}(\vx_{t_i})]
\label{eq:gradient-reparam}
\\[-1.5em]
\tilde{\mathcal{G}}(\vx_{t_i})
&=
h(t_i)\widehat{\nabla r(\vx_{t_i})}
+ \delta(t_i)^2 C_1(t_i)\Big(
\overbrace{(\vg^\theta_{t_i} - \vg^\text{ref}_{t_i})}^{\mathclap{\text{KL Reg.}}}
+ C_2(t_i)\overbrace{(\vg^\theta_{t_i} - \vg^{\theta^\dagger}_{t_i})}^{\mathclap{\text{Trust Region}}}
\Big).
\label{eq:gradient-reparam-expanded}
\\[-0.5em]
h(t_i)
&:=
\delta(t_i) C_1(t_i)\times
\begin{cases}
\gamma(t_i)/\alpha, & (\text{current-state}),\\
\gamma(t_i)\sigma_{t_i}^2/(\alpha\Omega(t_i)), & (\text{lookahead}).
\end{cases}
\label{eq:h-definition}
\end{align}
The \emph{Normalized Influence Metric} $h(t)$ measures the effective scaling of the reward gradient after accounting for temporal weighting and parameterization.
For representative methods (Fig.~\ref{fig:h}), $h(t)$ reveals a clear contrast: representative first-order methods concentrate mass on high-SNR timesteps, suppressing low-SNR steps where bias is largest, whereas improved zeroth-order methods do the opposite, suppressing high-SNR steps where variance is largest under fixed branching width (Eq.~\eqref{eq:branch-covariance}).
Note that, while omitted for simplicity, temporal weighting also subsumes any rescaling of reward-gradients (first-order) or rewards (zeroth-order).

\textbf{Trust-region realization} appears through $C_2(t)$ and clipping. Optimization strength depends not only on $h(t)$, but also on how much of the reward gradient survives regularization and clipping.

\section{Experiments}
\label{sec:optimizing_design}
\vspace{-0.5em}

We study the practical design space of RSM in two stages.
First, we isolate estimator design at a single timestep under fixed compute, asking which local allocations estimate $\textcolor{red}{\mathbf{\Psi}_{t_i}^\star}$ most effectively.
Second, we study end-to-end design under fixed total compute, asking how estimator design, temporal weighting, and trust-region realization should be chosen jointly across timesteps.

\begin{figure}[!t]
\centering
\vspace{-0.5em}
\includegraphics[width=0.95\linewidth]{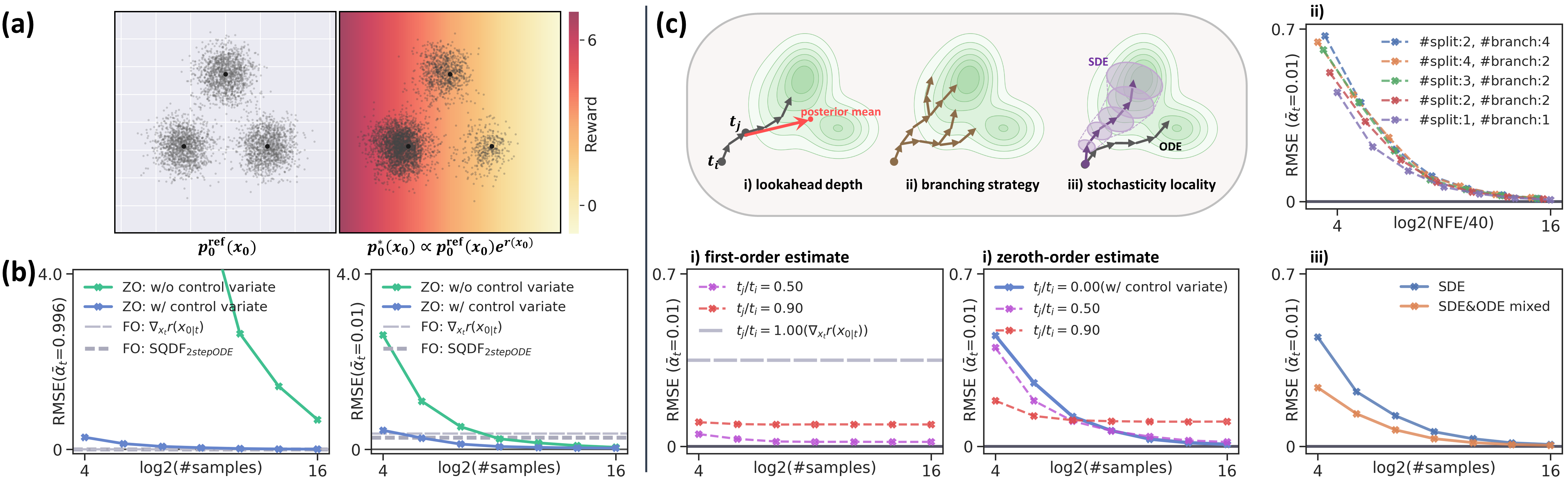}
\caption{\textbf{Toy analysis of estimator quality under fixed compute.}
(a) Reference and reward-tilted distributions.
(b) RMSE of representative first-order (FO) and zeroth-order (ZO) estimators\protect\footnotemark\ at two timesteps.
(c) RMSE--NFE for estimators with varying lookahead depth, branching, and stochasticity localization. \emph{\#split} is the number of recursive branching stages; \emph{\#branch} is the number of child trajectories per stage. Mixed SDE-ODE sampling follows TempFlow-GRPO~\citep{he2025tempflowgrpotimingmattersgrpo}.}
\label{fig:toy_experiment}
\vspace{-1em}
\end{figure}
\footnotetext{$\text{SQDF}_{2stepODE}$ refers to the 2-step ODE estimator proposed in SQDF~\cite{kang2026sqdf}.}

\subsection{Isolated Estimator Analysis}
\label{subsec:isolated-estimator}
\vspace{-0.5em}

\textbf{Setup.}
We construct a 2D toy problem with a Gaussian-mixture reference distribution $p_0^{\mathrm{ref}}(\vx_0)$ and simple $r(x)$, yielding an analytically tractable $\textcolor{red}{\mathbf{\Psi}_{t_i}^\star}$.
We compare estimators $\textcolor{ForestGreen}{\hat{\mathbf{\Psi}}_{t_i}^\text{LA}}$ 
from Eqs.~\eqref{eq:fo-estimator}--\eqref{eq:zo-estimator} 
against ground truth $\textcolor{red}{\mathbf{\Psi}_{t_i}^\star}$ by RMSE under matched compute.
See Appendix~\ref{appendix:toy-experiment-details} for details.

\textbf{Results.}
Fig.~\ref{fig:toy_experiment}(b) shows that one-step first-order lookahead is severely biased at low SNR, whereas full-rollout zeroth-order guidance is unbiased but requires many more samples.
Reward centering reduces RMSE at no additional cost.
Fig.~\ref{fig:toy_experiment}(c-i) shows that lookahead depth $j$ induces a bias--variance--compute tradeoff.
Although deeper lookahead reduces bias, under limited samples shallower lookahead can be both more accurate and more compute-efficient, especially for zeroth-order estimators, because averaging over posterior uncertainty lowers variance relative to full rollout.
At fixed lookahead depth, first-order estimators typically improve faster with sample count than zeroth-order estimators\footnote{In this toy analysis, first-order curves use the full Jacobian chain to compute the exact first-order signal.}.
Fig.~\ref{fig:toy_experiment}(c-ii) shows that, under fixed total sample budget, different branching patterns yield similar RMSE--NFE trends\footnote{Larger $K_i$ improves the conditional estimate for a given parent latent, but under fixed budget reduces the number of distinct parent latents explored at timestep $t_i$. This tradeoff is not captured here, but is important in practice.}.
Fig.~\ref{fig:toy_experiment}(c-iii) shows that stochasticity localization improves estimator quality under fixed compute by making the induced value function easier to predict\footnote{As discussed in Section~\ref{sec:value_estimate}, this benefit comes partly from changing the underlying stochastic control problem rather than merely reducing variance for the same target.}.

\textbf{Findings.}
Under finite compute, the best local estimator need not be the least biased one: shallower lookahead estimators may outperform unbiased full-rollout estimators.
Our results also suggest that point estimates are too noisy.
Once branching is used at all, the specific choice of $(\#\text{split}, \#\text{branch})$ is relatively less important.
In addition, subtracting the sample mean substantially reduces variance.

\subsection{End-to-End Design}
\label{subsec:end-to-end-design}
We now test whether these design principles improve end-to-end optimization under matched compute.
Our improvements require comparable or less per-epoch compute than baselines; see Appendices~\ref{appendix:experiment-details},~\ref{appendix:extended-ablations} and~\ref{appendix:qualitative-results} for experiment details, more ablations and qualitative results, respectively.

\begin{figure}[!t]
\centering
\includegraphics[width=0.82\linewidth]{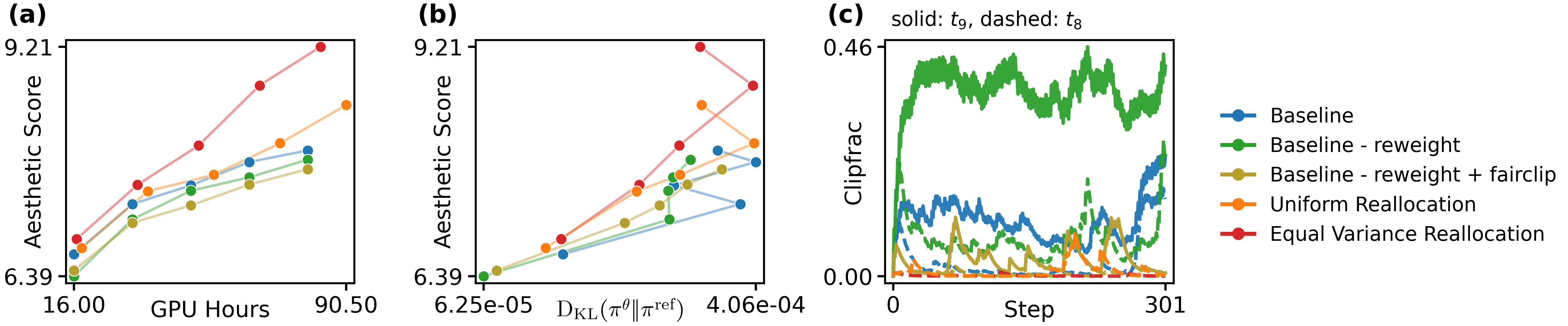}
\caption{\textbf{Improving high-SNR timesteps is better than merely suppressing them.}
Making clipping timestep-fair and reallocating budget improves performance and stability.
(a) Aesthetic Score vs.\ GPU hours.
(b) Aesthetic Score vs.\ $\mathcal{D}_\text{KL}$.
(c) Clip fraction for $t_9$ (solid) and $t_8$ (dashed).}
\label{fig:fairclip_reallocation}
\vspace{-1em}
\end{figure}

\textbf{Mechanistic Analysis.}
We begin with a case study on the zeroth-order flow-matching baseline TempFlow-GRPO.
Eq.~\eqref{eq:branch-covariance} suggests that, under fixed branching width, high-SNR timesteps are intrinsically noisier.
Yet the baseline allocates no additional budget to them, relying instead on strong clipping and downweighting for stability.

To address this inefficiency, we propose a redesign in three steps.
First, we make clipping comparable across timesteps by moving the timestep-dependent scale outside the clipping operator:
\begin{equation}
\label{eq:fair-clip}
\text{clip}_\xi\!\left(\frac{\|\mu_{t_i}^\theta-\mu_{t_i}^{\theta^\dagger}\|^2}{2\tilde{\sigma}_{t_i}^2 \Delta t_i}\right)
    \;\longrightarrow\;
    \frac{1}{\tilde{\sigma}_{t_i}^2 \Delta t_i}
    \text{clip}_\xi\!\left(\frac{\|\mu_{t_i}^\theta-\mu_{t_i}^{\theta^\dagger}\|^2}{2}\right).
\end{equation}
Second, we reallocate branching budget toward high-SNR timesteps, to reduce estimator variance where it is largest.
Third, despite this redistribution, we find that $t_9$ remains too noisy and heavily clipped to justify investment, so we deactivate it and redistribute its budget to other timesteps.

Fig.~\ref{fig:fairclip_reallocation} shows that uniform redistribution already improves over TempFlow-GRPO, while variance-aware redistribution performs best.
Our redesign improves reward more quickly while maintaining a comparable reward--KL tradeoff, suggesting that budget should be concentrated on timesteps whose guidance can be made reliable rather than on timesteps whose updates are later suppressed.

\textbf{Validation Across Settings.}
In zeroth-order flow matching, uniform branching across timesteps is suboptimal because timestep importance is reward-dependent.
For the primarily semantic GenEval~\cite{ghosh2023geneval} reward, high-SNR timesteps matter less, since semantic structure is mainly determined at low SNR~\cite{yu2023freedom}.
A redesign that concentrates branching budget on low-SNR timesteps reaches GenEval \(=0.97\) with an estimated \(5\times\) speedup over the reported TempFlow-GRPO baseline (Fig.~\ref{fig:zeroth_order_geneval_hps}(a)).

In zeroth-order diffusion, the PCPO~\cite{lee2025pcpo} baseline uses $K_i = 1$ for $\forall i$; each timestep receives a high-variance point estimate of the value gradient.
Ideally, one would choose $K_i \gg 1$ at all timesteps, but doing so over the full trajectory ($N=50$) is prohibitively expensive.
As a practical compromise, we choose $K_i \gg 1$ at a sparse subset of steps and $K_i=0$ elsewhere.
This sparse-allocation design substantially improves efficency while achieving a better reward--KL tradeoff (Fig.~\ref{fig:zeroth_order_geneval_hps}(b, c)).

Existing first-order baselines~\cite{liu2025vggflow,liu2025resnabladb,kang2026sqdf} rely on Tweedie-based guidance. The severe bias at low-SNR timesteps is largely obscured by their original aesthetic-dominant evaluation settings\footnote{Prior first-order methods are evaluated either on purely aesthetic rewards or on HPSv2.1~\cite{wu2023humanpreferencescorev2} with prompt subsets such as \texttt{hpd\_photo\_painting}, where semantic demands are relatively weak.} and aggressive low-SNR downweighting.
When fine-tuning on HPSv2.1 with GenEval training prompts, where both aesthetic quality and semantic fidelity matter, performance drops sharply under original designs, since meaningful learning must occur at low-SNR timesteps as well.
We therefore (i) replace these local gradients with terminal-image reward gradients \(\nabla_{\vx_0} r(\vx_0)\), effectively moving to \(j=0\), which yields a practical \emph{reward score distillation} update\footnote{Although the exact quantity \(\nabla_{\vx_{t_i}} r(\vx_0)\) would provide an unbiased first-order guidance signal, computing it requires a high-order Jacobian chain through the full denoising trajectory, which is both computationally expensive and potentially destabilizing, as noted in \citet{poole2023dreamfusion}. We therefore use the practical surrogate \(\nabla_{\vx_0} r(\vx_0)\); see Appendix~\ref{appendix:experiment-details} for details.}, while (ii) reversing the inherited low-SNR downweighting.
Across both flow-matching and diffusion backbones, this yields substantially faster reward improvement while remaining competitive in the reward--KL tradeoff (Fig.~\ref{fig:first_order_validation})\footnote{We display \emph{transition mean drift} for first-order flow matching, since KL divergence is undefined for the VGG-Flow baseline~\cite{liu2025vggflow}; see Appendices~\ref{appendix:end-to-end-design-details} and \ref{appendix:extended-ablations}.}.

\begin{figure}[!t]
\centering
\vspace{-1.2em}
\includegraphics[width=0.72\linewidth]{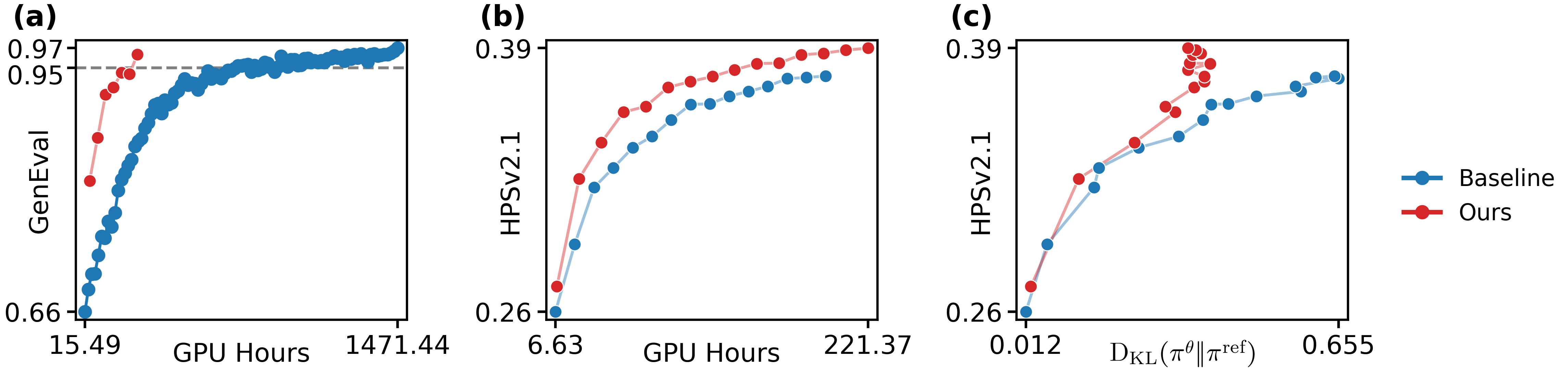}
\caption{\textbf{Validation: Zeroth-order methods.}
Principled budget allocation and temporal weighting improve performance on
(a) GenEval with SD3.5-M\protect\footnotemark\ and
(b, c) HPSv2.1 with SD1.5.}
\label{fig:zeroth_order_geneval_hps}
\vspace{-0.5em}
\end{figure}
\footnotetext{Due to resource constraints, we did not rerun the GenEval baseline; instead, we plot reverse-engineered data from Figure~3 of \citet{he2025tempflowgrpotimingmattersgrpo}. As a result, reward--KL tradeoff comparisons are unavailable for GenEval.}

\begin{figure}[!t]
\centering
\includegraphics[width=0.9\linewidth]{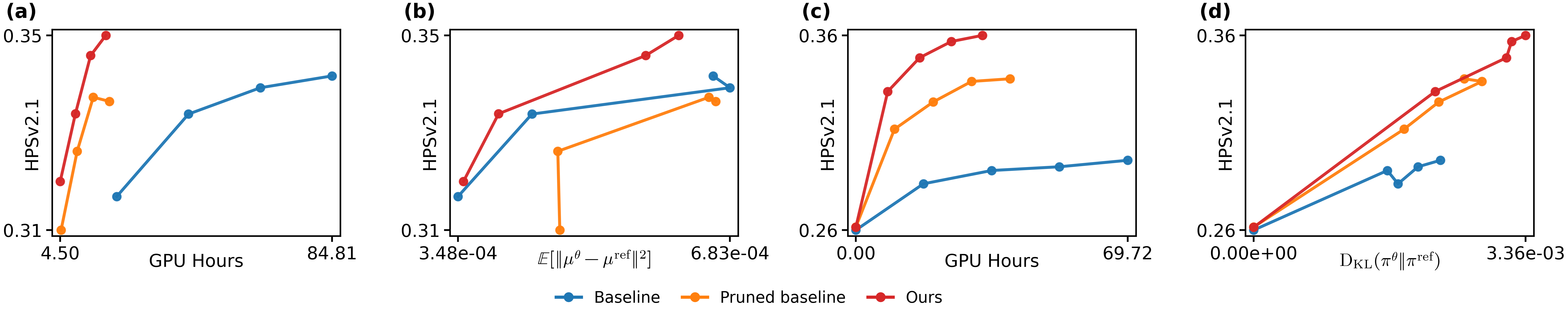}
\caption{\textbf{Validation: First-order methods.} Improved reward guidance for low-SNR timesteps yields faster reward gains, while maintaining a competitive reward--KL tradeoff on (a, b) SD3.5-M and (c, d) SD1.5.
See Appendix~\ref{appendix:reward-hacking} for more results.
}
\label{fig:first_order_validation}
\vspace{-1em}
\end{figure}

\section{Discussion}
\label{sec:discussion}

\textbf{Broadening the framework.}
RSM covers a broad class of affine flow-based reward fine-tuning methods, including methods derived from Soft RL as well as seemingly different formulations such as VGG-Flow, DRaFT/DRaFT-$K$, ReFL, and Adjoint Matching.
Several related objectives remain closely connected but non-identical.
Soft Q-learning is also a subset of soft RL but admits slightly different gradients~\cite{schulman2018equivalencepolicygradientssoft} (see Appendix~\ref{appendix:hard-soft-value} for details).

DiffusionNFT~\cite{zheng2025diffusionnft} can be interpreted as reward-weighted regression, which is closely connected to Soft RL but does not reduce exactly to the RSM objective (see Appendix~\ref{appendix:temp-diffusionnft}).
This connection suggests a natural first-order analogue of reward-weighted regression, obtained through a straightforward extension of DRaFT-LV (see Appendix~\ref{appendix:draft-lv-rwr}).

A complementary direction is to improve critic learning~\cite{jensen2026value,zhao2025scoreasaction}, since existing auxiliary networks $g_\phi$ often collapse under weak supervision and may benefit from better bootstrapped or consistency-based targets.

Related ideas extend to inference-time alignment, which can also be understood through the lens of value gradients~\cite{uehara2025inferencetimealignmentdiffusionmodels,holderrieth2026diamondmapsefficientreward}.
Appendix~\ref{appendix:inference_time_scaling} shows that methods such as DNO~\cite{tang2025inferencetime} rely on closely related zeroth-order principles.

\textbf{Better tuning of the design space.}
Rather than performing exhaustive tuning or aggressive engineering, we tested only a small number of principled modifications, yet observed substantial gains.
These results suggest considerable headroom for more systematic exploration of RSM's design space.
Further simplification may also be possible, for example by absorbing the trust-region term $C_2$ into temporal weighting, as in GRPO-Guard~\cite{wang2025grpoguard} (see Eq.~\eqref{eq:grpo-guard-unified-terms}).
In addition, joint optimization of the sampler and training objective, as in PCPO (diffusion)~\cite{lee2025pcpo}, could enable more fine-grained control over estimator variance and native temporal weights.
Furthermore, for differentiable rewards, interpolating between first- and zeroth-order estimators~\cite{ren2026halforder} may yield a more preferable tradeoff.

\textbf{Limitations.}
Although our redesigns improve existing reward-based fine-tuning methods, they do not eliminate reward hacking.
Our auxiliary evaluations suggest that the gains are not driven by \emph{more} aggressive reward hacking, but they do not establish emergent robustness to reward misspecification.
More broadly, while RSM exactly covers a broad class of reward-based fine-tuning methods for flow-based models, its current formulation does not cover offline preference-based objectives such as Diffusion-DPO~\cite{Wallace_2024_CVPR}.

\textbf{Broader Impact.}
\label{sec:impact-statement}
Improving alignment of flow-based models is important for generating data with desired properties, thereby better serving human needs across various applications.
Such methods may have harmful downstream uses, but we are not aware of any additional negative impacts specific to this work that require particular emphasis.

\section{Conclusion}
\label{sec:conclusion}

We introduced Reward Score Matching (RSM), a unified view of reward-based fine-tuning for flow and diffusion models.
Under this view, many methods that are motivated from different perspectives reduce to the same underlying score-matching problem, with their main differences arising from how value gradients are estimated, weighted, and regularized across timesteps.

This perspective also clarifies the practical structure of the design space.
In particular, it puts first-order and zeroth-order methods on common ground as different estimators of the same target, exposing shared bias--variance--compute tradeoffs that are otherwise obscured by method-specific derivations.
It further shows that much of the remaining variation can be understood through three coupled design dimensions: estimator design, temporal weighting, and trust-region realization.

Guided by this framework, we identified several simplifications and principled redesigns that improve fine-tuning-based alignment across representative differentiable and black-box rewards, often at substantially lower compute cost.
Taken together, these results suggest that progress depends less on introducing increasingly specialized objectives and more on understanding how to estimate, weight, and preserve value-guided updates effectively.
We aspire for RSM to provide a common language for comparing existing methods, and a practical foundation for designing stronger ones.

\bibliographystyle{plainnat} 
\bibliography{main}

\medskip

\newpage
\appendix

\section{Notation Guide}
\label{appendix:notation-guide}

Table~\ref{tab:notation-guide} summarizes the main symbols used throughout the paper.

\begin{table*}[t]
\centering
\small
\setlength{\tabcolsep}{5pt}
\renewcommand{\arraystretch}{1.08}
\caption{\textbf{Notation guide.} Main symbols used in the unified RSM framework.}
\label{tab:notation-guide}
\begin{tabular}{p{0.06\linewidth} p{0.43\linewidth} p{0.49\linewidth}}
\toprule
\textbf{Symbol} & \textbf{Meaning} & \textbf{Main role} \\
\midrule
$\textcolor{red}{\mathbf{\Psi}_{t_i}^\star}$ 
& Optimal value guidance, $\frac{1}{\alpha}\nabla_{\vx_{t_i}}V_{t_i}$ 
& Ideal reward-guided correction to the reference score \\

$\textcolor{ForestGreen}{\hat{\mathbf{\Psi}}_{t_i}}$ 
& Practical estimate of $\textcolor{red}{\mathbf{\Psi}_{t_i}^\star}$ 
& Guidance estimate before temporal reweighting \\

$\textcolor{blue}{\mathbf{\Psi}_{t_i}}$ 
& Effective guidance, $\gamma(t_i)\textcolor{ForestGreen}{\hat{\mathbf{\Psi}}_{t_i}}$ 
& Guidance actually used in the loss \\

$\gamma(t_i)$ 
& Heuristic temporal weight 
& Controls timestep-wise guidance strength \\

$C_1(t_i)$ 
& Main timestep weight in Eq.~\eqref{eq:master_equation} 
& Scales the score-matching update at $t_i$ \\

$C_2(t_i)$ 
& Trust-region coefficient in Eq.~\eqref{eq:master_equation} 
& Scales anchoring to the previous policy / score \\

$h(t_i)$ 
& Normalized Influence Metric 
& Effective timestep-wise optimization strength \\

$\Omega(t_i)$ 
& Score coefficient in the reverse transition kernel 
& Converts score guidance into mean shift \\

$\delta(t_i)$ 
& Parameterization-conversion factor 
& Maps native outputs to score differences \\

$j$ 
& Lookahead depth 
& Sets how far rollout proceeds before reward extraction \\

$K_i$ 
& Branching width at timestep $t_i$ 
& Number of descendants used for conditional estimation \\
\bottomrule
\end{tabular}
\end{table*}

\section{Main Proofs}
\label{appendix:core}

\subsection{Generality of affine conditional flows}
\label{appendix:general-flow-de}
We show that the affine conditional flow (Eq.~\eqref{eq:affine-flow-expand}) subsumes VP-SDE, VE-SDE, and Rectified Flow.

\textbf{VP-SDE.}
We define the coefficients $a_t=\sqrt{\bar\alpha_t}$ and $b_t=\sqrt{1-\bar\alpha_t}$, where $\bar\alpha_t=\exp{(-\int_0^t \beta(s) ds)}$ is derived from a pre-defined noise scheduler $\beta(s)$ satisfying $\beta(0)=0$ and strictly increasing $\beta(t)$.
Consequently, the flow ODE becomes
\begin{equation}
    d\vx_t = -\frac{1}{2}\beta(t)\vx_t dt- \frac{1}{2}\beta(t)\nabla_{\vx_t}\log p_t(\vx_t) dt.
\end{equation}

\textbf{VE-SDE.}
Let $a_t=1$ and $b_t=\sigma_t$ where $\sigma_t$ denotes a pre-defined noise variance schedule satisfying $\sigma_0=0$ and $\sigma_t < \sigma_{t+1}$.
The corresponding flow ODE is
\begin{equation}
    d\vx_t = -\dot\sigma_t \sigma_t \nabla_{\vx_t} \log p_t(\vx_t)dt.
\end{equation}

\textbf{Rectified Flow.}
Let $a_t=1-t$ and $b_t=t$ for $t\in[0,1]$.
Then, the flow ODE becomes
\begin{equation}
    d\vx_t = \frac{\vx_t - \mathbb{E}[\vx_0|\vx_t]}{t}dt.
\end{equation}

\subsection{Proof of discretized affine flow SDE}
\label{appendix:discretize}

\begin{proof}
Consider a conditional affine flow $\vx_t:=\psi_t(\vx_1|\vx_0)=a_t\vx_0 + b_t\vx_1$ with $\vx_1\sim\gN(0,\rmI)$.
The corresponding flow ODE is
\begin{align}
    d\vx_t = \vv_t(\vx_t, t)dt &= \frac{\dot a_t}{a_t}\vx_t dt + \left( \frac{\dot a_t b_t^2 - a_t\dot b_t b_t}{a_t}\right)\nabla_{\vx_t}\log p_t(\vx_t) dt\nonumber\\
    &= A(t) \vx_t dt + B(t) \vs(\vx_t, t).
\end{align}
From \citet{song2021scorebased}, the flow SDE that matches the same marginal density across time $t$ is
\begin{align}
    d\vx_t &= [\vv_t(\vx_t, t) -\frac{1}{2}g_t \vs(\vx_t, t)]dt + \sqrt{g_t}d\vw_t\\
    &= [A(t)\vx_t + (B(t)-\frac{1}{2}g_t)\vs(\vx_t, t)]dt + \sqrt{g_t}d\vw_t\\
    &= [A(t)\vx_t + \tilde B(t)\vs(\vx_t, t)]dt + \sqrt{g_t}d\vw_t,
\end{align}
where we set $\tilde B(t)=B(t)-0.5 g_t$. Note that $\tilde B(t)=B(t)$ when $g_t=0$ (i.e., flow ODE).

Consider an integrating factor $\Phi(t)=\exp{(\int A(\tau)d\tau)}$. For $\vy_t = \Phi(t)^{-1} \vx_t$, we apply the Itô derivative and obtain
\begin{align}
    d\vy_t &= \Phi(t)^{-1} d\vx_t + \vx_t d\Phi(t)^{-1}\\
    &= \Phi(t)^{-1}[A(t)\vx_tdt + \tilde B(t)\vs(\vx_t, t)dt + \sqrt{g_t} d\vw_t] - A(t)\Phi(t)^{-1}\vx_t dt\\
    &= \Phi(t)^{-1}[\tilde B(t)\vs(\vx_t,t)dt + \sqrt{g_t}d\vw_t],
\end{align}
where the first equality holds as $d\Phi(t)^{-1}d\vx_t \approx 0$.
Integrating both sides from $s$ to $t$ gives
\begin{align}
    \vx_t &= \frac{\Phi(t)}{\Phi(s)} \vx_s+  \frac{\Phi(t)}{\Phi(s)} \int_s^t  \frac{\Phi(s)}{\Phi(r)} \tilde B(r)\vs(\vx_r, r) dr +  \frac{\Phi(t)}{\Phi(s)} \int_s^t  \frac{\Phi(s)}{\Phi(r)} \sqrt{g_r}d\vw_r\\
    &\approx \frac{\Phi(t)}{\Phi(s)} \vx_s+  \left[\frac{\Phi(t)}{\Phi(s)} \int_s^t \frac{\Phi(s)}{\Phi(r)} \tilde B(r)dr \right] \vs(\vx_s, s) +  \frac{\Phi(t)}{\Phi(s)} \int_s^t  \frac{\Phi(s)}{\Phi(r)} \sqrt{g_r}d\vw_r,
\end{align}
where we assume $\vs(\vx_r, r) \approx \vs(\vx_s, s)$, which corresponds to the \emph{first-order exponential integrator} underlying the DPM-solver.
Note that higher-order DPM-solvers use higher-order approximations here.

For conditional flow models, $A(t):=\dot a_t/a_t$ and $\Phi(t):=a_t$.
Thus, the discretized SDE simplifies to:
\begin{align}
    \vx_t = \frac{a_t}{a_s} \vx_s + \left[ \frac{a_t}{a_s} \int_s^t \frac{a_s}{a_r}\tilde B(r)dr \right] \vs(\vx_s, s) + a_t \int_s^t \frac{1}{a_r}\sqrt{g_r} d\vw_r,
\end{align}
where the first term denotes a scaled previous sample, the second term denotes a score function, and the third term denotes the Gaussian noise with zero mean and covariance $\sigma_{s}^2\rmI$ where $\sigma_{s}^2 = a_t^2 \int_s^t \frac{g_r}{a^2_r} dr$.
Thus, setting $s=t_i$ and $t=t_{i-1}$ yields the general transition density:
\begin{align}
    p(\vx_{t_{i-1}}|\vx_{t_i}) = \gN(\underbrace{\kappa(t_i)\vx_{t_i} + \Omega(t_i)\vs(\vx_{t_i})}_{=:\pmb\mu_{t_i}(\vx_{t_i})}, \sigma_{t_i}^2 \rmI).
    \label{eq:general_kernel}
\end{align}
\end{proof}

\subsection{Proof of Proposition~\ref{cor:gen_soft_value}}
We begin with a simple lemma establishing the zero-mean score identity.
\begin{lemma}[Zero-mean score identity under tail regularity]
\label{thm:zero_mean_score}
Fix $\vx\in\mathbb{R}^d$ and let $p(\cdot\mid \vx)$ be a conditional density on $\mathbb{R}^d$
(with respect to Lebesgue measure) such that:
\begin{enumerate}
\item $\int_{\mathbb{R}^d} p(\vu\mid \vx)\,d\vu = 1$;
\item $p(\cdot\mid \vx)$ is differentiable in $\vu$, and $\int_{\mathbb{R}^d}\big|\partial_{u_j} p(\vu\mid x)\big|\,d\vu < \infty \ \text{ for } \ \forall j\in\{1,\dots,d\}$;
\item For $\forall j$, for almost every $\vu_{-j}\in\mathbb{R}^{d-1}$,
      $\lim_{\vu_j\to\pm\infty} p(u_j,\vu_{-j}\mid \vx)=0$.
\end{enumerate}
Then the conditional score has zero mean:
\begin{equation}
\label{eq:zero_mean_score}
\mathbb{E}_{\vu\sim p(\cdot\mid \vx)}\big[\nabla_\vu \log p(\vu\mid \vx)\big] = \mathbf{0}.
\end{equation}
\begin{proof}
For each $j$,
\[
\mathbb{E}_{p(\cdot\mid \vx)}[\partial_{u_j}\log p(\vu\mid \vx)]
= \int_{\mathbb{R}^d} p(\vu\mid \vx)\,\frac{\partial_{u_j}p(\vu\mid \vx)}{p(\vu\mid \vx)}\,d\vu
= \int_{\mathbb{R}^d} \partial_{u_j}p(\vu\mid \vx)\,d\vu,
\]
where the equality is valid wherever $p(\vu \mid \vx)>0$, and condition 2 justifies the integral.
Writing $\vu=(u_j,\vu_{-j})$ and using Fubini's theorem, we obtain
\[
\int_{\mathbb{R}^d}\partial_{u_j}p(\vu\mid \vx)\,d\vu
= \int_{\mathbb{R}^{d-1}}\left(\int_{\mathbb{R}}\partial_{u_j}p(u_j,\vu_{-j}\mid \vx)\,du_j\right)d\vu_{-j}.
\]
For fixed $\vu_{-j}$, the inner integral is a 1D integral of a derivative, hence by the fundamental
theorem of calculus,
\[
\int_{\mathbb{R}}\partial_{u_j}p(u_j,\vu_{-j}\mid \vx)\,du_j
= \Big[p(u_j,\vu_{-j}\mid \vx)\Big]_{u_j=-\infty}^{u_j=+\infty}.
\]
By condition 3, this boundary term is $0$ for almost every $\vu_{-j}$, so the outer integral
is $0$.
Since this holds for each $j$, we conclude Eq.~\eqref{eq:zero_mean_score}.
\end{proof}
\end{lemma}
With this identity in place, we now prove the main result for the first-order score-matching kernel.
\gensoftval*
\begin{proof}
First, we will prove the first part, i.e. the current-state characterization:
\begin{equation}
    \textcolor{red}{\mathbf{\Psi}^\star_{t_i}}(\vx_{t_i}) = \frac{1}{\alpha}\nabla_{\vx_{t_i}}V_{t_i}(\vx_{t_i}).
\end{equation}
From Eq.~\eqref{eq:optimal-marginal-distribution}, we derive the optimal Value Guidance Vector.
Take the logarithm of both sides:
\begin{equation}
  \log p_{t_i}^\star(\vx_{t_i}) = \log p_{t_i}^{\text{ref}}(\vx_{t_i}) + \frac{1}{\alpha}V_{t_i}(\vx_{t_i}) - \log Z.
\end{equation}
Then, differentiate with respect to $\vx$ at $t = t_i$:
\begin{equation}
  \begin{aligned}
    \nabla_{\vx_{t_i}} \log p_{t_i}^\star(\vx_{t_i})
    &= \nabla_{\vx_{t_i}} \log p_{t_i}^{\text{ref}}(\vx_{t_i})
    + \frac{1}{\alpha}\nabla_{\vx_{t_i}}V_{t_i}(\vx_{t_i}) \\
    \vs^\star_{t_i}(\vx_{t_i})
    &= \vs^{\text{ref}}_{t_i}(\vx_{t_i})
    + \frac{1}{\alpha}\nabla_{\vx_{t_i}} V_{t_i}(\vx_{t_i}).
    \label{eq:reward_optimal_score}
  \end{aligned}
\end{equation}
Finally, by comparing this to the definition $\vs^\star_{t_i}(\vx_{t_i}) = \vs^{\text{ref}}_{t_i}(\vx_{t_i}) + \textcolor{red}{\mathbf{\Psi}^\star_{t_i}}(\vx_{t_i})$, we conclude that:
\begin{equation}
    \textcolor{red}{\mathbf{\Psi}^\star_{t_i}}(\vx_{t_i}) = \frac{1}{\alpha}\nabla_{\vx_{t_i}}V_{t_i}(\vx_{t_i}).
\end{equation}

Next,  we will express the optimal Value Guidance Vector in relation to $\nabla_{\vx_{t_{i-1}}} V_{t_{i-1}}(\vx_{t_{i-1}})$.
Using Eq.~\eqref{eq:reverse_transition},
 the \emph{reference} one-step kernel at time $t_i$ can be represented by
\[
p^{\mathrm{ref}}_{t_i}(\vx_{t_{i-1}}\mid \vx_{t_i})
=\mathcal{N}\!\Big(\,\underbrace{\kappa(t_i)\vx_{t_i}+{\Omega(t_i)}\vs_{t_i}^\text{ref}(\vx_{t_i})}_{=:~\vmu^{\mathrm{ref}}_{t_i}(\vx_{t_i})}\;,\;{\sigma_{t_i}^2}\mathbf{I}\Big),
\]
Using the definition of Bellman equations and optimal policies~\cite{uehara2024understandingreinforcementlearningbasedfinetuning}, we obtain:
\begin{equation}
\label{eq:soft-opt-kernel}
p^\star_{t_i}(\vx_{t_{i-1}}\mid \vx_{t_i})
=
\frac{
\exp\!\big(V^\star_{t_{i-1}}(\vx_{t_{i-1}})/\alpha\big)\;
p^{\mathrm{ref}}_{t_i}(\vx_{t_{i-1}}\mid \vx_{t_i})
}{
\exp\big(V^\star_{t_{i}}(\vx_{t_{i}})/\alpha\big)
}.
\end{equation}
Substituting the Gaussian reference kernel into the preceding identity gives
\[
\log p^\star_{t_i}(\vx_{t_{i-1}}\mid \vx_{t_i})
=
\frac{V^\star_{t_{i-1}}(\vx_{t_{i-1}})}{\alpha}
-\frac{1}{2\sigma_{t_i}^2}\big(\vx_{t_{i-1}}-\vmu^{\mathrm{ref}}_{t_i}(\vx_{t_i})\big)^\top
\big(\vx_{t_{i-1}}-\vmu^{\mathrm{ref}}_{t_i}(\vx_{t_i})\big),
\]
up to an additive constant independent of $\vx_{t_{i-1}}$. Next,
we differentiate w.r.t.\ $\vx_{t_{i-1}}$:
\[
\nabla_{\vx_{t_{i-1}}}\log p^\star_{t_i}(\vx_{t_{i-1}}\mid \vx_{t_i})
=
\frac{1}{\alpha}\nabla_{\vx_{t_{i-1}}}V^\star_{t_{i-1}}(\vx_{t_{i-1}})
-\frac{1}{\sigma_{t_i}^2}\big(\vx_{t_{i-1}}-\vmu^{\mathrm{ref}}_{t_i}(\vx_{t_i})\big).
\]
Taking the conditional expectation under $p^\star_{t_i}(\cdot\mid \vx_{t_i})$ and using Lemma~\ref{thm:zero_mean_score}, we obtain
\[
\mathbf{0}
=
\frac{1}{\alpha}\,
\mathbb{E}_{\vx_{t_{i-1}}\sim p^\star_{t_i}(\cdot\mid \vx_{t_i})}
\!\left[\nabla_{\vx_{t_{i-1}}}V^\star_{t_{i-1}}(\vx_{t_{i-1}})\right]
-\frac{1}{\sigma_{t_i}^2}\Big(
\mathbb{E}_{\vx_{t_{i-1}}\sim p^\star_{t_i}(\cdot\mid \vx_{t_i})}[\vx_{t_{i-1}}]
-\vmu^{\mathrm{ref}}_{t_i}(\vx_{t_i})
\Big).
\]
Rearranging yields the conditional mean identity
\begin{equation}
\mathbb{E}_{p^\star_{t_i}}[\vx_{t_{i-1}}\mid \vx_{t_i}]
=
\vmu^{\mathrm{ref}}_{t_i}(\vx_{t_i})
+
\frac{\sigma_{t_i}^2}{\alpha}\,
\mathbb{E}_{\vx_{t_{i-1}}\sim p^\star_{t_i}(\cdot\mid \vx_{t_i})}
\!\left[\nabla_{\vx_{t_{i-1}}}V^\star_{t_{i-1}}(\vx_{t_{i-1}})\right].
\end{equation}
Matching the mean of $p_{t_i}^{\textcolor{blue}{\Psi}}(\cdot\mid \vx_{t_i})$, namely 
$\vmu^{\mathrm{ref}}_{t_i}(\vx_{t_i})+\Omega(t_i)\textcolor{blue}{\mathbf{\Psi}_{t_i}}(\vx_{t_i})$, with $\mathbb{E}_{p^\star_{t_i}}[\vx_{t_{i-1}}\mid \vx_{t_i}]$ yields:
\[\textcolor{red}{\mathbf{\Psi}^\star_{t_i}}(\vx_{t_i})
=
\frac{\sigma_{t_i}^2}{\alpha\Omega(t_i)}\,
\mathbb{E}_{\vx_{t_{i-1}}\sim p^\star_{t_i}(\cdot\mid \vx_{t_i})}
\!\left[\nabla_{\vx_{t_{i-1}}}V^\star_{t_{i-1}}(\vx_{t_{i-1}})\right].
\]
\end{proof}

\subsection{Proof of Proposition~\ref{prop:consistent_first_order}}
\label{appendix:formalize-first-order}

\consistentfirstorder*

\begin{proof}
For \(j=i\), Eq.~\eqref{eq:fo-consistent-estimator-compact} reduces to
\begin{equation}
\textcolor{ForestGreen}{\hat{\mathbf{\Psi}}_{t_i}^{\mathrm{FO,res}}(i)}
=
\frac{1}{\alpha}
\Big(
\gamma_{t_i}\nabla_{\vx_{t_i}} r(\hat \vx_{0 \mid t_i})
+
g_\phi(\vx_{t_i})
\Big)
=
\frac{1}{\alpha}\nabla_{\vx_{t_i}} V_{t_i}(\vx_{t_i})
=
\textcolor{red}{\mathbf{\Psi}_{t_i}^\star}(\vx_{t_i}),
\end{equation}
where the second equality follows from Eq.~\eqref{eq:fo-decomposition-prop}, and the third from Eq.~\eqref{eq:alternate-optimal-psi}.

For \(j=i-1\), by linearity of expectation and Eq.~\eqref{eq:fo-decomposition-prop},
\begin{equation}
\begin{aligned}
\mathbb{E}\!\left[
\textcolor{ForestGreen}{\hat{\mathbf{\Psi}}_{t_i}^{\mathrm{FO,res}}(i-1)}
\;\middle|\;
\vx_{t_i}
\right]
&=
\frac{\sigma_{t_i}^2}{\alpha \Omega(t_i)}
\mathbb{E}_{\vx_{t_{i-1}}\sim p^\star_{t_i}(\cdot \mid \vx_{t_i})}
\left[
\gamma_{t_i}\nabla_{\vx_{t_{i-1}}} r(\hat \vx_{0 \mid t_{i-1}})
+
g_\phi(\vx_{t_{i-1}})
\right] \\
&=
\frac{\sigma_{t_i}^2}{\alpha \Omega(t_i)}
\mathbb{E}_{\vx_{t_{i-1}}\sim p^\star_{t_i}(\cdot \mid \vx_{t_i})}
\left[
\nabla_{\vx_{t_{i-1}}} V_{t_{i-1}}(\vx_{t_{i-1}})
\right].
\end{aligned}
\end{equation}
Applying Proposition~\ref{cor:gen_soft_value} yields
\begin{equation}
\mathbb{E}\!\left[
\textcolor{ForestGreen}{\hat{\mathbf{\Psi}}_{t_i}^{\mathrm{FO,res}}(i-1)}
\;\middle|\;
\vx_{t_i}
\right]
=
\textcolor{red}{\mathbf{\Psi}_{t_i}^\star}(\vx_{t_i}),
\end{equation}
which proves the claim.
\end{proof}

Proposition~\ref{prop:consistent_first_order} formalizes the idealized condition required by residual-corrected first-order methods: if \(g_\phi\) exactly captures the Jensen gap introduced by the Tweedie-based approximation, then both current-state and one-step lookahead first-order guidance are unbiased.
However, in practice, Appendix~\ref{subsec:experiment-invalidate-auxiliary} shows that the learned residual term is negligible, suggesting that this condition is not realized empirically.
Thus, the Jensen gap
    $V_t(\vx_t) - r(\hat \vx_0(\vx_t))
    =
    \mathbb{E}[r(\vx_0)] - r(\mathbb{E}[\vx_0])$
is not meaningfully closed.

Under exact residual correction, Eq.~\eqref{eq:fo-decomposition-prop} recovers the intended $\mathcal{L}_\text{forward}$ of Residual $\nabla$-DB.
Choosing $g_\phi(\cdot)=\mathbf{0}$ yields SQDF.
$\gamma_t=1, g_\phi(\cdot)=\mathbf{0}$ returns to the case where we approximate the value function with only Tweedie's formula, as in DPS~\cite{chung2023diffusion}.

\subsection{Proof of Theorem~\ref{thm:stein-lemma}}
\label{appendix:formalize-stein}
We prove that the existing implementation of the KL-regularized REINFORCE estimator (derived in Section~\ref{appendix:subsec-zeroth-order}) is an unbiased estimator of the gradient of the \emph{hard} value function.
We clarify that the following derivation aligns with the practical reality that existing implementations of Soft PPO and GRPO \cite{fan2023dpok, liu2025flowgrpo} treat the KL penalty as additive intermediate rewards, rather than solving the exact log-sum-exp soft Bellman equations as suggested by~\citet{schulman2018equivalencepolicygradientssoft}.

\begin{definition}[Value Function of Diffusion MDP]
Let $r: \R^d \to \R$ be the terminal reward at $t = t_0$.
Given a transition function $f: \R^d \times \Z_{\ge 0} \to \R^d$ and a noise schedule $\sigma: \Z_{\ge 0} \to \R$, the definition of the value function $V: \R^d \times \Z_{\ge 0} \to \R$ under the expectation over the optimal policy is stated as:
\begin{align}
 \label{eq:soft-v-bellman-opt}
 V_{t_i}(\vx_{t_{i}}) &:= \mathbb{E}_{\vx_{t_{i-1:0}}\sim p_{t_{i:1}}^\star} [r(\vx_{t_0})]
\end{align}
\end{definition}

\policygradientzeroth*
\begin{proof}
When $\veps_{t_i}$ denotes the Gaussian noise injected in the transition process described in Eq.~\ref{eq:reverse_transition}, the value gradient estimate is given as:
\begin{align}
&\mathbb{E}_{\vx_{t_{i-1}}\sim p_{t_i}^\star}\left[\nabla_{\vx_{t_{i-1}}}V_{t_{i-1}}(\vx_{t_{i-1}})\right] \label{eq:thm37-line1} \\
&= \mathbb{E}_{\vx_{t_{i-1}}\sim p_{t_i}^\star}\left[\nabla_{\vx_{t_{i-1}}} \left(\mathbb{E}_{\vx_{t_{i-2:0}}\sim p_{t_{i-1:1}}^\star, \veps_{t_{i-1:1}} \overset{\mathrm{iid}}{\sim} \mathcal{N}(\mathbf{0}, \mathbf{I})} [r(\vx_{t_0})]\right)\right] \label{eq:thm37-line2}\\
&= \mathbb{E}_{\vx_{t_{i-1}}\sim p_{t_i}^\star, \veps_{t_i}\sim \mathcal{N}(\mathbf{0}, \mathbf{I})}\left[\frac{1}{\sigma_{t_i}}\nabla_{\veps_{t_i}} \left(\mathbb{E}_{\vx_{t_{i-2:0}}\sim p_{t_{i-1:1}}^\star, \veps_{t_{i-1:1}} \overset{\mathrm{iid}}{\sim} \mathcal{N}(\mathbf{0}, \mathbf{I})} [r(\vx_{t_0})]\right)\right] \\
&= \mathbb{E}_{\vx_{t_{i-1}}\sim p_{t_i}^\star, \veps_{t_i}\sim \mathcal{N}(\mathbf{0}, \mathbf{I})}\left[\frac{1}{\sigma_{t_i}} \left(\mathbb{E}_{\vx_{t_{i-2:0}}\sim p_{t_{i-1:1}}^\star, \veps_{t_{i-1:1}} \overset{\mathrm{iid}}{\sim} \mathcal{N}(\mathbf{0}, \mathbf{I})} [r(\vx_{t_0})\veps_{t_i}]\right)\right]. \label{eq:thm37-line4} 
\end{align}
The last line comes from a special case of Stein's identity~\cite{Stein1981}, which states that for a differentiable function $h$,
\begin{equation}
\label{eq:steins-identity3}
\mathbb{E}_{\veps \sim \mathcal{N}(\mathbf{0}, \mathbf{I})}[h(\veps)\veps] = \mathbb{E}_{\veps \sim \mathcal{N}(0, \mathbf{I})}[\nabla_\veps h(\veps)].
\end{equation}

Multiplying Eqs.~\eqref{eq:thm37-line1},~\eqref{eq:thm37-line4} by $\frac{\sigma_{t_i}^2}{\alpha \Omega(t_i)}$ completes the proof:
\begin{align}
    \mathbb{E}_{\vx_{t_{i-1}}\sim p_{t_{i}}^\star}[\frac{\sigma_{t_i}^2}{\alpha \Omega(t_i)}\nabla_{\vx_{t_{i-1}}} V_{t_{i-1}}(\vx_{t_{i-1}})] = \frac{\sigma_{t_i}}{\alpha \Omega(t_i)}\mathbb{E}_{\vx_{t_{i-1:0}}\sim p_{t_{i:1}}^\star, \veps_{t_{i:1}} \overset{\mathrm{iid}}{\sim} \mathcal{N}(\mathbf{0}, \mathbf{I})}[r(\vx_0)\veps_{t_i}].
\end{align}
\end{proof}

\subsection{Hard versus Soft Value Gradients in Practical Implementations}
\label{appendix:hard-soft-value}

Sections~\ref{appendix:formalize-first-order} and~\ref{appendix:formalize-stein}
show that full-rollout estimators can provide unbiased estimates of a value gradient.
The relevant point, however, is which value function is being differentiated.
In the main text, we show that the existing diffusion fine-tuning methods are unified to admit $\textcolor{ForestGreen}{\hat{\mathbf{\Psi}}_{t_i}}$ which denotes an estimator of the gradient of the
\emph{hard} value function, namely the expected unregularized terminal reward
\begin{equation}
V^{\rm hard}_{t_i}(\vx_{t_i})
:=
\mathbb{E}_{p^\star}
\!\left[
r(\vx_{t_0}) \mid \vx_{t_i}
\right].
\end{equation}
To get this value-gradient surrogate of many practical diffusion fine-tuning methods, the rollout estimator uses the terminal reward $r(\vx_{t_0})$, while the KL regularization against the reference model is introduced as a separate timestep-wise penalty or score correction.

However, this convention differs from the fully entropy-regularized, or \emph{soft}, value function from soft RL.
For a given policy $p_{t_{i:1}}$, define the entropy-augmented return from $\vx_{t_i}$ by
\begin{equation}
\tilde r_{t_i}
:=
r(\vx_{t_0})
-
\alpha
\sum_{k=1}^{i}
D_{\rm KL}
\!\left(
p_{t_k}(\cdot \mid \vx_{t_k})
\Vert
p^{\rm ref}_{t_k}(\cdot \mid \vx_{t_k})
\right).
\end{equation}
The corresponding soft value function is
\begin{equation}
V^{\rm soft}_{t_i}(\vx_{t_i})
:=
\max_{p_{t_{i:1}}}
\mathbb{E}
\!\left[
\tilde r_{t_i}
\mid
\vx_{t_i}
\right],
\end{equation}
where the expectation is taken over the reverse trajectory induced by $p_{t_{i:1}}$.
Equivalently, if $p^\star_{t_{i:1}}$ denotes an optimizer of the entropy-regularized objective, then
\begin{equation}
V^{\rm soft}_{t_i}(\vx_{t_i})
=
\mathbb{E}_{p^\star_{t_{i:1}}}
\!\left[
\tilde r_{t_i}
\mid
\vx_{t_i}
\right].
\end{equation}
This is the diffusion analogue of the entropy-augmented value function in soft RL, where the return itself contains the KL costs rather than adding only an instantaneous KL-gradient term~\cite{schulman2018equivalencepolicygradientssoft}.
Under the same regularity assumptions used in Appendix~\ref{appendix:formalize-stein}, applying Stein's identity to the entropy-augmented return yields the soft-value analogue of the full-rollout estimator (see Appendix~\ref{appendix:completeness-soft-version}):
\begin{equation}
\frac{\sigma_{t_i}}{\alpha \Omega(t_i)}
\mathbb{E}_{\vx_{t_{i-1:0}}\sim p_{t_{i:1}}^\star, \veps_{t_{i:1}} \sim \mathcal{N}(\mathbf{0}, \mathbf{I})}
\!\left[
\tilde r_{t_i}\veps_{t_i}
\mid
\vx_{t_i}
\right]
=
\frac{1}{\alpha}
\nabla_{\vx_{t_i}}
V^{\rm soft}_{t_i}(\vx_{t_i})
=:
\textcolor{red}{\mathbf{\Psi}^{\star}_{t_i}(\vx_{t_i})}.
\label{eq:soft_value_zo_estimate}
\end{equation}
Thus, a fully soft-value-consistent zeroth-order method would replace the terminal reward
$r(\vx_{t_0})$ in the rollout estimator by the entropy-augmented return $\tilde r_{t_i}$.

It can be shown that Policy gradient methods using the entropy-augmented reward $\tilde r_{t_i}$ admit an unbiased zeroth-order estimator for the soft value gradient in Eq.~\eqref{eq:soft_value_zo_estimate}, equivalent to that of soft Q-learning methods.
On the other hand, estimators that add only a local KL-gradient term, which is a widely-used surrogate of many practical diffusion fine-tuning methods, are not identical to estimators in Eq.~\eqref{eq:soft_value_zo_estimate} because they do not account for how earlier actions affect future KL costs\cite{schulman2018equivalencepolicygradientssoft}.
More formally, the surrogate of existing implementations unified by RSM corresponds to the form used in many existing diffusion fine-tuning implementations:
\begin{equation}
-\frac{1}{\alpha}
\nabla_{\vx_{t_i}}
V^{\rm hard}_{t_i}(\vx_{t_i})
+
\left(
\vs^\star_{t_i}(\vx_{t_i})
-
\vs^{\rm ref}_{t_i}(\vx_{t_i})
\right).
\label{eq:hard_value_plus_kl_surrogate}
\end{equation}
This expression is closely related to, but not algebraically identical to,
\begin{equation}
-\frac{1}{\alpha}
\nabla_{\vx_{t_i}}
V^{\rm soft}_{t_i}(\vx_{t_i}).
\end{equation}

The difference is not a failure of the RSM reduction; rather, it specifies the exact surrogate that the reduction describes. Since our goal is to unify the objectives used by \emph{existing} fine-tuning methods (see Table~\ref{tab:unified-table}), the main text focuses on the hard-value estimator with a separately added KL correction. A fully soft-value-consistent variant can be obtained by using $\tilde r_{t_i}$ inside the value-gradient estimator, as in Eq.~\eqref{eq:soft_value_zo_estimate}. We leave a systematic empirical study of when this distinction matters to future work.

\subsubsection{Soft-Value Estimator}
\label{appendix:completeness-soft-version}

For completeness, we now formalize the soft value gradient estimator.
This is a direct analogue of Theorem~\ref{thm:stein-lemma}, with the terminal reward replaced by the entropy-augmented return.

\begin{restatable}[Soft-value analogue of the zeroth-order estimator]{theorem}{softvaluestein}
\label{thm:stein-lemma-soft-value}
Assume the regularity conditions of Theorem~\ref{thm:stein-lemma}, and assume that the optimizer
$p^\star_{t_{i:1}}$ of the entropy-regularized objective satisfies the standard envelope-theorem conditions.
Then
\begin{equation}
\frac{\sigma_{t_i}}{\alpha \Omega(t_i)}
\mathbb{E}_{\vx_{t_{i-1:0}}\sim p_{t_{i:1}}^\star, \veps_{t_{i:1}} \sim \mathcal{N}(\mathbf{0}, \mathbf{I})}
\!\left[
\tilde r_{t_i}\veps_{t_i}
\mid
\vx_{t_i}
\right]
=
\frac{1}{\alpha}
\nabla_{\vx_{t_i}}
V^{\rm soft}_{t_i}(\vx_{t_i}).
\end{equation}
\end{restatable}

\begin{proof}
By definition,
\begin{equation}
V^{\rm soft}_{t_i}(\vx_{t_i})
=
\max_{p_{t_{i:1}}}
\mathbb{E}_{\vx_{t_{i-1:0}}\sim p_{t_{i:1}}}
\!\left[
\tilde r_{t_i}
\mid
\vx_{t_i}
\right].
\end{equation}
Let $p^\star_{t_{i:1}}$ be an optimizer. By the envelope theorem~\cite{Milgrom_Segal_envelope}, the derivative of the optimized value
with respect to $\vx_{t_i}$ is obtained by differentiating the objective at $p^\star_{t_{i:1}}$, without
including the derivative of the optimizer itself. Hence,
\begin{equation}
\nabla_{\vx_{t_i}}
V^{\rm soft}_{t_i}(\vx_{t_i})
=
\nabla_{\vx_{t_i}}
\mathbb{E}_{\vx_{t_{i-1:0}}\sim p^\star_{t_{i:1}}}
\!\left[
\tilde r_{t_i}
\mid
\vx_{t_i}
\right].
\end{equation}
The remaining argument is identical to the derivation of Theorem~\ref{thm:stein-lemma} from Eq.~\eqref{eq:thm37-line2} to Eq.~\eqref{eq:thm37-line4}, using Stein's identity as a key bridge. In particular, under the one-step Gaussian perturbation used there,
differentiating the conditional expectation with respect to $\vx_{t_i}$ can be written as a Gaussian
score-function identity, giving
\begin{equation}
\nabla_{\vx_{t_i}}
V^{\rm soft}_{t_i}(\vx_{t_i})
=
\frac{\sigma_{t_i}}{\Omega(t_i)}
\mathbb{E}_{p_{t_{i:1}}^\star, \veps_{t_{i:1}} \sim \mathcal{N}(\mathbf{0}, \mathbf{I})}
\!\left[
\tilde r_{t_i}\veps_{t_i}
\mid
\vx_{t_i}
\right].
\end{equation}
Multiplying both sides by $1/\alpha$ gives the claim.
\end{proof}

\subsection{Convergence of Policy Iteration}  \label{sec:convergence}
Section~\ref{sec:theory} and Appendix~\ref{appendix:core} characterize the reward-guided target and its associated optimal guidance as fixed-point identities of the optimum $p^\star$.
In practice, however, RL-based fine-tuning methods sample from the current policy $p^\theta$ and update that policy using the RSM loss.
We therefore rely on prior policy optimization theory for convergence guarantees.
TRPO~\cite{schulmantrpo2015} was originally derived from a monotonic improvement perspective, and subsequent theory established that exact TRPO converges to the optimal policy in entropy-regularized MDPs~\cite{neuentropymdp2017}.
More recent results established global convergence guarantees for PPO-Clip and its variants under neural function approximation~\cite{huangppoclip}.
Accordingly, our core theory specifies the fixed point that the reward score matching update is designed to target, while convergence of the resulting iterative optimization procedure is justified by prior works rather than proven in this paper.

\section{Proof of Theorem~\ref{theorem:rsm}}
\label{appendix:deriving-psi}

\rsm*
\begin{proof}
    The proof for each method is laid throughout Appendix~\ref{appendix:deriving-psi}.
\end{proof}

\textbf{Technical remark ($\alpha = 0$).}
SQDF and policy-gradient methods are well-defined even when $\alpha=0$.
$\textcolor{blue}{\mathbf{\Psi}_{t_i}}$ and $C_2(t_i)$ appear singular due to $1/\alpha$, but these quantities appear only through products $C_1(t_i)\textcolor{blue}{\mathbf{\Psi}_{t_i}}$ and $C_1(t_i)C_2(t_i)$ in Eq.~\eqref{eq:canonical-gradient-form}. The factor of $\alpha$ in $C_1(t_i)$ cancels the apparent singularity.
Therefore, when $\alpha=0$, the canonical gradient remains the correct object for interpreting optimization dynamics, although the literal decomposition from Eq.~\eqref{eq:optimal-decomposition} no longer strictly applies.

\subsection{First-order Methods}
\label{appendix:first-order-reduction}

\subsubsection{\texorpdfstring{Residual $\nabla$-DB}{Residual Nabla-DB}}
\label{appendix:residual_DB}
The relationship between Residual $\nabla$-DB and Soft RL has been briefly discussed in Appendix D of \citet{liu2025resnabladb}.
We rewrite the objective by mapping forward and backward GFlowNet processes
$(P_F,P_B)$ to reverse sampling and forward noising, respectively.
\begin{align}
\intertext{Residual $\nabla$-DB minimizes a weighted sum of four terms:}
    \label{eq:total-loss}
    \mathcal{L}_\text{total}(\theta, \phi) 
    &= w_F \mathcal{L}_\text{forward}(\theta, \phi) 
     + w_B \mathcal{L}_\text{backward}(\theta, \phi) 
     + \mathcal{L}_\text{terminal}(\phi) 
     + w_R \mathcal{L}_\text{reg}(\theta),
    \\
    \label{eq:forward-loss}
    \mathcal{L}_\text{forward}(\theta, \phi)
    &= \mathbb{E}_{t_i, \vx_{t_i}, \veps_{t_i}} \Big[ \Big\| 
        \nabla_{\vx_{t_{i-1}}} \ell_{t_i}^\theta(\vx_{t_{i-1}},\vx_{t_i})
        - \mathcal{G}_\text{GFN\_F}(\vx_{t_i}, \veps_{t_i})
    \Big\|^2 \Big], \\
    \label{eq:backward-loss}
    \mathcal{L}_\text{backward}(\theta, \phi)
    &= \mathbb{E}_{t_i, \vx_{t_i}, \veps_{t_i}} \Big[\Big\| 
        \nabla_{\vx_{t_i}}\ell_{t_i}^\theta(\vx_{t_{i-1}},\vx_{t_i})
        + \mathcal{G}_\text{GFN\_B}(\vx_{t_i})
    \Big\|^2\Big], \\
    \label{eq:terminal-loss}
    \mathcal{L}_\text{terminal}(\phi)
    &= \mathbb{E}_{\vx_0}[ \| g_\phi(\vx_0) \|^2], \\
    \label{eq:reg-loss}
    \mathcal{L}_\text{reg}(\theta)
    &= \mathbb{E}_{t_i, \vx_{t_i}} [ \| \veps_{t_i}^\theta - \veps_{t_i}^{\theta^\dagger} \|^2] 
    = \mathbb{E}_{t_i, \vx_{t_i}} [ (1 - \bar{\alpha}_{t_i}) \| \vs_{t_i}^\theta - \vs_{t_i}^{\theta^\dagger} \|^2],
\intertext{where $\ell$, gradient fields $\mathcal{G}$ and Gaussian reparameterization are:}
    \label{eq:ell-ratio}
    \ell_{t_i}^\theta(\vx_{t_{i-1}},\vx_{t_i})
    &:= \log p^\theta(\vx_{t_{i-1}}\mid \vx_{t_i})
     - \log p^{\mathrm{ref}}(\vx_{t_{i-1}}\mid \vx_{t_i}),
    \\
    \label{eq:aux-gf}
    \mathcal{G}_\text{GFN\_F}(\vx_{t_i}, \veps_{t_i}) 
    &= \frac{1}{\alpha} \tilde\gamma_{t_i} \sg  ( \nabla_{\vx_{t_{i-1}}} r(\hat\vx_{0|t_{i-1}})) + g_\phi(\vx_{t_{i-1}}), \\
    \label{eq:aux-gb}
    \mathcal{G}_\text{GFN\_B}(\vx_{t_i}) 
    &= \frac{1}{\alpha} \tilde\gamma_{t_i} \sg ( \nabla_{\vx_{t_{i}}} r(\hat\vx_{0|t_{i}}) )+ g_\phi(\vx_{t_i}), \\
    \label{eq:sampling}
    \vx_{t_{i-1}} &= \pmb\mu_{t_i}^\theta + \sigma_{t_i} \veps_{t_i}, \quad \veps_{t_i} \sim \mathcal{N}(\mathbf{0}, \mathbf{I}). \nonumber
\end{align}

\textbf{Pruning $\mathcal{L}_\text{backward}$ and $\mathcal{L}_\text{terminal}$.}
$\mathcal{L}_\text{backward}$ requires the gradient \(\nabla_{\vx_{t_i}}\ell_{t_i}^\theta(\vx_{t_{i-1}},\vx_{t_i})\).
Under the Gaussian reparameterization, \(\vx_{t_{i-1}}\) depends on \(\vx_{t_i}\) through the model output, which introduces Jacobian terms \(\partial \veps_{t_i}^\theta / \partial \vx_{t_i}\) into \(\nabla_{\vx_{t_i}}\ell_{t_i}^\theta\).
Consequently, \(\nabla_\theta \mathcal{L}_\text{backward}\) involves higher-order derivatives, including mixed \(\theta\)-\(\vx\) terms, rendering optimization numerically unstable and empirically ineffective.
This is consistent with our ablation results in Appendix~\ref{subsec:backward_redundant}; removing \(\mathcal{L}_\text{backward}\) does not degrade performance, whereas removing \(\mathcal{L}_\text{forward}\) causes training to fail. 

$\mathcal{L}_\text{terminal}$ (Eq.~\eqref{eq:terminal-loss}) is not parameterized by $\theta$ and therefore contributes no gradient to the canonical loss term in Theorem~\ref{theorem:rsm}.
Moreover, as discussed in Appendix~\ref{subsec:refinement_redundant}, the outputs of $g_\phi$ are negligible in magnitude and do not affect training dynamics.

\textbf{Reduction to Reward Score Matching.}
\(\nabla_{\vx_{t_{i-1}}}\ell_{t_i}^\theta\) can be expressed as a scaled score difference:
\begin{equation}
    \nabla_{\vx_{t_{i-1}}}\ell_{t_i}^\theta(\vx_{t_{i-1}}, \vx_{t_i}) 
    = \frac{1}{\sigma_{t_i}^2}(\pmb\mu_{t_i}^\theta - \pmb\mu_{t_i}^\text{ref})
    =\frac{\Omega(t_i)}{\sigma_{t_i}^2}(\vs_{t_i}^\theta - \vs_{t_i}^\text{ref}).
\end{equation}
\begin{equation}
\label{eq:residual_nabla_db_forward_simplified}
\therefore \mathcal{L}_\text{forward}(\theta,\phi)
=
\mathbb{E}_{t_i,\vx_{t_i},\veps_{t_i}}
\Bigg[
\frac{\Omega(t_i)^2}{\sigma_{t_i}^4}
\bigg\|
\vs_{t_i}^\theta
-
\big(
\vs_{t_i}^\text{ref} + \frac{\sigma_{t_i}^2}{\Omega(t_i)}\,\mathcal{G}_\text{GFN\_F}(\vx_{t_i},\veps_{t_i})
\big)
\bigg\|^2
\Bigg].
\end{equation}
Thus, $w_F \mathcal{L}_\text{forward} + w_R \mathcal{L}_\text{reg}$ is a special case of the canonical loss, with
\begin{equation}
\textcolor{blue}{\mathbf{\Psi}_{t_i}} = \tilde\gamma_{t_i}\frac{\sigma_{t_i}^2 }{\alpha\Omega(t_i)} \sg (\nabla_{\vx_{t_{i-1}}} r(\hat\vx_{0|t_{i-1}}) )
,
\quad C_1(t_i) =
\frac{1}{d}
\frac{\Omega(t_i)^2}{\sigma_{t_i}^4},
\quad C_2(t_i) = \frac{(1 - \bar \alpha_{t_i})\sigma_{t_i}^4}{\Omega(t_i)^2}\frac{w_R}{w_F},
\end{equation}
where the factor \(\frac{1}{d}\) arises implicitly from a \texttt{torch.mean()} operation over tensor dimensions.

\subsubsection{SQDF}
\label{appendix:sqdf}
SQDF is motivated by the following objective, given in Eq.~(11) of \citet{kang2026sqdf}:
\begin{equation}
    \mathcal{L}(\theta) = \mathbb{E}_{t_i, \vx_{t_i}, \veps_{t_i}}\big[-r(\hat \vx_{0|t_{i-1}}) + \alpha \mathcal{D}_\text{KL}(p^\theta(\vx_{t_{i-1}} \mid \vx_{t_i}) \| p^\text{ref}(\vx_{t_{i-1}} \mid \vx_{t_i}))\big].
\end{equation}
In the actual implementation, however, the transition mean $\pmb\mu_{t_{i-1}}^\theta$ is reparameterized, and a $\text{stopgrad}$ operation is applied to the Tweedie estimate inside the reward term.
In addition, an attenuation factor $\tilde\gamma_{t_i} = \tilde\gamma_\text{SQDF}^{Nt_i}$ is introduced in the reward term for temporal credit assignment.
This yields:
\begin{equation}
    \mathcal{L}(\theta) = \mathbb{E}_{t_i, \vx_{t_i}, \veps_{t_i}}\Big[- \tilde\gamma_{t_i} \ \sg (\nabla_{\vx_{t_{i-1}}} r(\hat\vx_{0|t_{i-1}}) ) \cdot \pmb\mu_{t_{i-1}}^\theta + \frac{\alpha}{2\sigma_{t_i}^2} \big\| \pmb\mu_{t_{i-1}}^\theta - \pmb\mu_{t_{i-1}}^\text{ref} \big\|^2 \Big].
\end{equation}
Substituting
$\pmb\mu_{t_{i-1}}^\theta - \pmb\mu_{t_{i-1}}^\text{ref} = \Omega(t_i)(\vs_{t_i}^\theta - \vs_{t_i}^\text{ref})$
yields
\begin{equation}
    \mathbb{E} \bigg[
    - \tilde\gamma_{t_i} \ \sg (\nabla_{\vx_{t_{i-1}}} r(\hat\vx_{0|t_{i-1}}) )
    \cdot
    \big(\kappa(t_i)\vx_{t_i} + \Omega(t_i)\vs_{t_i}^\theta \big)
    + \frac{\alpha\Omega(t_i)^2}{2\sigma_{t_i}^2}
    \big\| \vs_{t_i}^\theta - \vs_{t_i}^\text{ref}  \big\|^2
    \bigg].
\end{equation}
Up to additive terms that do not contribute gradients, this is equivalent to:
\begin{equation}
    \mathbb{E}_{t_i, \vx_{t_i}, \veps_{t_i}}
    \left[ \frac{\alpha\Omega(t_i)^2}{2\sigma_{t_i}^2}
    \left\|
    \vs_{t_i}^\theta - \left(\vs_{t_i}^\text{ref} + \tilde\gamma_{t_i}\frac{\sigma_{t_i}^2}{\alpha\Omega(t_i)} \ \sg (\nabla_{\vx_{t_{i-1}}} r(\hat\vx_{0|t_{i-1}}) ) \right) \right\|^2 \right].
\end{equation}
Therefore, SQDF is a special case of the canonical loss, with
\begin{equation}
\textcolor{blue}{\mathbf{\Psi}_{t_i}} = \tilde\gamma_{t_i}\frac{\sigma_{t_i}^2}{\alpha\Omega(t_i)}  \ \sg (\nabla_{\vx_{t_{i-1}}} r(\hat\vx_{0|t_{i-1}}) ),
\quad C_1(t_i) =
\frac{\alpha}{2}
\frac{\Omega(t_i)^2}{\sigma_{t_i}^2},
\quad C_2(t_i) = 0.
\end{equation}

\subsubsection{VGG-Flow}
We rewrite VGG-Flow objectives, simplified using $G^\phi(\vx_{t_i}) \triangleq -\tilde\gamma_{t_i} \ \text{sg}\big(\nabla_{\vx_{t_i}}r(\hat \vx_{0 | t_i})\big) - g^\phi(\vx_{t_i})$.
\begin{align}
    \mathcal{L}_\text{consistency}(\phi) &= \mathbb{E}_{t_i, \vx_{t_i}} \Big\| \frac{\partial}{\partial t}G^\phi(\vx_{t})\Big|_{t = t_i} + [\nabla G^\phi(\vx_{t_i})]^\top(\vv_{t_i}^\text{ref} \textcolor{red}{+} \frac{1}{\lambda} G^\phi(\vx_{t_i})) + [\nabla_{\vx_{t_i}} \vv_{t_i}^\text{ref}]^\top G^\phi(\vx_{t_i}) \Big\|^2, \\
    \mathcal{L}_\text{boundary}(\phi) &= \mathbb{E}_{\vx_0} \| g^\phi(\vx_0)\|^2, \\
    \mathcal{L}_\text{matching}(\theta) &= \mathbb{E}_{\vx_{t_i}}\Big[\|\vv_{t_i}^\theta - \vv_{t_i}^\text{ref} \textcolor{red}{-} \frac{1}{\alpha} [-\tilde\gamma_{t_i} \sg ((\nabla_{\vx_{t_i}} r(\hat{\vx}_{0 \mid t_{i}}^\theta) ) - g^\phi(\vx_{t_i})]\|^2\Big].
\end{align}
Here,
$\mathcal{L}_\text{total}(\theta, \phi) = \mathcal{L}_\text{matching}(\theta) + \mathcal{L}_\text{consistency}(\phi) + w_b \mathcal{L}_\text{boundary}(\phi)$.\footnote{The original paper and our manuscript use opposite time conventions ($0\to1$ vs.\ $1\to0$). To match our notation, we flip the signs highlighted in \textcolor{red}{red} relative to the original formulation.}

As with Residual $\nabla$-DB, and as discussed in the main text, $\mathcal{L}_\text{boundary}$ effectively enforces $g^\phi(\vx_{t_i}) \approx \mathbf{0}$ for all $t_i$.
After dropping $g^\phi$, we analyze the dominant term, $\mathcal{L}_\text{matching}(\theta)$, by mapping the velocity field $\vv_{t_i}$ to our unified score parameterization $\vs_{t_i}$ under the general SDE formulation, specifically using the $\vs_{t_i}$ form without $\Omega(t_i)$.

Using $\delta(t_i)$ for $\vv$-parameterized Rectified Flow (Eq.~\eqref{eq:delta-v-rectifedflow}), $\mathcal{L}_\text{matching}$ can be written as:
\begin{equation}
\begin{aligned}
    \mathcal{L}_\text{matching}(\theta) &= \mathbb{E}_{\vx_{t_i}}\Big[\Big\|\frac{1}{\delta(t_i)}(\vs_{t_i}^\theta - \vs_{t_i}^\text{ref}) \textcolor{red}{-} \frac{1}{\alpha} \tilde\gamma_{t_i} \sg (\nabla_{\vx_{t_i}} r(\hat{\vx}_{0 \mid t_{i}}^\theta) ) \Big\|^2\Big] \\
    &= \mathbb{E}_{\vx_{t_i}}\Big[ \frac{1}{\delta(t_i)^2} \Big\|\vs_{t_i}^\theta - \Big(\vs_{t_i}^\text{ref} \textcolor{red}{+} \frac{\delta(t_i)}{\alpha} \tilde\gamma_{t_i} \sg (\nabla_{\vx_{t_i}} r(\hat{\vx}_{0 \mid t_{i}}^\theta) ) \Big) \Big\|^2 \Big].
\end{aligned}
\end{equation}

Therefore, VGG-Flow is a special case of the canonical loss, with
\begin{equation}
\textcolor{blue}{\mathbf{\Psi}_{t_i}} = \frac{\tilde\gamma_{t_i}}{\alpha} \delta(t_i) \ \sg ((\nabla_{\vx_{t_i}} r(\hat{\vx}_{0 \mid t_{i}}^\theta) ),
\quad C_1(t_i) =
\frac{1}{d}
\frac{1}{\delta(t_i)^2},
\quad C_2(t_i) = 0,
\end{equation}
where the factor $\frac{1}{d}$ arises implicitly from a \texttt{torch.mean()} operation over tensor dimensions, rather than from an explicit scalar normalization.

\subsubsection{DRaFT, DRaFT-$K$}
\label{appendix:draft}

DRaFT~\cite{clark2024draft} directly differentiates the terminal reward through the denoising trajectory.
Consider its objective with the optional additive KL regularization described in Appendix B.8 of \citet{clark2024draft}.
Using the canonical-gradient representation above, we have
\begin{equation}
\mathcal{G}(\vx_{t_i})
=
\underbrace{\frac{\alpha}{\delta(t_i)^2}}_{C_1(t_i)}
\left(
-
\underbrace{\frac{\delta(t_i)^2}{2\alpha}
\nabla_{\vs_{t_i}^\theta} r(\vx_0)}_{\textcolor{blue}{\mathbf{\Psi}_{t_i}}}
+
(\vs_{t_i}^\theta-\vs_{t_i}^{\mathrm{ref}})
\right),
\label{eq:draft-canonical-gradient}
\end{equation}
with $C_2(t_i)=0$.

Define the reward gradient propagated through the remaining denoising trajectory as
\begin{equation}
\lambda_{t_{i-1}}
\triangleq
\nabla_{\vx_{t_{i-1}}}r(\vx_0).
\end{equation}
Since the unified transition parameterization gives
\begin{equation}
\frac{\partial \vx_{t_{i-1}}}
{\partial \vs_{t_i}^{\theta}}
=
\Omega(t_i)\mathbf I,
\end{equation}
we obtain
\begin{equation}
\nabla_{\vs_{t_i}^{\theta}}r(\vx_0)
=
\Omega(t_i)\lambda_{t_{i-1}},
\end{equation}
and hence
\begin{equation}
\textcolor{blue}{\mathbf{\Psi}_{t_i}}
=
\frac{\delta(t_i)^2\Omega(t_i)}{2\alpha}
\lambda_{t_{i-1}}.
\label{eq:draft-effective-guidance}
\end{equation}

DRaFT differentiates through every subsequent denoising step, so $\lambda_{t_{i-1}}$ is the full-lookahead reward gradient.
Using Definition~\ref{def:lookahead_estimators},
\begin{equation}
\textcolor{ForestGreen}{
\hat{\mathbf{\Psi}}_{t_i}^{\mathrm{LA},1}(0,1)}
=
\frac{\sigma_{t_i}^2}{\alpha\Omega(t_i)}
\lambda_{t_{i-1}},
\end{equation}
and Eq.~\eqref{eq:draft-effective-guidance} can therefore be written as
\begin{equation}
\textcolor{blue}{\mathbf{\Psi}_{t_i}}
=
\underbrace{
\frac{\delta(t_i)^2\Omega(t_i)^2}
{2\sigma_{t_i}^2}
}_{\gamma(t_i)}
\textcolor{ForestGreen}{
\hat{\mathbf{\Psi}}_{t_i}^{\mathrm{LA},1}(0,1)}.
\end{equation}
Thus, DRaFT is a special case of RSM with
\begin{equation}
\textcolor{ForestGreen}{
\hat{\mathbf{\Psi}}_{t_i}}
=
\hat{\mathbf{\Psi}}_{t_i}^{\mathrm{LA},1}(0,1),
\qquad
\gamma(t_i)
=
\frac{\delta(t_i)^2\Omega(t_i)^2}{2\sigma_{t_i}^2},
\qquad
C_1(t_i)=\frac{\alpha}{\delta(t_i)^2},
\qquad
C_2(t_i)=0.
\label{eq:draft-rsm}
\end{equation}

For deterministic DDIM sampling, $\sigma_{t_i}=0$, so
$\gamma(t_i)$ and
$\hat{\mathbf{\Psi}}_{t_i}^{\mathrm{LA},1}(0,1)$
should not be interpreted separately; their product in
Eq.~\eqref{eq:draft-effective-guidance} remains well-defined.
This is analogous to the $\alpha=0$ interpretation discussed at the beginning of Appendix~\ref{appendix:deriving-psi}.

DRaFT-$K$ truncates the Jacobian chain by applying
$\vx_{t_K}\leftarrow\operatorname{stopgrad}(\vx_{t_K})$
and differentiating only through the final $K$ denoising steps.
Under our time convention, it therefore coincides with DRaFT for
$1\leq i\leq K$ and has zero reward guidance for $i>K$:
\begin{equation}
\textcolor{blue}{\mathbf{\Psi}_{t_i}^{\text{DRaFT-}K}}
=
\mathbf 1[i\leq K]\,
\gamma(t_i)
\textcolor{ForestGreen}{
\hat{\mathbf{\Psi}}_{t_i}^{\mathrm{LA},1}(0,1)}.
\label{eq:draft-k-rsm}
\end{equation}
Here, $K$ denotes the backpropagation horizon of DRaFT-$K$ and is distinct from the branching width $K_i$ in Definition~\ref{def:lookahead_estimators}.
DRaFT provides an unbiased full-lookahead first-order estimator, but requires backpropagation through the full Jacobian chain and consequently incurs prohibitive memory cost at scale; DRaFT-$K$ trades this consistency for reduced memory usage.

\subsubsection{ReFL}
\label{appendix:refl}

ReFL~\cite{ImageReward} directly differentiates the reward through a current-state Tweedie prediction.
For VP-SDE, ReFL samples an intermediate timestep $t_i$ and computes
\begin{equation}
\hat{\vx}_{0\mid t_i}
=
\frac{
\operatorname{stopgrad}(\vx_{t_i})
-
\frac{1}{\delta(t_i)}\veps_{t_i}^{\theta}
}{a_{t_i}}
=
\frac{
\operatorname{stopgrad}(\vx_{t_i})
+
\frac{1}{\delta(t_i)^2}\vs_{t_i}^{\theta}
}{a_{t_i}},
\qquad
a_{t_i}=\sqrt{\bar\alpha_{t_i}}.
\label{eq:refl-tweedie}
\end{equation}
The stop-gradient operation removes the dependence of the reward on
$\vx_{t_i}$, yielding
\begin{equation}
\frac{\partial \hat{\vx}_{0\mid t_i}}
{\partial \vs_{t_i}^{\theta}}
=
\frac{1}{a_{t_i}\delta(t_i)^2}\mathbf I.
\label{eq:refl-jacobian}
\end{equation}
Therefore,
\begin{equation}
\nabla_{\vs_{t_i}^{\theta}}
r(\hat{\vx}_{0\mid t_i})
=
\frac{1}{a_{t_i}\delta(t_i)^2}
\nabla_{\hat{\vx}_{0\mid t_i}}
r(\hat{\vx}_{0\mid t_i}).
\label{eq:refl-score-gradient}
\end{equation}

Substituting Eq.~\eqref{eq:refl-score-gradient} into the canonical gradient gives
\begin{equation}
\mathcal{G}(\vx_{t_i})
=
\underbrace{\frac{\alpha}{\delta(t_i)^2}}_{C_1(t_i)}
\left(
-
\frac{1}{2\alpha a_{t_i}}
\nabla_{\hat{\vx}_{0\mid t_i}}
r(\hat{\vx}_{0\mid t_i})
+
(\vs_{t_i}^{\theta}-\vs_{t_i}^{\mathrm{ref}})
\right).
\label{eq:refl-canonical-gradient}
\end{equation}

Using the current-state pullback notation of
Definition~\ref{def:current_state_estimator},
\begin{equation}
\textcolor{ForestGreen}{
\hat{\mathbf{\Psi}}_{t_i}^{\mathrm{CS},1}[\mathbf I]}
=
\frac{1}{\alpha}
\nabla_{\hat{\vx}_{0\mid t_i}}
r(\hat{\vx}_{0\mid t_i}),
\end{equation}
so the effective guidance in Eq.~\eqref{eq:refl-canonical-gradient} is
\begin{equation}
\textcolor{blue}{\mathbf{\Psi}_{t_i}}
=
\underbrace{\frac{1}{2a_{t_i}}}_{\gamma(t_i)}
\textcolor{ForestGreen}{
\hat{\mathbf{\Psi}}_{t_i}^{\mathrm{CS},1}[\mathbf I]}.
\end{equation}
Thus, ReFL is a special case of RSM with
\begin{equation}
\textcolor{ForestGreen}{
\hat{\mathbf{\Psi}}_{t_i}}
=
\hat{\mathbf{\Psi}}_{t_i}^{\mathrm{CS},1}[\mathbf I],
\qquad
\gamma(t_i)
=
\frac{1}{2\sqrt{\bar\alpha_{t_i}}},
\qquad
C_1(t_i)
=
\frac{\alpha}{\delta(t_i)^2},
\qquad
C_2(t_i)=0.
\label{eq:refl-rsm}
\end{equation}

ReFL therefore belongs to the same current-state first-order family as VGG-Flow, while differing in its local pullback: VGG-Flow uses the full Tweedie Jacobian $\mathbf J_{\mathrm{Tw}}$, whereas ReFL uses the identity pullback induced by its stop-gradient construction, with the remaining scalar Jacobian factor absorbed into $\gamma(t_i)$.

\subsubsection{Adjoint Matching}
\label{appendix:adjoint_matching}
Adjoint Matching~\cite{domingo-enrich2025adjoint} performs reward fine-tuning of diffusion and flow-matching models with entropy-regularized stochastic optimal control, where the optimal control term is obtained by regressing onto an adjoint state computed from a backward ODE. 
Specifically, given the following reverse-time SDE formulation of the reference model's generation process,
\begin{equation}
    d\vx_t = b_t^{\text{ref}}(\vx_t)dt + \tilde{\sigma}_tdw_t, \quad \vx_1 \sim p_1, \quad t:1\rightarrow0,  \label{eq:am_base_sde}
\end{equation}
Adjoint Matching aims to find the optimal control term $u_t(\vx_t,t)$ that minimizes the cost functional of the following SOC problem.
\begin{align}
    &\min_{u} \mathbb{E} \left[\int_0^1 \left(\frac{\alpha}{2}\left\|u_t(\vx_t)\right\|^2 - \rho_t(\vx_t)\right)dt - r(\vx_0)\right] \\
    \text{s.t.} \quad d\vx_t = b_t^{\theta}(\vx_t)dt& + \tilde{\sigma}_tdw_t := \left(b_t^{\text{ref}}(\vx_t) - \tilde{\sigma}_tu_t(\vx_t)\right)dt + \tilde{\sigma}_tdw_t, \quad \vx_1 \sim p_1, \quad t:1\rightarrow0 \notag
    \label{eq:am_soc_cost}
\end{align}
Here, $\rho_t(\vx_t)$ and $r(\vx_0)$ denote intermediate and terminal reward, respectively.

For a trajectory $\tau$ sampled from SDE controlled by $u$, adjoint state $\va_t(\tau, u)$ is defined as the pathwise derivative of the realized future cost with respect to the state $\vx_t$. It can be computed backward along the generative trajectory through the following adjoint ODE:
\begin{align}
    \frac{d\va_t}{dt} &= -\va_t^\top\nabla_{\vx_t}\left(b_t^{\text{ref}}(\vx_t) - \tilde{\sigma}_tu_t(\vx_t)\right) + \nabla_{\vx_t}\left(-\rho_t(\vx_t) + \frac{\alpha}{2}\left\|u_t(\vx_t)\right\|^2\right) \\
    \va_0&=-\nabla_{\vx_0}r(\vx_0).      \label{eq:am_exact_adjoint_ode}
\end{align}

Adjoint Matching describes how $\va_t$ obtained from different trajectories on $\vx_t$ can be a regression target for the difference between the reference and fine-tuned model on $\vx_t$. Based on the Hamilton-Jacobi-Bellman equation, Eq.~(34) of \cite{domingo-enrich2025adjoint} states the following objective;
\begin{equation}
    \mathcal{L}_{\mathrm{AM}}(\theta) = \mathbb{E}_{\tau\sim p^{u}, t}\left[\left\| u_t(\vx_t) + \frac{1}{\alpha}\tilde{\sigma}_t\va_t(\tau,\text{sg}(u))\right \|^2\right].
    \label{eq:am_original_loss}
\end{equation}
Substituting the linear correlation between $u_t(\vx_t)$, $b_t(\vx_t)$, and $\vs^\theta_t(\vx_t)$ into Eq.~\eqref{eq:am_original_loss} yields
\begin{align}
    \mathcal{L}_{\mathrm{AM}}(\theta) &= \mathbb{E}_{\tau\sim p^{u}, t}\left[\left\| -\frac{1}{\tilde{\sigma}_t}(b_t^{\theta}(\vx_t) - b_t^{\text{ref}}(\vx_t)) + \frac{1}{\alpha}\tilde{\sigma}_t\va_t\right \|^2\right] \\
    &= \mathbb{E}_{\tau\sim p^{u}, t}\left[\left\| \left(\frac{\tilde{\sigma}_t}{2}+\frac{1}{\tilde{\sigma}_t\delta(t)}\right)(\vs_t^{\theta}(\vx_t) - \vs_t^{\text{ref}}(\vx_t)) + \frac{1}{\alpha}\tilde{\sigma}_t\va_t\right \|^2\right] \\
    &= \mathbb{E}_{\tau\sim p^{u}, t}\left[\left(\frac{\tilde{\sigma}_t}{2}+\frac{1}{\tilde{\sigma}_t\delta(t)}\right)^2\left\| \vs_t^{\theta}(\vx_t) - \left(\vs_t^{\text{ref}}(\vx_t) - \frac{1}{\alpha}\frac{2\tilde{\sigma}_t^2\delta(t)}{\tilde{\sigma}_t^2\delta(t)+2}\va_t\right)\right\|^2\right],
\end{align}

where $\delta(t)$ represents the coefficient relating the score difference to the velocity difference, as defined in Appendix~\ref{appendix:delta-derivations}.
Thus, Adjoint Matching is an instance of the RSM objective in Eq.~\eqref{eq:master_equation}, with
\begin{equation}
    \textcolor{blue}{\mathbf{\Psi}_{t}} = -\frac{1}{\alpha}\frac{2\tilde{\sigma}_{t}^2\delta(t)}{\tilde{\sigma}_{t}^2\delta(t)+2}\va_{t}, \quad \gamma(t)=1, \quad C_1(t)=\left(\frac{\tilde{\sigma}_{t}}{2}+\frac{1}{\tilde{\sigma}_{t}\delta(t)}\right)^2, \quad C_2(t)=0. \label{eq:am_rsm_correspondence}
\end{equation}

Meanwhile, when we define cost-to-go functional $J(u;\vx_t,t)$ of $\vx_t$ as the expected future cost under a fixed control $u$,
\begin{equation}
    J(u;\vx_t,t) = \mathbb{E}_{\tau\sim p^{u}} \left[\left.\int_0^t\left(\frac{\alpha}{2}\left\|u_s(\vx_s)\right\|^2 - \rho_s(\vx_s)\right)ds - r(\vx_0)\right|\vx_t \right],  \label{eq:am_cost_to_go}
\end{equation}
it is known that the adjoint state of the controlled trajectory from $\vx_t$ is an unbiased estimator of the gradient of the cost-to-go functional for any fixed control $u$;
\begin{equation}
    \mathbb{E}_{\tau\sim p^{\vu}}\left[\va_t(\tau, u)|\vx_t\right] = \nabla_{\vx_t}J(u;\vx_t,t). \label{eq:am_adjoint_unbiased}
\end{equation}
Under the optimal control $u^\star$, the cost-to-go functional $J(u^\star;\vx_t,t)$ becomes the value function $V^{\mathrm{cost}}_t(\vx_t)$ in terms of cost, which is the negated value function for the reward. Therefore,
\begin{equation}
    \mathbb{E}_{\tau\sim p^{u^\star}}\left[-\frac{1}{\alpha}\va_t(\tau, u^\star)|\vx_t\right] = -\frac{1}{\alpha}\nabla_{\vx_t}V^{\mathrm{cost}}_{t}(\vx_t)=\frac{1}{\alpha}\nabla_{\vx_t}V_t(\vx_t) = \textcolor{red}{\mathbf{\Psi}_{t}^\star}(\vx_t)\label{eq:am_optimal_adjoint_guidance}
\end{equation}

Compared to the unbiased value gradient estimator in Eq.~\eqref{eq:am_optimal_adjoint_guidance},  $\textcolor{blue}{\mathbf{\Psi}_{t}}$ in Eq.~\eqref{eq:am_rsm_correspondence} has an additional coefficient attached to the adjoint value $\va_{t}$. This means that Adjoint Matching only yields the optimal tilted marginal distribution in Eq.~\eqref{eq:optimal-marginal-distribution} when $\frac{2\tilde{\sigma}_{t}^2\delta(t)}{\tilde{\sigma}_{t}^2\delta(t)+2}=1$, which only happens when $\tilde{\sigma}_{t}$ is equal to $\sqrt{2/\delta(t)}$, also known as \emph{memoryless} noise schedule. In this specific setting of stochasiticity, the corresponding RSM coefficients for RSM in Eq.~\eqref{eq:am_rsm_correspondence} can be further simplified as
\begin{equation}
    \textcolor{blue}{\mathbf{\Psi}_{t}} = -\frac{1}{\alpha}\va_{t}, \quad \gamma(t)=1, \quad C_1(t)=\frac{2}{\delta(t)}, \quad C_2(t)=0. 
\end{equation}

Adjoint Matching propagates the terminal reward gradient to intermediate states through a backward ODE to utilize an unbiased value gradient estimator for a first-order method. However, as mentioned in the main manuscript, this requires Jacobian propagation throughout timesteps, making it computationally heavy and time-consuming. Moreover, as a trade-off of introducing a full rollout, the estimator can have increased variance, which had to be addressed by variance reduction techniques such as the lean adjoint regression.

Note that since the adjoint state is an unbiased value gradient estimator only when the control is already optimal, the guarantee of convergence to the optimal control and the training dynamics have to be discussed similarly to other methods in Appendix~\ref{sec:convergence}, which were further addressed by the authors~\cite{domingo-enrich2025adjoint}.

\subsection{Zeroth-order Methods}
\label{appendix:subsec-zeroth-order}

We now turn to policy-gradient methods, beginning with KL-regularized REINFORCE as the canonical strictly on-policy zeroth-order base case.
This yields the fundamental zeroth-order effective value guidance, from which other variants arise as straightforward modifications within the RSM design space.

\subsubsection{KL-regularized REINFORCE}
\label{appendix:reinforce}

We formulate REINFORCE \cite{Sutton1998,Williams:92} in the KL-regularized form used by Algorithm 1 (Soft PPO) of \citet{uehara2024understandingreinforcementlearningbasedfinetuning}. Using our general parameterization, the objective is
\begin{equation}
\label{eq:reinforce-with-kl}
    \mathcal{L}(\theta) = \mathbb{E}_{\vx_{t_i} \sim p^\theta, \veps_{t_{i:1}}} \left[ -r(\vx_0) + \alpha \sum_{i=1}^N \mathcal{D}_\text{KL}(p^\theta(\vx_{t_{i-1}}|\vx_{t_i}) \| p^\text{ref}(\vx_{t_{i-1}}|\vx_{t_i})) \right].
\end{equation}

For Gaussian transitions with fixed covariance $\sigma_{t_i}^2\mathbf{I}$, each KL term reduces to
\begin{equation}
    \frac{\alpha}{2} \| \frac{1}{\sigma_{t_i}} (\pmb\mu_{t_{i-1}}^\theta(\vx_{t_i}) - \pmb\mu_{t_{i-1}}^\text{ref}(\vx_{t_i})) \|^2.
\end{equation}

The REINFORCE gradient of the reward term is
\begin{equation}
    \nabla_{\theta} \mathbb{E}_{\vx_{t_i}\sim p^\theta, \veps_{t_{i:1}}}[-r(\vx_0)]
    =\mathbb{E}_{\vx_{t_i} \sim p^\theta, \veps_{t_{i:1}}}
    \left[ -r(\vx_0) \nabla_\theta \log p_{\theta}(\vx_{t_{i-1}}|\vx_{t_i})\right].
\end{equation}

Since
\begin{equation}
    \nabla_\theta \log p_{\theta}(\vx_{t_{i-1}}|\vx_{t_i}) = (\vx_{t_{i-1}} - \pmb\mu_{t_{i-1}}^\theta)^\top \frac{1}{\sigma_{t_i}^2} \nabla_\theta \pmb\mu_{t_{i-1}}^\theta = (\frac{1}{\sigma_{t_i}}\veps_{t_i})^\top \nabla_\theta \pmb\mu_{t_{i-1}}^\theta,
\end{equation}
combining the reward and KL terms yields
\begin{equation}
    \nabla_\theta \mathcal{L}(\theta) = \mathbb{E}_{\tau}\left[ \Big( -r(\vx_0) \frac{1}{\sigma_{t_i}}\veps_{t_i} + \alpha \frac{1}{\sigma_{t_i}^2}(\pmb\mu_{t_{i-1}}^\theta - \pmb\mu_{t_{i-1}}^\text{ref}) \Big)^\top \nabla_\theta \pmb\mu_{t_{i-1}}^\theta \right],
\end{equation}
where $\tau=\{\vx_{t_i}\}_{i=0}^{N}$ denotes the trajectory induced by $\theta$.

Substituting the unified parameterization
$\pmb\mu_{t_{i-1}}^\theta=\kappa(t_i)\vx_{t_i}+\Omega(t_i)\vs_{t_i}^\theta$
from Proposition~\ref{cor:gen_soft_value}, we obtain
\begin{align}
    \nabla_\theta \mathcal{L}(\theta) &= \mathbb{E}_{\tau}\left[ \Big( -r(\vx_0) \frac{1}{\sigma_{t_i}}\veps_{t_i} + \frac{\alpha\Omega(t_i)}{\sigma_{t_i}^2} (\vs_{t_i}^\theta - \vs_{t_i}^\text{ref}) \Big)^\top \Omega(t_i)\nabla_\theta \vs_{t_i}^\theta \right] \\
    &= \mathbb{E}_{\tau}\left[ \frac{\alpha\Omega(t_i)^2}{\sigma_{t_i}^2} \Big( \vs_{t_i}^\theta - \vs_{t_i}^\text{ref} - \frac{\sigma_{t_i}}{\alpha\Omega(t_i)}r(\vx_0)\veps_{t_i} \Big)^\top \nabla_\theta \vs_{t_i}^\theta \right].
\end{align}
This is precisely the gradient of the weighted squared loss
\begin{equation}
    \mathcal{L}(\theta) = \mathbb{E}_{\tau} \left[ \frac{\alpha\Omega(t_i)^2}{2\sigma_{t_i}^2} \|  (\vs_{t_i}^\theta - (\vs_{t_i}^\text{ref} + \textcolor{blue}{\mathbf{\Psi}_{t_i}})) \|^2 \right],
\end{equation}
with
\begin{equation}
\textcolor{blue}{\mathbf{\Psi}_{t_i}} = \frac{\sigma_{t_i}}{\alpha\Omega(t_i)}r(\vx_0)\veps_{t_i}, \quad  C_1(t_i) =
\frac{\alpha}{2}
\frac{\Omega(t_i)^2}{\sigma_{t_i}^2}, \quad C_2(t_i) = 0.
\end{equation}

In relation to Proposition~\ref{cor:gen_soft_value}, REINFORCE corresponds to the zeroth-order estimation of the lookahead value gradient,
\begin{equation}
    \mathbb{E}_{\vx_{t_{i-1}}\sim p^\star_{\vx_{t_i}}} \left[ \nabla_{\vx_{t_{i-1}}}   V^\star_{t_{i-1}}(\vx_{t_{i-1}})\right] = \frac{\mathbb{E}[r(\vx_0)\veps_{t_i}]}{\sigma_{t_i}},
\end{equation}
using a point estimation for the expectation.

\subsubsection{Clipped log-ratio surrogate}
We next consider the clipped log-ratio surrogate loss used by PCPO-base~\cite{lee2025pcpo}. Denoting the importance sampling ratio by
\begin{equation}
\rho_{t_i}
=
\frac{p_{t_i}^\theta(\vx_{t_{i-1}}\mid \vx_{t_i})}
{p_{t_i}^{\theta^\dagger}(\vx_{t_{i-1}}\mid \vx_{t_i})},
\end{equation}
the clipped objective is
\begin{equation}
\label{eq:log-hinge-clip-loss}
\mathcal{L}_\text{clip-log}(\theta)=
\begin{cases}
\mathbb{E}_{\vx_{t_i} \sim p_{t_i}^{\theta^\dagger},\,\veps_{t_i}}\!\left[
\dfrac{\alpha \Omega(t_i)^2}{2\sigma_{t_i}^2}
\left\|\vs_{t_i}^\theta-\vs_{t_i}^\mathrm{ref}\right\|^2
\right]
&
\begin{aligned}[t]
&\text{if } \big(\log\rho_{t_i}<-\xi \land r(\vx_0)\le 0\big) \\
&\quad \text{or } \big(\log\rho_{t_i}>\xi \land r(\vx_0)\ge 0\big),
\end{aligned}
\\[1ex]
\mathbb{E}_{\vx_{t_i} \sim p_{t_i}^{\theta^\dagger},\,\veps_{t_i}}\!\left[
-r(\vx_0)\log\rho_{t_i}
+ \dfrac{\alpha \Omega(t_i)^2}{2\sigma_{t_i}^2}
\left\|\vs_{t_i}^\theta-\vs_{t_i}^\mathrm{ref}\right\|^2
\right]
&
\text{otherwise,}
\end{cases}
\end{equation}
where
\begin{equation}
\label{eq:negative-log-rho}
- \log \rho_{t_i} = -\frac{\Omega(t_i)}{\sigma_{t_i}} (\vs_{t_i}^\theta - \vs_{t_i}^{\theta^\dagger})^\top \veps_{t_i} + \frac{\Omega(t_i)^2}{2\sigma_{t_i}^2} \left\| \vs_{t_i}^\theta - \vs_{t_i}^{\theta^\dagger} \right\|^2.
\end{equation}
We group terms by $\vs_{t_i}^\theta$, and perform a square completion equivalent to the previous section:
\begin{align}
    &r(\vx_0) (-\log \rho_{t_i}) + \frac{\alpha\Omega(t_i)^2}{2\sigma_{t_i}^2} \left\| \vs_{t_i}^\theta - \vs_{t_i}^\text{ref} \right\|^2\nonumber\\
    &=-\frac{r(\vx_0)\Omega(t_i)}{\sigma_{t_i}} (\vs_{t_i}^\theta - \vs_{t_i}^{\theta^\dagger} -\vs_{t_i}^{\text{ref}} + \vs_{t_i}^{\text{ref}})^\top \veps_{t_i}
    + \frac{r(\vx_0)}{2} \frac{\Omega(t_i)^2}{\sigma_{t_i}^2}
    \| \vs_{t_i}^\theta - \vs_{t_i}^{\theta^\dagger} \|^2
    + \frac{\alpha\Omega(t_i)^2}{2\sigma_{t_i}^2} \|\vs_{t_i}^\theta - \vs_{t_i}^{\text{ref}}\|^2 \nonumber\\
    &=\frac{\alpha\Omega(t_i)^2}{2\sigma_{t_i}^2} \| \vs_{t_i}^\theta -
    \big( \vs_{t_i}^{\text{ref}} + \frac{\sigma_{t_i}}{\alpha\Omega(t_i)}r(\vx_0) \veps_{t_i} \big)\|^2
    + \frac{r(\vx_0)}{2} \frac{\Omega(t_i)^2}{\sigma_{t_i}^2}
    \| \vs_{t_i}^\theta - \vs_{t_i}^{\theta^\dagger} \|^2
    + C
\end{align}
where $C$ is a non-gradient-producing constant w.r.t. $\theta$.
Thus,
\begin{equation}
\label{eq:pcpo-base-loss}
\begin{aligned}
\mathcal{L}(\theta)=
\begin{cases}
\mathbb{E}\!\left[
\frac{\alpha\Omega(t_i)^2}{2\sigma_{t_i}^2}
\left\|\vs_{t_i}^\theta-\vs_{t_i}^{\mathrm{ref}}\right\|^2
\right]
&
\begin{aligned}[t]
&\text{if } \big(\log\rho_{t_i}<-\xi \land r(\vx_0)\le 0\big) \\
&\quad \text{or } \big(\log\rho_{t_i}>\xi \land r(\vx_0)\ge 0\big),
\end{aligned}
\\[1ex]
\mathbb{E}\!\left[
\frac{\alpha\Omega(t_i)^2}{2\sigma_{t_i}^2}
\left\|\vs_{t_i}^\theta-
\big(\vs_{t_i}^{\mathrm{ref}}+\textcolor{blue}{\mathbf{\Psi}_{t_i}}\big)\right\|^2
+
\frac{r(\vx_0)}{2}\frac{\Omega(t_i)^2}{\sigma_{t_i}^2}
\left\|\vs_{t_i}^\theta-\vs_{t_i}^{\theta^\dagger}\right\|^2
\right]
&
\text{otherwise.}
\end{cases}
\end{aligned}
\end{equation}
Therefore, the clipped log-ratio surrogate loss admits the canonical form with
\begin{equation}
\label{eq:pcpo-base-unified-terms}
\textcolor{blue}{\mathbf{\Psi}_{t_i}} =
\frac{\sigma_{t_i}}{\alpha\Omega(t_i)}r(\vx_0) \veps_{t_i},
\quad C_1(t_i) =
\frac{\alpha}{2}
\frac{\Omega(t_i)^2}{\sigma_{t_i}^2},
\quad C_2(t_i) = \frac{r(\vx_0)}{\alpha}.
\end{equation}

\subsubsection{Variants}

\textbf{EPG.} EPG~\cite{choi2026rethinking} is a variant of REINFORCE, which extends it to the \emph{loose} on-policy setting, and uses a control variate for variance reduction. Its loss can be written as
\begin{equation}
\mathcal{L}_{\mathrm{EPG}}(\theta)
=
\mathbb{E}_{\tau \sim p^{\theta^\dagger}}
\Big[
\sg (\rho_{t_i} )\,
\frac{\alpha\Omega(t_i)^2}{2\sigma(t_i)^2}
\big\|
\vs_{t_i}^\theta
-
\big(
\vs_{t_i}^{\mathrm{ref}} + \textcolor{blue}{\mathbf{\Psi}_{t_i}}
\big)
\big\|^2
\Big],
\end{equation}
with the same $\mathbf{\Psi}_{t_i}$ as KL-regularized REINFORCE.
Compared to REINFORCE, the only modification is the multiplicative stop-gradient importance ratio, which does not produce gradients itself.
Note that the practical calculation of $\rho_{t_i}$ requires exponentiation, which is numerically unstable and exhibits worse training dynamics compared to the log-ratio surrogate~\cite{lee2025pcpo}.

\textbf{PPO, GRPO.} \citet{huangppoclip} proved that the KL-regularized PPO loss
\begin{equation}
    \mathbb{E}_{\vx_{t_i} \sim p^{\theta^\dagger}, \veps_{t_i}}[\max\big(-\rho_{t_i}r(\vx_0), -\text{clip}_\xi(\rho_{t_i})r(\vx_0)\big) + \frac{\alpha\Omega(t_i)^2}{2\sigma_{t_i}^2} \left\| \vs_{t_i}^\theta - \vs_{t_i}^\text{ref} \right\|^2]
\end{equation}
is equivalent to
\begin{equation}
\label{eq:hinge-loss-clip}
\mathcal{L}_\text{clip}(\theta)=
\begin{cases}
\mathbb{E}_{\vx_{t_i} \sim p^{\theta^\dagger},\,\veps_{t_i}}\!\left[
\dfrac{\alpha \Omega(t_i)^2}{2\sigma_{t_i}^2}
\left\|\vs_{t_i}^\theta-\vs_{t_i}^\mathrm{ref}\right\|^2
\right]
&
\begin{aligned}[t]
&\text{if } \big(\rho_{t_i}<1-\xi \land r(\vx_0)\le 0\big) \\
&\quad \text{or } \big(\rho_{t_i}>1+\xi \land r(\vx_0)\ge 0\big),
\end{aligned}
\\[1ex]
\mathbb{E}_{\vx_{t_i} \sim p^{\theta^\dagger},\,\veps_{t_i}}\!\left[
-r(\vx_0)(\rho_{t_i}-1)
+ \dfrac{\alpha \Omega(t_i)^2}{2\sigma_{t_i}^2}
\left\|\vs_{t_i}^\theta-\vs_{t_i}^\mathrm{ref}\right\|^2
\right]
&
\text{otherwise.}
\end{cases}
\end{equation}
up to constants independent of $\theta$.
Therefore, PPO simply replaces the log-ratio term in Eq.~\eqref{eq:log-hinge-clip-loss} with $\rho_{t_i} - 1$.
For Gaussian transitions,
\begin{equation}
    \rho_{t_i} - 1 = \exp
    \left(
    -\frac{\Omega(t_i)}{\sigma_{t_i}} (\vs_{t_i}^\theta - \vs_{t_i}^{\theta^\dagger})^\top \veps_{t_i} 
    - \frac{\Omega(t_i)^2}{2\sigma_{t_i}^2} \left\| \vs_{t_i}^\theta - \vs_{t_i}^{\theta^\dagger} \right\|^2 \right) - 1.
\end{equation}
Directly optimizing this exponential is numerically unstable. In practice, diffusion/flow-based PPO and GRPO implementations \cite{black2023training,fan2023dpok,xue2025dancegrpo,liu2025flowgrpo} use an extremely tight clipping threshold (e.g., $\xi=10^{-4}$). Under such tight clipping, the first-order approximation $\exp(x)-1 \approx x$ incurs error of order $\xi^2/2$ for the interior regime \(1-\xi \le \rho_{t_i} \le 1+\xi\), as well as for clipped cases where the loss is already fixed at the clipping boundary\footnote{The remaining non-clipped sign-consistent regimes, \((\rho_{t_i}<1-\xi \land r(\vx_0)\ge 0\)) and \((\rho_{t_i}>1+\xi \land r(\vx_0)\le 0)\), admit a looser bound. Empirically, however, \citet{lee2025pcpo} observe that \( |\rho_{t_i}-1| < 0.01 \) throughout training due to a low learning rate, so the approximation error remains small in practice.}.
Under this necessary approximation, PPO / GRPO reduce to the clipped log-ratio surrogate in Eq.~\eqref{eq:pcpo-base-loss}, and therefore have the same unified terms as Eq.~\eqref{eq:pcpo-base-unified-terms}.

\textbf{PCPO-reweight.}
For flow-matching models, PCPO \emph{explicitly} reweights the policy loss by
\begin{equation}
\gamma_{t_i}
=
\frac{w_\text{new}(t_i)}{w(t_i)},
\qquad
w_\text{new}(t_i)=\zeta \Delta t_i,
\qquad
\zeta=\sum_{i=1}^N w(t_i).
\end{equation}
See Appendix~\ref{appendix:weights} for more details.
This modifies only the effective guidance and anchor strength:
\begin{equation}
\label{eq:pcpo-flow-reweight-unified-terms}
    \textcolor{blue}{\mathbf{\Psi}_{t_i}} =
    \gamma_{t_i}\frac{\sigma_{t_i}}{\alpha\Omega(t_i)}r(\vx_0) \veps_{t_i},
    \quad C_1(t_i) =
    \frac{\alpha}{2}
    \frac{\Omega(t_i)^2}{\sigma_{t_i}^2},
    \quad C_2(t_i) = \gamma_{t_i}\frac{r(\vx_0)}{\alpha}.
\end{equation}
For diffusion models, PCPO \emph{implicitly} reweights $h(t_i)$ by modifying the variance schedule $\sigma_{t_i}$ to $\sigma'_{t_i}$. Its effect is summarized in Table~\ref{tab:unified-table}.

\textbf{GRPO-Guard.}
\label{appendix:grpo-guard}
GRPO-Guard observes that $\log \rho_{t_i}$ is negatively biased by the KL term between the current and old policies, which empirically shifts its mean below zero. To correct this, it centers and rescales the log-ratio as
\[
\sigma_{t_i}\big(\log \rho_{t_i}-\mathbb{E}[\log \rho_{t_i}]\big),
\]
and additionally reweights the policy term by $1/\Delta t_i$. Under the clipped log-ratio objective, this yields
\begin{equation}
\label{eq:guard-base-loss}
\begin{aligned}
    \mathcal{L}(\theta)
    &=
    \begin{cases}
        \mathbb{E}
        \left[
        \frac{\alpha\Omega(t_i)^2}{2\sigma_{t_i}^2}
        \left\|
        \vs_{t_i}^\theta-\vs_{t_i}^\text{ref}
        \right\|^2
        \right]
        &\text{if } \big(\sigma_{t_i}(\log\rho_{t_i} - \mathbb{E}[\log \rho_{t_i}])<-\xi \land r(\vx_0)\le 0\big) \\ 
&\quad \text{or } \big(\sigma_{t_i}(\log\rho_{t_i} - \mathbb{E}[\log \rho_{t_i}])>\xi \land r(\vx_0)\ge 0\big), \\
        \mathbb{E}
        \left[
        \frac{\alpha\Omega(t_i)^2}{2\sigma_{t_i}^2}
        \left\|
        \vs_{t_i}^\theta
        -
        \big(
        \vs_{t_i}^\text{ref}
        +
        \textcolor{blue}{\mathbf{\Psi}_{t_i}}
        \big)
        \right\|^2
        \right]
        & \text{otherwise},
    \end{cases}
\end{aligned}
\end{equation}
with
\begin{equation}
\label{eq:grpo-guard-unified-terms}
    \textcolor{blue}{\mathbf{\Psi}_{t_i}} = \gamma_{t_i}
    \frac{\sigma_{t_i}}{\alpha\Omega(t_i)}r(\vx_0) \veps_{t_i},
    \quad C_1(t_i) =
    \frac{\alpha}{2}
    \frac{\Omega(t_i)^2}{\sigma_{t_i}^2},
    \quad C_2(t_i) = 0,
\end{equation}
where $\gamma_{t_i} = \frac{\sigma_{t_i}\Omega(t_i)}{\Delta t_i}$. Thus, GRPO-Guard explicitly removes the old-policy anchor term.

\textbf{TempFlow-GRPO.} TempFlow-GRPO introduces \emph{trajectory branching}: starting from a deterministic ODE trajectory, it revisits each intermediate latent $\vx_{t_i}$, injects one-step SDE noise to generate multiple descendants, and then deterministically denoises each descendant to the terminal state for reward evaluation.
This is closely related to \emph{vine sampling} used to reduce estimator variance in TRPO~\cite{schulmantrpo2015}, and is precisely a multi-particle lookahead estimator in our framework.
Note that TempFlow-GRPO does not explicitly average value-guidance vectors; instead, it applies an $L_2$ regression loss to multiple reward-weighted point estimates.
However, because the gradient of a sum of squared errors is centered at the sample mean, this optimization is functionally equivalent to regressing toward the average of those particle-wise targets.
Hence, under RSM, TempFlow-GRPO implicitly realizes a zeroth-order lookahead value-guidance estimator with branching width $K_i>1$, together with heuristic temporal reweighting $\gamma(t_i)=\frac{9}{4}\sigma_{t_i}$.
This yields
\begin{equation}
\label{eq:tempflow-unified-terms}
    \textcolor{blue}{\mathbf{\Psi}_{t_i}} =
    \gamma_{t_i}\frac{\sigma_{t_i}}{\alpha\Omega(t_i)}
    \Big( \frac{1}{K_i} \sum_{k=1}^{K_i}
    r(\vx_0^{(k)}) \veps_{t_i}^{(k)} \Big),
    \quad C_1(t_i) =
    \frac{\alpha}{2}
    \frac{\Omega(t_i)^2}{\sigma_{t_i}^2},
    \quad C_2(t_i) = \gamma_{t_i}\frac{r(\vx_0)}{\alpha}.
\end{equation}

\textbf{BranchGRPO.} BranchGRPO likewise uses multiple descendants to estimate a zeroth-order lookahead guidance term, but differs from TempFlow-GRPO in that its branching is \emph{recursive} rather than one-step from each revisited parent latent. Under the same RSM view, it therefore realizes a multi-particle estimator form:
\begin{equation}
\label{eq:branch-unified-terms}
    \textcolor{blue}{\mathbf{\Psi}_{t_i}} =
    \frac{\sigma_{t_i}}{\alpha\Omega(t_i)}
    \Big( \frac{1}{K_i} \sum_{k=1}^{K_i}
    r(\vx_0^{(k)}) \veps_{t_i}^{(k)} \Big),
    \quad C_1(t_i) =
    \frac{\alpha}{2}
    \frac{\Omega(t_i)^2}{\sigma_{t_i}^2},
    \quad C_2(t_i) = \frac{r(\vx_0)}{\alpha}.
\end{equation}

\subsection{Remark on parameters.}
To address the discrepancies in naming conventions across Soft RL and GFlowNet literature, we have unified the notation for components that serve identical functions.
Table~\ref{tab:notation_mapping} maps the various notations found in prior works to the single, unified notation used in this text.

\begin{table}[!t]
    \caption{\textbf{Unified Notation for Prior Works.}}
    \label{tab:notation_mapping}
    \centering
    \resizebox{0.65\linewidth}{!}{
    \begin{tabular}{l l c}
        \toprule
        \textbf{Parameter Description} & \textbf{Notation in Literature} & \textbf{Unified} \\
        \midrule
        \multirow{4}{*}{Entropy Regularization} 
            & $\alpha$ \cite{uehara2024understandingreinforcementlearningbasedfinetuning} & \multirow{4}{*}{$\alpha$} \\
            & $\frac{2\alpha}{d}$ \cite{kang2026sqdf} & \\
            & $\beta$ \cite{lee2025pcpo,liu2025flowgrpo} & \\
            & $\frac{1}{\beta}$ \cite{liu2025resnabladb,liu2025vggflow} & \\
        \midrule
        \multirow{2}{*}{Downweighting Factor}   
            & $\eta_t$ \cite{liu2025vggflow} & \multirow{2}{*}{$\tilde\gamma_t$} \\
            & $\gamma_t$ \cite{liu2025resnabladb, kang2026sqdf} & \\
        \midrule
        \multirow{2}{*}{Residual Parameterization} 
            & $g_\phi$ \cite{liu2025resnabladb} & \multirow{2}{*}{$g_\phi$} \\
            & $-\nu_\phi$ \cite{liu2025vggflow} & \\
        \midrule
        \multirow{2}{*}{Reward Representation}      
            & $\log R(\cdot)$ \cite{liu2025resnabladb, zhang2025dagdb} & \multirow{2}{*}{$r(\cdot)$} \\
            & $r(\cdot)$ \cite{kang2026sqdf, lee2025pcpo, uehara2024understandingreinforcementlearningbasedfinetuning, liu2025flowgrpo,liu2025vggflow} & \\
        \bottomrule
    \end{tabular}
    }
\end{table}

\textbf{Entropy Regularization ($\alpha$)}: While often denoted as $\beta$ (inverse temperature) or $1/\beta$ in Soft RL contexts, we adopt $\alpha$ to strictly control the degree of entropy regularization\footnote{SQDF~\cite{kang2026sqdf} apply mean reduction over the latent dimension $d$ for \emph{only} entropy regularization, and omit the $\frac{1}{2}$ coefficient for the KL divergence. Consequently, their reported $\alpha$ corresponds to an effective regularization weight of $2\alpha/d$.}
.

\textbf{Downweighting Factor ($\tilde\gamma_t$)}: We rename VGG-Flow's one-step reward estimation factor from $\eta_t$ to $\tilde\gamma_t$ to align with the conventions established by Residual $\nabla$-DB and SQDF.

\textbf{Reward Value ($r(\cdot)$)}: In diffusion fine-tuning GFlowNet literature, the reward is often interpreted as the log of the GFlowNet reward. We use $r(\cdot)$ directly to refer to this value.

\section{Parameterizations}
\subsection{\texorpdfstring{$\Omega$ Derivations for Popular Samplers}{Omega Derivations for Popular Samplers}}
\label{appendix:omega-derivations}

In this section, we provide step-by-step derivations of the scaling factor $\Omega(t_i)$ for two widely used sampling algorithms: DDIM applied to VP-SDE (e.g., Stable Diffusion 1.5) and Euler Discrete applied to Rectified Flow (e.g., Stable Diffusion 3).

From Proposition~\ref{cor:gen_soft_value}, the relationship between the optimal score $\vs_{t_i}^\star$, the reference score $\vs_{t_i}^\text{ref}$, and the value gradient is established via the conditional transition mean $\pmb\mu_{t_{i-1}}$. When $\pmb\mu_{t_{i-1}}$ can be expressed in the form $\pmb\mu_{t_{i-1}} = \kappa(t_i)\vx_{t_i} + {\Omega(t_i)}\vs_{t_i}$, the guidance term simplifies to:
\begin{equation}
    \vs_{t_i}^\star = \vs_{t_i}^\text{ref} + \frac{\sigma_{t_i}^2}{\alpha\Omega(t_i)}\nabla_{\vx_{t_{i-1}}}V_{t_{i-1}}(\vx_{t_{i-1}}).
\end{equation}
The specific value of $\Omega(t_i)$ depends on the choice of noise schedule and numerical solver.

\subsubsection{DDIM Sampler (VP-SDE)}
\label{subsec:Omega-ddim}

For the DDIM sampler applied to a VP-SDE, the transition from $\vx_{t_i}$ to $\vx_{t_{i-1}}$ is defined as:
\begin{equation}
    \vx_{t_{i-1}} = \sqrt{\bar{\alpha}_{t_{i-1}}} \hat{\vx}_{0}(\vx_{t_i}) + \sqrt{1 - \bar{\alpha}_{t_{i-1}} - \sigma_{t_i}^2} \veps_{t_i}^\theta + \sigma_{t_i} \veps_{t_i}
\end{equation}
where $\hat{\vx}_{0}(\vx_{t_i}) = \frac{\vx_{t_i} - \sqrt{1 - \bar{\alpha}_{t_i}} \veps_{t_i}^\theta}{\sqrt{\bar{\alpha}_{t_i}}}$ is the Tweedie estimate of the clean data, and $\veps_{t_i} \sim \mathcal{N}(\mathbf{0}, \mathbf{I})$ is the injected noise.
The conditional transition mean $\pmb\mu_{t_{i-1}}(\vx_{t_i})$ can be rewritten as:
\begin{align}
\label{eq:Omega-ddim}
    \pmb\mu_{t_{i-1}}(\vx_{t_i}) &= \sqrt{\bar{\alpha}_{t_{i-1}}} \left( \frac{\vx_{t_i} - \sqrt{1 - \bar{\alpha}_{t_i}} \veps_{t_i}^\theta}{\sqrt{\bar{\alpha}_{t_i}}} \right) + \sqrt{1 - \bar{\alpha}_{t_{i-1}} - \sigma_{t_i}^2} \veps_{t_i}^\theta \nonumber \\
    &= \sqrt{\frac{\bar{\alpha}_{t_{i-1}}}{\bar{\alpha}_{t_i}}} \vx_{t_i} - \left( \sqrt{\frac{\bar{\alpha}_{t_{i-1}}}{\bar{\alpha}_{t_i}}} \sqrt{1 - \bar{\alpha}_{t_i}} - \sqrt{1 - \bar{\alpha}_{t_{i-1}} - \sigma_{t_i}^2} \right) \veps_{t_i}^\theta \nonumber \\
    &= \sqrt{\frac{\bar{\alpha}_{t_{i-1}}}{\bar{\alpha}_{t_i}}} \vx_{t_i} + \underbrace{\left( \sqrt{\frac{\bar{\alpha}_{t_{i-1}}}{\bar{\alpha}_{t_i}}} \sqrt{1 - \bar{\alpha}_{t_i}} - \sqrt{1 - \bar{\alpha}_{t_{i-1}} - \sigma_{t_i}^2} \right) \sqrt{1 - \bar{\alpha}_{t_i}}}_{\Omega(t_i)}\vs_{t_i}^\theta.
\end{align}

\subsubsection{SDE-DPM-Solver++ (VP-SDE)}

To further demonstrate the generality of our framework, we derive $\Omega(t_i)$ for SDE-DPM-Solver++~\cite{lu2025dpmppsolverfastode}, a higher-order solver supported by Residual $\nabla$-DB.



Let $a_t = \sqrt{\bar\alpha_t}$ and $b_t^2 = 1-a_t^2$ for the brevity. Let $\lambda_t=\log(a_t/b_t)$ be log-SNR and $h = \lambda_{t_{i-1}}-\lambda_{t_i}$.
For SDE-DPM-Solver++ 1 sampler applied to a VP-SDE, the transition from $\vx_{t_i}$ to $\vx_{t_{i-1}}$ is defined as:
\begin{align}
    \vx_{t_{i-1}} = \frac{b_{t_{i-1}}}{b_{t_i}}e^{-h}\vx_{t_i} + a_{t_{i-1}}(1-e^{-2h})\hat\vx_{0}(\vx_{t_i}) + b_{t_{i-1}}\sqrt{1-e^{-2h}}\veps_{t_i},
\end{align}
where $\hat\vx_{0}(\vx_{t_i}) = \frac{\vx_{t_i}-b_{t_i}\veps^\theta_{t_i}}{a_{t_i}}=\frac{\vx_{t_i} + b_{t_i}^2 \vs_{t_i}^\theta}{a_{t_i}}$ denotes the Tweedie estimate of the clean data, and $\veps_{t_i}\sim \gN(0, \rmI)$ is the sampled random noise.
Thus, 
\begin{align}
    \pmb\mu_{t_{i-1}}(\vx_{t_i}) = \left(\frac{b_{t_{i-1}}}{b_{t_i}}e^{-h} + \frac{a_{t_{i-1}}}{a_{t_i}}(1-e^{-2h}) \right) \vx_{t_i} + \frac{a_{t_{i-1}}}{a_{t_i}}(1-e^{-2h})b_{t_i}^2\vs_{t_i}^\theta.
\end{align}
The covariance for this step is $\pmb\Sigma_{t_i}=b^2_{t_{i-1}}(1-e^{-2h}) \rmI:=\sigma_{t_i} \rmI$. By comparing this result with the form $\pmb\mu_{t_{i-1}}=\kappa(t_i)\vx_{t_i} + \Omega(t_i)\vs_{t_i}^\theta(\vx_{t_i})$, we have
\begin{align}
    \Omega(t_i) = \frac{a_{t_{i-1}}}{a_{t_i}} \left(1-e^{-2h} \right) b_{t_i}^2 = \frac{a_{t_{i-1}}}{a_{t_i}} \frac{b_{t_i}^2}{b_{t_{i-1}}^2} \left(1-e^{-2h} \right) b_{t_{i-1}}^2=  \sqrt{\frac{\bar \alpha_{t_{i-1}}}{\bar\alpha_{t_i}}} \frac{1-\bar\alpha_{t_i}}{1-\bar\alpha_{t_{i-1}}} \sigma_{t_i}.
\end{align}

\subsubsection{Euler Discrete Sampler (Flow-SDE)}
\label{subsec:Omega-euler-discrete}
For the Euler Discrete sampler applied to Rectified Flow (e.g., Flow-GRPO~\cite{liu2025flowgrpo}), we consider the reverse SDE that shares the same marginal densities as the probability flow ODE. 
The reverse SDE step from $t_i$ to $t_{i-1}$ (where $t_{i-1} < t_i$, and we define $\Delta t_i = t_i - t_{i-1}$) takes the form:
\begin{equation}
    \vx_{t_{i-1}} = \underbrace{\vx_{t_i} - \Delta t_i \vv_{t_i}^\theta(\vx_{t_i}) + \frac{\sigma_{t_i}^2}{2} \vs_{t_i}^\theta(\vx_{t_i})}_{\pmb\mu_{t_{i-1}}} +\sigma_{t_i} \veps_{t_i},
\end{equation}
where $\vv_{t_i}^\theta(\vx_{t_i})$ is the velocity field prediction, $\vs_{t_i}^\theta(\vx_{t_i})$ is the score prediction, and $\veps_{t_i} \sim \mathcal{N}(\mathbf{0}, \mathbf{I})$ is the injected noise\footnote{Here, $\sigma_t$ follows the definition in Proposition~\ref{cor:gen_soft_value}. Note that $\sigma_t = \tilde \sigma_t (\Delta t)$.}.
In Rectified Flow, the velocity field and the marginal score are related by:
\begin{equation}
\label{eq:v-and-score-relation-rectified-flow}
    \vv_{t_i}(\vx_{t_i}) = -\frac{1}{1 - t_i} \vx_{t_i} - \frac{t_i}{1 - t_i} \vs_{t_i}(\vx_{t_i}),
\end{equation}
which is consistent with the Tweedie estimate $\hat{\vx}_0(\vx_{t_i}) = \frac{\vx_{t_i} + \vs_{t_i}^\theta(\vx_{t_i}) t_i^2}{1 - t_i} = \vx_{t_i} - \vv_{t_i}^\theta(\vx_{t_i}) t_i$.
Substituting Eq.~\eqref{eq:v-and-score-relation-rectified-flow} for $\vv_{t_i}^\theta$ into the conditional transition mean $\pmb\mu_{t_{i-1}}$, we get:
\begin{align}
    \pmb\mu_{t_{i-1}}(\vx_{t_i}) &= \vx_{t_i} - \Delta t_i \left( -\frac{1}{1 - t_i} \vx_{t_i} - \frac{t_i}{1 - t_i} \vs_{t_i}^\theta(\vx_{t_i}) \right) +  \frac{\sigma_{t_i}^2}{2} \vs_{t_i}^\theta(\vx_{t_i}) \nonumber \\
    &= \left( 1 + \frac{\Delta t_i}{1 - t_i} \right) \vx_{t_i} + \left( \frac{t_i}{1 - t_i}\Delta t_i + \frac{\sigma_{t_i}^2}{2} \right) \vs_{t_i}^\theta(\vx_{t_i}).
\end{align}
Next, our aim is to write the guidance term via the transition mean $\pmb\mu_{t_{i-1}} = \kappa(t_i)\vx_{t_i} + \Omega(t_i)\vs_{t_i}^\theta(\vx_{t_i})$. 
Equating the coefficients of $\vs_{t_i}^\theta(\vx_{t_i})$, we obtain:
\begin{equation}
\label{eq:Omega-EulerDiscrete}
    \Omega(t_i) = \frac{t_i}{1 - t_i}\Delta t_i + \frac{\sigma_{t_i}^2}{2}.
\end{equation}
Note that $\sigma_{t_i} = \tilde \sigma_{t_i} \sqrt{\Delta t_i}$, where $\tilde \sigma_{t_i} = a\sqrt{\frac{t_i}{1-t_i}}$ for Flow-GRPO SDE and $\tilde \sigma_{t_i} = a$ for Dance-GRPO SDE.

This derivation clarifies how the intrinsic parameters of the flow model ($t_i$) and the chosen noise scale ($\sigma_{t_i}$) jointly determine the magnitude of value guidance $\Omega(t_i)$ required to navigate the reverse SDE towards the optimal policy.

\subsubsection{CPS Sampler}
\label{subsec:Omega-cps}

CPS~\cite{wang2025coefficientspreservingsamplingreinforcementlearning} uses a modified sampler for flow-matching models with a coefficient-preserving stochastic interpolation.
Its transition mean is defined using
\begin{equation}
    \tilde{\sigma}_{t_i}
    =
    t_{i-1}\sin\left(\frac{\eta\pi}{2}\right),
\end{equation}
and
\begin{equation}
    \pmb\mu_{t_{i-1}}(\vx_{t_i})
    =
    (1-t_{i-1})\hat{\vx}_0
    +
    \sqrt{t_{i-1}^2-\tilde{\sigma}_{t_i}^2}\,
    \hat{\vx}_1,
\end{equation}
where
\begin{equation}
    \hat{\vx}_0
    =
    \vx_{t_i}-t_i\vv_{t_i},
    \qquad
    \hat{\vx}_1
    =
    \vx_{t_i}+(1-t_i)\vv_{t_i}.
\end{equation}

Since
\begin{equation}
    \sqrt{t_{i-1}^2-\tilde{\sigma}_{t_i}^2}
    =
    t_{i-1}\cos\left(\frac{\eta\pi}{2}\right),
\end{equation}
the transition mean becomes
\begin{align}
    \pmb\mu_{t_{i-1}}(\vx_{t_i})
    &=
    \left(
        1-t_{i-1}
        +
        t_{i-1}\cos\left(\frac{\eta\pi}{2}\right)
    \right)\vx_{t_i}
    \nonumber\\
    &\quad+
    \left(
        -t_i(1-t_{i-1})
        +
        (1-t_i)t_{i-1}
        \cos\left(\frac{\eta\pi}{2}\right)
    \right)\vv_{t_i}.
\end{align}

For Rectified Flow,
\begin{equation}
    \vv_{t_i}
    =
    -\frac{1}{1-t_i}\vx_{t_i}
    -
    \frac{t_i}{1-t_i}\vs_{t_i},
\end{equation}
so substituting the velocity--score relation yields
\begin{align}
    \pmb\mu_{t_{i-1}}(\vx_{t_i})
    &=
    \left(
        1-t_{i-1}
        +
        t_{i-1}\cos\left(\frac{\eta\pi}{2}\right)
        +
        \frac{t_i(1-t_{i-1})}{1-t_i}
        -
        t_{i-1}\cos\left(\frac{\eta\pi}{2}\right)
    \right)\vx_{t_i}
    \nonumber\\
    &\quad+
    \underbrace{
    \left(
        \frac{t_i^2(1-t_{i-1})}{1-t_i}
        -
        t_i t_{i-1}
        \cos\left(\frac{\eta\pi}{2}\right)
    \right)
    }_{\Omega(t_i)}
    \vs_{t_i}.
\end{align}

Therefore, CPS admits the unified transition parameterization
\begin{equation}
    \pmb\mu_{t_{i-1}}
    =
    \kappa(t_i)\vx_{t_i}
    +
    \Omega(t_i)\vs_{t_i},
\end{equation}
with
\begin{equation}
    \boxed{
    \Omega(t_i)
    =
    \frac{t_i^2(1-t_{i-1})}{1-t_i}
    -
    t_i t_{i-1}
    \cos\left(\frac{\eta\pi}{2}\right)
    }.
\end{equation}
The corresponding stochastic transition coefficient is
\begin{equation}
    \boxed{
    \sigma_{t_i}
    =
    t_{i-1}
    \sin\left(\frac{\eta\pi}{2}\right)
    \sqrt{t_i-t_{i-1}}
    }.
\end{equation}

\subsection{\texorpdfstring{$\delta$ Derivations for Common Parameterizations}{Delta Derivations for Common Parameterizations}}
\label{appendix:delta-derivations}

We derive $\delta(t_i)$, the signed scalar relating the parameterization difference to the score difference, for cases used in practice in this paper.

\textbf{$\veps$-prediction, VP-SDE.}
Under the forward marginal, the score satisfies
\begin{equation}
\label{eq:score-noise}
\vs_{t_i} = -\frac{1}{\sqrt{1-\bar{\alpha}_{t_i}}}\veps,
\qquad
\veps = -\sqrt{1-\bar{\alpha}_{t_i}}\,\vs_{t_i}.
\end{equation}
Therefore,
\[
\veps_{t_i}^\theta - \veps_{t_i}^\text{ref}
=
-\sqrt{1-\bar{\alpha}_{t_i}}
\bigl(\vs_{t_i}^\theta - \vs_{t_i}^\text{ref}\bigr).
\]
Comparing this with
\[
\veps_{t_i}^\theta - \veps_{t_i}^\text{ref}
=
-\frac{1}{\delta(t_i)}
\bigl(\vs_{t_i}^\theta - \vs_{t_i}^\text{ref}\bigr),
\]
we obtain, for VP-SDE diffusion models,
\begin{equation}
    \delta(t_i) = \frac{1}{\sqrt{1 - \bar \alpha_{t_i}}}.
    \label{eq:delta-eps-ddim}
\end{equation}

\textbf{$\vv$-prediction, Rectified Flow.}
Consider the general forward marginal
$p_t(\vx_{t_i} \mid \vx_0) = \mathcal{N}(\vx_{t_i}; a_{t_i}\vx_0, b_{t_i}^2 \mathbf{I})$,
which induces the trajectory
$\vx_{t_i} = a_{t_i}\vx_0 + b_{t_i}\veps$.
Differentiating with respect to time gives
$\vv_{t_i} \triangleq \frac{d\vx_{t_i}}{dt} = \dot a_{t_i}\vx_0 + \dot b_{t_i}\veps$.
Substituting
$\vx_0 = \frac{1}{a_{t_i}}(\vx_{t_i} - b_{t_i}\veps)$
to eliminate the clean data yields
\[
\vv_{t_i}(\vx_{t_i})
=
\frac{\dot a_{t_i}}{a_{t_i}}\vx_{t_i}
+
\left(
\dot b_{t_i} - \frac{\dot a_{t_i} b_{t_i}}{a_{t_i}}
\right)\veps.
\]
Using the score identity (Eq.~\eqref{eq:score-noise}), this becomes
\[
\vv_{t_i}(\vx_{t_i})
=
\frac{\dot a_{t_i}}{a_{t_i}}\vx_{t_i}
+
\left(
\frac{\dot a_{t_i} b_{t_i}^2 - a_{t_i}\dot b_{t_i} b_{t_i}}{a_{t_i}}
\right)\vs_{t_i}(\vx_{t_i}).
\]

The coefficient of $\vx_{t_i}$ depends only on the fixed forward schedule and therefore contains no learned parameters.
Hence, when taking the difference between the learned policy and the reference policy, the $\vx_{t_i}$ terms cancel exactly, giving
\[
\vv_{t_i}^\theta - \vv_{t_i}^\text{ref}
=
\left(
\frac{\dot a_{t_i} b_{t_i}^2 - a_{t_i}\dot b_{t_i} b_{t_i}}{a_{t_i}}
\right)
\bigl(\vs_{t_i}^\theta - \vs_{t_i}^\text{ref}\bigr).
\]
Combining this with
\[
\vv_{t_i}^\theta - \vv_{t_i}^\text{ref}
=
-\frac{1}{\delta(t_i)}
\bigl(\vs_{t_i}^\theta - \vs_{t_i}^\text{ref}\bigr),
\]
we obtain
\[
\delta(t_i)
=
-
\frac{a_{t_i}}
{\dot a_{t_i} b_{t_i}^2 - a_{t_i}\dot b_{t_i} b_{t_i}}.
\]

For Rectified Flow, where $a_{t_i} = 1-t_i$ and $b_{t_i} = t_i$, this reduces to
\begin{equation}
    \delta(t_i) = \frac{1-t_i}{t_i}.
    \label{eq:delta-v-rectifedflow}
\end{equation}

\subsection{\texorpdfstring{Sampler-dependent weights $w$}{Sampler-dependent weights}}
\label{appendix:weights}

Table~\ref{tab:unified-table} shows $h(t_i)$, which summarizes timestep-wise optimization strength.
For lookahead methods, a convenient bridge to this quantity is the sampler-dependent weight
\begin{equation}
\label{eq:pcpo_w_identity}
w(t_i):=\frac{\Omega(t_i)\delta(t_i)}{\sigma_{t_i}}
\quad\Longleftrightarrow\quad
\Omega(t_i)\delta(t_i)=w(t_i)\sigma_{t_i}.
\end{equation}
Substituting Eq.~\eqref{eq:pcpo_w_identity} into the lookahead form of \(h(t_i)\) yields, for $C_1(t_i) = \frac{\alpha}{2}\frac{\Omega(t_i)^2}{\sigma_{t_i}^2}$,\footnote{Residual $\nabla$-DB is slightly different: since \(C_1(t_i)=\frac{1}{d}\frac{\Omega(t_i)^2}{\sigma_{t_i}^4}\), the same substitution yields \(h(t_i)\propto w(t_i)\sigma_{t_i}^{-1}\) rather than \(w(t_i)\sigma_{t_i}\). The role of \(w(t_i)\) as a useful sampler-level lens remains analogous.}
\begin{equation}
h(t_i)=\delta(t_i)C_1(t_i)\gamma(t_i)\frac{\sigma_{t_i}^2}{\alpha\Omega(t_i)}=\frac{\gamma(t_i)}{2}\delta(t_i)\Omega(t_i)=\frac{\gamma(t_i)}{2}w(t_i)\sigma_{t_i}\propto w(t_i)\sigma_{t_i}.
\end{equation}
Thus, \(w(t_i)\) provides a convenient sampler-level lens for understanding how different design choices reschedule effective optimization strength across timesteps.
See Figure~2 of \citet{lee2025pcpo} for the shapes of \(w(t_i)\) under different samplers.

We now use PCPO-reweight as a case study.
Its key modification is to replace the native schedule induced by \(w(t_i)\) with a flatter alternative \(w'(t_i)\), thereby reshaping \(h(t_i)\).

\textbf{DDIM sampler, \(\veps\)-prediction (SD1.5).}
PCPO defines
\begin{equation}
w(t_i)=\frac{\frac{\sqrt{1-\bar\alpha_{t_i}}}{\sqrt{\alpha_{t_i}}}-\sqrt{1-\bar\alpha_{t_{i-1}}-\sigma_{t_i}^2}}{\sigma_{t_i}},
\end{equation}
which satisfies Eq.~\eqref{eq:pcpo_w_identity} together with \(\Omega\) from Eq.~\eqref{eq:Omega-ddim} and \(\delta\) from Eq.~\eqref{eq:delta-eps-ddim}.
PCPO then modifies the noise schedule from \(\sigma_{t_i}\) to \(\sigma'_{t_i}\) so that the resulting weights become approximately constant, \(w'(t_i)\approx w_{\mathrm{const}}\).
Accordingly, the induced schedule changes from \(\frac{1}{2}\sigma_{t_i}w(t_i)\) to \(\frac{1}{2}\sigma'_{t_i}w'(t_i)\).
To avoid confounding this reweighting with a global change in loss scale, PCPO normalizes the modified schedule so that \(\mathrm{avg}(w')=\mathrm{avg}(w)\).

\textbf{Euler discrete sampler, \(\vv\)-prediction (SD3).}
Here PCPO keeps the sampler noise schedule fixed and instead reweights the native timestep weights directly.
Writing \(\sigma_{t_i}=\tilde\sigma_{t_i}\sqrt{\Delta t_i}\), where \(\tilde\sigma_{t_i}\) is the instantaneous diffusion coefficient in the It\^o SDE, the standard weight is
\begin{equation}
w(t_i)=\frac{\sqrt{\Delta t_i}}{\tilde\sigma_{t_i}}\left(1+\frac{(1-t_i)\tilde\sigma_{t_i}^2}{2t_i}\right)^2,
\end{equation}
which again satisfies Eq.~\eqref{eq:pcpo_w_identity} with \(\Omega\) from Eq.~\eqref{eq:Omega-EulerDiscrete} and \(\delta\) from Eq.~\eqref{eq:delta-v-rectifedflow}.
PCPO replaces this native schedule by
\(
w'(t_i)=w_{\mathrm{const}}\Delta t_i,
\)
so the induced weighting changes from \(\frac{1}{2}\sigma_{t_i}w(t_i)\) to \(\frac{1}{2}\sigma_{t_i}w'(t_i)\).
As in the DDIM case, PCPO matches the average scale by enforcing \(\mathrm{avg}(w')=\mathrm{avg}(w)\).

\section{Experiment Details}
\label{appendix:experiment-details}

\begin{table*}[t]
\centering
\small
\caption{\textbf{Summary of experiment settings for Section~\ref{subsec:end-to-end-design}.}}
\label{tab:experiment-master}
\resizebox{\linewidth}{!}{
\begin{tabular}{llllll}
\toprule
\textbf{Experiment setting} & \textbf{Baseline} & \textbf{Base model} & \textbf{Reward model} & \textbf{Prompts} & \textbf{Max epoch} \\
\midrule
Mechanistic Analysis
& TempFlow-GRPO
& SD3.5-M
& Aesthetic Score
& Pick-a-Pic
& 150 \\

Validation: ZO, Flow
& TempFlow-GRPO
& SD3.5-M
& GenEval
& GenEval
& 420 \\

Validation: ZO, Diffusion
& PCPO
& SD1.5
& HPSv2.1
& HPDv2
& 300 \\

Validation: FO, Flow
& VGG-Flow
& SD3.5-M
& HPSv2.1
& GenEval
& 200 \\

Validation: FO, Diffusion
& Residual $\nabla$-DB
& SD1.5
& HPSv2.1
& GenEval
& 100 \\
\bottomrule
\end{tabular}
}
\end{table*}

\subsection{Toy experiments.}
\label{appendix:toy-experiment-details}
To conduct the experiments in Section~\ref{subsec:isolated-estimator}, we designed a base data distribution on 2 dimensional space $p_0^{\text{ref}}(\vx_0)$ as mixture of three Gaussian nodes with unit covariance:

\begin{equation}
p_0^{\text{ref}}(\vx_0):=\frac{1}{3}\biggl\{\mathcal{N}\left(\begin{bmatrix}3\\-\sqrt{3}\end{bmatrix},\mathbf{I}\right) + \mathcal{N}\left(\begin{bmatrix}-3\\-\sqrt{3}\end{bmatrix},\mathbf{I}\right) + \mathcal{N}\left(\begin{bmatrix}0\\2\sqrt{3}\end{bmatrix},\mathbf{I}\right)\biggr\}
\end{equation}

We defined our reward $r(\vx_0):=\vx_0[0]/2+3$ as a linear function, which allows the reward-tilted distribution $p_0^\star(\vx_0)\propto p_0^{\text{ref}}(\vx_0)e^{r(\vx_0)}$ to be another mixture of Gaussians. The noised marginal distributions $p_t^{\text{ref}}(\vx_t)$ and $p_t^\star(\vx_t)$ can be obtained in a closed form as an interpolation between data Gaussians and noise prior $\mathcal{N}(0, \mathbf{I})$. Their score function difference is used as a ground truth for the value gradient according to Eq.~\eqref{eq:optimal-decomposition}.

We trained diffusion models for $p_0^{\text{ref}}(\vx_0)$ and $p_0^\star(\vx_0)$ with $\epsilon$-prediction on 500 timesteps. The training was done with batch size 4096 for 2000 iterations, using Adam optimizer with learning rate $10^{-3}$. To simulate the value gradient estimator for different reward alignment training methods, DDIM sampling was done with timestep interval 10 with Markovian stochasticity. 


\subsection{End-to-end Design.}
\label{appendix:end-to-end-design-details}
Table~\ref{tab:experiment-master} summarizes experiment settings used in Section~\ref{subsec:end-to-end-design}.
Below we record the method-specific choices omitted from the main text.
Unless otherwise specified, we use the default configurations of the corresponding baseline.
Although we used H200 GPUs for some experiments to reduce wall-clock time, all experiments are reproducible on GPUs with 24GB VRAM except for the non-pruned VGG-Flow baseline.

\textbf{Mechanistic analysis (ZO, flow).}
We follow TempFlow-GRPO on SD3.5-M~\cite{esser2024scaling} with Aesthetic Score~\cite{schuhmann2022laion} reward and Pick-a-Pic~\cite{Pick-a-Pic} prompts.
Each epoch uses 16 prompts with 4 seeds per prompt.
For entropy regularization, we use $\alpha=0.01$ and follow the heuristic in the Flow-GRPO repository, which temporally weights the entropy term by $\Delta t_i$.
The baseline branching profile is
$
K=[6,6,6,6,6,6,6,6,6].
$
For the redesigned variants, the Uniform Reallocation profile is
$
K=[7,7,7,7,7,7,6,6,0],
$
and the Equal Variance Reallocation profile is
$
K=[4,5,5,6,7,8,9,10,0].
$
Following the policy-gradient setting, we use the group-normalized advantage $\hat A$ in place of raw reward $r$; see Appendix~\ref{subsec:additional_design}.
Evaluation was performed every 30th epoch on a fixed set of prompts.
Wall-clock time is reported in H200 hours; 2 H200 GPUs were used for experiments.

\textbf{Validation: ZO, flow.}
We follow the batch setting of TempFlow-GRPO, using 48 prompts per epoch with 4 seeds per prompt.
For entropy regularization, we use $\alpha=0.004$ and do \emph{not} apply the Flow-GRPO $\Delta t_i$ heuristic.
Because GenEval is primarily semantic, our redesign does not invest branching budget at high-SNR timesteps; instead, it concentrates compute on the first few low-SNR timesteps, where semantic guidance is most useful.
Concretely, we use the branching profile
$
K=[6,6,8,10,0,0,0,0,0],
$
together with GRPO-Guard temporal reweighting and the fair clipping rule in Eq.~\eqref{eq:fair-clip}.
For non-branching timesteps, we add a KL penalty anchoring to $\vs^\text{ref}_{t_i}$ using ODE trajectories, which avoids additional sampling cost for anchor trajectories.
We use the same $\alpha$ for these ODE anchors.
As above, we use the group-normalized advantage $\hat A$ in place of reward $r$.
Wall-clock time is measured excluding GenEval reward computation.
For the baseline, wall-clock time is estimated from the per-step time required under the TempFlow-GRPO setting.
Our method was evaluated every 60th epoch, while the TempFlow-GRPO baseline reported evaluations every 20th epoch.
Evaluation was performed on the GenEval test prompt dataset, separate from the training set.
Wall-clock time is reported in H200 hours; 4 H200 GPUs were used for experiments.

\textbf{Validation: ZO, diffusion.}
The baseline PCPO can be viewed as using the uniform profile $K_i \equiv 1$, i.e., a single sampled continuation at every reverse step.
Ideally, one would branch at every reverse step, but doing so over the full trajectory ($N=50$) is prohibitively expensive.
We therefore branch at only \(10\%\) of reverse steps, leaving the remaining steps unbranched.
Specifically, for a roughly matched per-epoch compute budget, we use 8 prompts per epoch with 5 seeds per prompt, and set $K_i=4$ at every 10th timestep.
For non-branching timesteps, we add a small $\|\vs^\theta-\vs^{\mathrm{ref}}\|^2$ regularizer on \(10\%\) of ODE rows to prevent uncontrolled drift from the reference model, with $\alpha=8\times 10^{-5}$.
We fine-tune SD1.5~\cite{rombach2022high} with LoRA~\cite{LoRA} rank 16.
Evaluation was performed every 10th epoch on a fixed set of prompts.
As above, we use the group-normalized advantage $\hat A$ in place of reward $r$.
Wall-clock time is reported in RTX4090 hours; 4 RTX4090 GPUs were used for experiments.

\textbf{Validation: FO, flow.}
We follow the batch setting of VGG-Flow, using 32 generation trajectories and one update per epoch.
The baseline relies on local Tweedie-based reward gradients, whereas our redesign uses terminal-image reward gradients $\nabla_{\vx_0} r(\vx_0)$, which effectively correspond to using $j=0$ reward information.
The exact unbiased quantity would be $\nabla_{\vx_{t_i}} r(\vx_0)$, but computing it requires an expensive high-order Jacobian chain.
Following score distillation approaches, we therefore use $\nabla_{\vx_0} r(\vx_0)$ as a practical surrogate.
To counteract inherited downweighting of low-SNR timesteps, we replace the baseline schedule $\tilde{\gamma}_{t_i}=(1-t_i)^2$ with $\tilde{\gamma}_{t_i}=\frac{1-t_i}{2}$.
We plot \emph{transition mean drift}, i.e. $\mathbb{E}[\|\vmu^\theta - \vmu^\text{ref}\|^2]$, which is analogous to the KL divergence defined for lookahead methods, i.e. $\mathbb{E}[\|\vmu^\theta - \vmu^\text{ref}\|^2/2\sigma^2]$.
Evaluation was performed every 50th epoch on the GenEval test dataset.
As we observed more variability across random seeds in this setting, we report results averaged over three runs with different seeds.
Wall-clock time is reported in H200 hours; 4 H200 GPUs were used for experiments.

\textbf{Validation: FO, diffusion.}
Both the baseline and our method collect 256 generation trajectories.
For training, we subsample \(20\%\) of transitions from each collected trajectory by uniformly splitting the trajectory into timestep intervals and sampling one transition from each interval, while always including the final transition, following \citet{liu2025resnabladb}.
The baseline again uses local Tweedie-based reward gradients, while our redesign uses terminal-image reward gradients $\nabla_{\vx_0} r(\vx_0)$ as a practical surrogate for the exact quantity $\nabla_{\vx_{t_i}} r(\vx_0)$.
This effectively corresponds to using \(j=0\) reward information without differentiating through the full Jacobian chain.
To reverse the inherited low-SNR downweighting, we multiply $\ell_{t_i}^\theta$ by $3\sigma_{t_i}$ and set $\tilde{\gamma}_{t_i}=\frac{3}{2}\sigma_{t_i}$.
Evaluation was perfomed every 50th epoch on the GenEval test dataset.
Wall-clock time is reported in RTX4090 hours; 4 RTX4090 GPUs were used for experiments.

\paragraph{Remark on Reward Score Distillation.}
The first-order validation experiments use terminal-image reward gradients $\nabla_{\vx_0} r(\vx_0)$ as a practical surrogate for the exact full-lookahead signal $\nabla_{\vx_{t_i}} r(\vx_0)$. This choice should be understood as a \emph{reward score distillation} approximation, analogous in spirit to Score Distillation Sampling (SDS)~\cite{poole2023dreamfusion}: rather than backpropagating through the full denoising trajectory, we use the terminal reward gradient as a tractable direction for improving the intermediate score update. The exact quantity $\nabla_{\vx_{t_i}} r(\vx_0)$ would be unbiased, but requires differentiating through the full Jacobian chain, which is prohibitively memory-intensive in practical settings and prevents the method from retaining the memory footprint of the original first-order baselines. The recently proposed half-order fine-tuning~\cite{ren2026halforder} provides a more principled alternative by constructing unbiased estimators with lower variance than zeroth-order policy-gradient estimators, at the cost of maintaining a local backward computation graph. Such estimators may offer the best overall bias--variance tradeoff when the additional memory overhead is acceptable. Our goal in this validation experiment is more conservative: to show that a Pareto-improving redesign is possible under essentially the same memory budget as the baseline. Compared with local Tweedie-based guidance, which is severely biased at low-SNR timesteps, replacing the local gradient $\nabla_{\vx_{t_i}} r(\hat{\vx}_0 \mid t_i)$ with terminal-image guidance $\nabla_{\vx_0} r(\vx_0)$ reduces the dominant low-SNR bias while preserving the low variance and low memory cost of the original first-order update. Thus, the experiment should not be interpreted as claiming that reward score distillation is the statistically optimal estimator, but rather as demonstrating that a lightweight SDS-style surrogate can substantially improve the practical first-order design without introducing additional rollout or memory overhead.

\subsection{Additional Experiments}
Figures~\ref{fig:auxiliary_network_useless}--\ref{fig:online-better-resnabdb} were obtained using the original codebases of Residual $\nabla$-DB~\cite{liu2025resnabladb} and VGG-Flow~\cite{liu2025vggflow}, with Aesthetic Score as the reward.

\section{Additional Results}
\label{appendix:additional-results}
\subsection{Ablations for First-Order Experiments}
\label{appendix:extended-ablations}

To isolate the contribution of the modified value-guidance estimator $\textcolor{blue}{\mathbf{\Psi}_{t_i}}$, we compare our full method against ablated variants that retain the baseline first-order estimator based on $\nabla_{\vx_t} r(\hat \vx_0)$ rather than $\nabla_{\vx_0} r(\vx_0)$.
In flow matching, our redesign changes both the estimator and the temporal weighting: it replaces the local Tweedie-based gradient with terminal-image reward gradients, and replaces the baseline weighting $\tilde\gamma(t_i)=(1-t_i)^2$ with a milder low-SNR schedule.
Since the original VGG-Flow paper already observed that a linear schedule can outperform $(1-t_i)^2$ in some settings, Fig.~\ref{fig:fo-ablate-estimator}(a, b) further compares against two linearized baselines to test whether our gains come solely from the weighting change.
\emph{Linear Baseline A} keeps the baseline estimator and uses $\tilde\gamma(t_i)=(1-t_i)$, while \emph{Linear Baseline B} keeps the baseline estimator and uses our scaled linear schedule.
Our full method still performs substantially better, showing that the improvement is not explained by temporal weighting alone, but critically depends on the estimator change itself.
In diffusion, Fig.~\ref{fig:fo-ablate-estimator}(c, d) shows the analogous comparison against a baseline that keeps the original local Tweedie-based estimator under the same modified weighting.
Overall, replacing $\nabla_{\vx_t} r(\hat \vx_0)$ with $\nabla_{\vx_0} r(\vx_0)$ consistently improves both reward efficiency and the reward--KL tradeoff in flow-matching and diffusion settings.

\begin{figure}[!tb]
\centering
\vspace{-1em}
\includegraphics[width=0.9\linewidth]{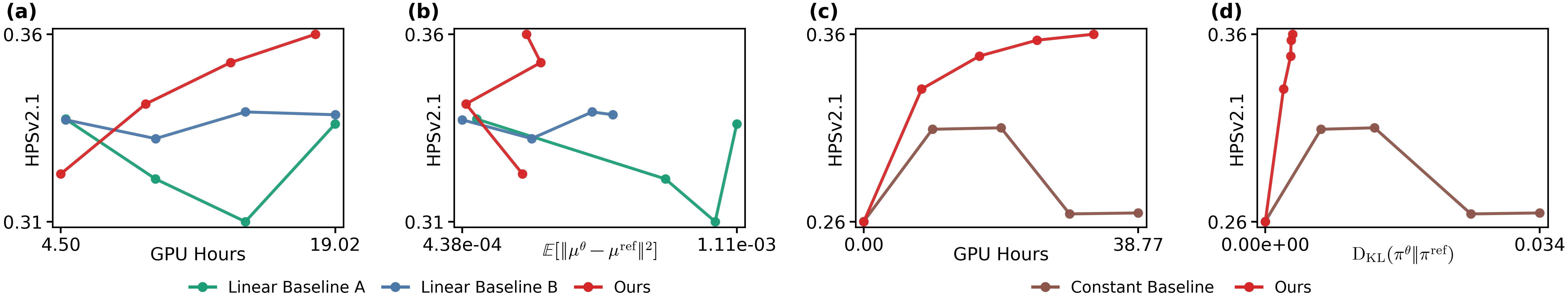}
\caption{\textbf{Ablating the first-order estimator.}
Replacing $\nabla_{\vx_t} r(\hat\vx_0)$ with $\nabla_{\vx_0} r(\vx_0)$ improves both reward efficiency and the reward--KL tradeoff.
In flow matching, we compare against two linearized baselines that keep the original local Tweedie-based estimator but adopt milder temporal weighting:
(a) reward vs.\ GPU hours;
(b) reward vs.\ KL.
In diffusion, we compare against the corresponding baseline with the original estimator under the modified weighting:
(c) reward vs.\ GPU hours;
(d) reward vs.\ KL.}
\label{fig:fo-ablate-estimator}
\vspace{-1em}
\end{figure}

\begin{figure}[!t]
\centering
\includegraphics[width=0.9\linewidth]{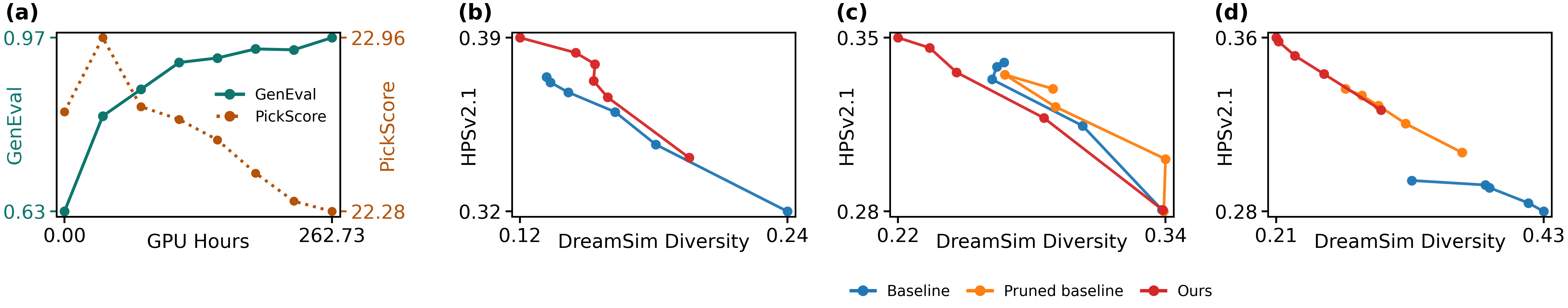}
\caption{\textbf{Auxiliary metrics suggest no obvious reward hacking.}
(a) PickScore during GenEval zeroth-order flow-matching fine-tuning. It improves initially and declines only modestly at later stages, rather than collapsing catastrophically.
(b--d) DreamSim diversity on HPSv2.1 for zeroth-order diffusion, first-order flow matching, and first-order diffusion, respectively.}
\label{fig:not-reward-hacking}
\end{figure}

\begin{table}[!t]
\centering
\small
\caption{\textbf{GenEval subtask breakdown.}
Our method improves performance broadly across subtasks, rather than only within a narrow subset of categories.
$^\dagger$: Results for SD3.5-M and TempFlow-GRPO are taken from Table~1 of \citet{he2025tempflowgrpotimingmattersgrpo}.}
\label{tab:geneval-subtasks}
\resizebox{\linewidth}{!}{
\begin{tabular}{lccccccc}
\toprule
\textbf{Model}
& \textbf{Overall}
& \textbf{Single Obj.}
& \textbf{Two Obj.}
& \textbf{Counting}
& \textbf{Colors}
& \textbf{Position}
& \textbf{Attr. Binding} \\
\midrule
SD3.5-M$^\dagger$
& 0.63
& 0.98
& 0.78
& 0.50
& 0.81
& 0.24
& 0.52 \\
TempFlow-GRPO$^\dagger$
& 0.97
& 1.00
& 1.00
& 0.96
& 0.95
& 0.99
& 0.91 \\
Ours
& 0.97
& 1.00
& 0.99
& 0.98
& 0.93
& 0.98
& 0.87 \\
\bottomrule
\end{tabular}
}
\end{table}

\begin{table}[!t]
\centering
\small
\caption{\textbf{Cross-reward evaluation of GenEval fine-tuning.}
Performance of the SD3.5-M base model and our final GenEval-fine-tuned checkpoint on the GenEval test prompts, measured using GenEval, CLIPScore, and PickScore.}
\label{tab:geneval-cross-reward}
\begin{tabular}{lccc}
\toprule
\textbf{Model}
& \textbf{GenEval}
& \textbf{CLIPScore}
& \textbf{PickScore} \\
\midrule
SD3.5-M
& 0.63
& 0.289
& 22.67 \\
Ours
& 0.97
& 0.316
& 22.28 \\
\bottomrule
\end{tabular}
\end{table}

\subsection{Reward Hacking}
\label{appendix:reward-hacking}

We next examine whether the redesigns introduced in Section~\ref{subsec:end-to-end-design} improve the target reward at the cost of substantially degraded sample quality or diversity.
Overall, we do not observe evidence that our improvements are driven by excessive reward hacking relative to the corresponding baselines.

For the GenEval zeroth-order flow-matching experiments, direct qualitative comparison to the baseline is unavailable because the original baseline checkpoints were not retained.
We report GenEval subtask results in Table~\ref{tab:geneval-subtasks}, which show that our method does not substantially underperform on any individual subtask.
We also report validation on unseen rewards (Table~\ref{tab:geneval-cross-reward}, Figure~\ref{fig:not-reward-hacking}(a)) as evidence against catastrophic model collapse.
The increase in CLIPScore indicates that GenEval improvement is not attributable to reward hacking.
PickScore exhibits a modest degradation at late stages of fine-tuning, consistent with the fact that PickScore captures aesthetics and general human preference in addition to prompt alignment, whereas GenEval favors simple and unambiguous images that are readily recognized by its external detector.

For other settings, we evaluate the tradeoff between reward and diversity by plotting reward--DreamSim Diversity~\cite{fu2023dreamsim} Pareto frontiers. Fig.~\ref{fig:not-reward-hacking}(b)--(d) show that our method is not worse than the baseline on this frontier, indicating that the reward gains do not come at the cost of a systematically worse diversity profile.
Taken together, these results suggest that the improvements from our redesigns reflect better optimization of the intended objective rather than merely exploiting weaknesses of the reward model.

\textbf{Remark on Fig.~\ref{fig:first_order_validation}(b).}
Although the non-pruned VGG-Flow baseline attains a slightly better reward--transition-drift Pareto frontier in Fig.~\ref{fig:first_order_validation}(b), we do not view this as evidence that the auxiliary network $g_\phi$ is essential.
First, $g_\phi$ may induce nontrivial interactions in optimization dynamics that make the unpruned model slightly easier to stabilize, but this modest benefit comes at a substantial systems cost: the non-pruned baseline requires roughly twice the VRAM and significantly longer wall-clock time, making it a poor tradeoff in practice.
In our view, comparable gains would be better pursued through simpler knobs such as a smaller learning rate or a larger effective batch size in the pruned baseline, rather than through an additional auxiliary network.
Second, among all baselines, VGG-Flow exhibited the largest seed-to-seed variability, so advantages at earlier stages of training should be interpreted cautiously.
Third, auxiliary evaluations based on DreamSim Diversity in Fig.~\ref{fig:not-reward-hacking}(d) show that the pruned and unpruned baselines are comparable, with the pruned baseline slightly better if anything.
Taken together, these observations support our main claim: $g_\phi$ is not a core ingredient of effective optimization, but an auxiliary component that adds substantial complexity for at most modest and inconsistent gains.

\section{Auxiliary Designs of Prior Work}
\subsection{Redundancy of Auxiliary Components}
\label{subsec:experiment-invalidate-auxiliary}
Our unified framework posits that the core optimization signal is fully captured by Eq.~\eqref{eq:master_equation}.
This implies that auxiliary components common in prior work—specifically refinement networks and backward losses—are theoretically redundant or even harmful.
We empirically validate that stripping these components simplifies the algorithm and reduces computational overhead without compromising performance; note that our findings are consistent with \citet{liu2025nablar2d3}.

\begin{figure}[!t]
\centering
\includegraphics[width=0.55\linewidth]{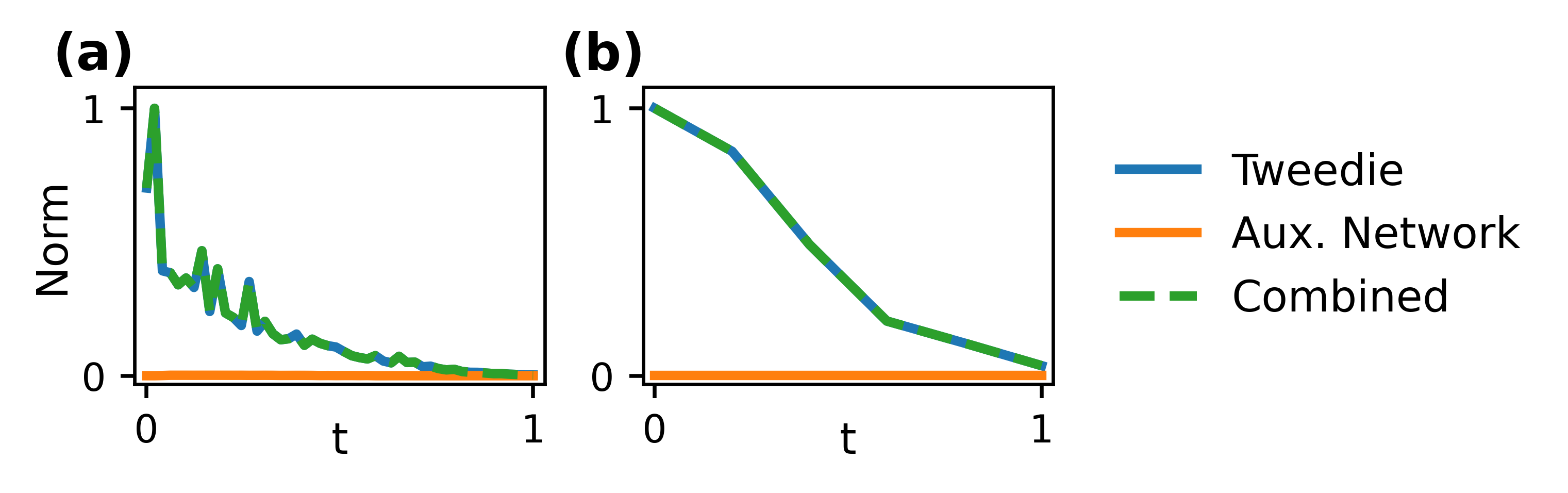}
\caption{\textbf{$\|g_\phi\|$ is negligible.} The learned refinement term $g_\phi$ is negligible compared to the analytic reward gradient throughout the entire generation process for both (a) Residual $\nabla$-DB, (b) VGG-Flow.}
\label{fig:auxiliary_network_useless}
\vspace{-1em}
\end{figure}

\begin{figure}[!t]
\centering
\includegraphics[width=\linewidth]{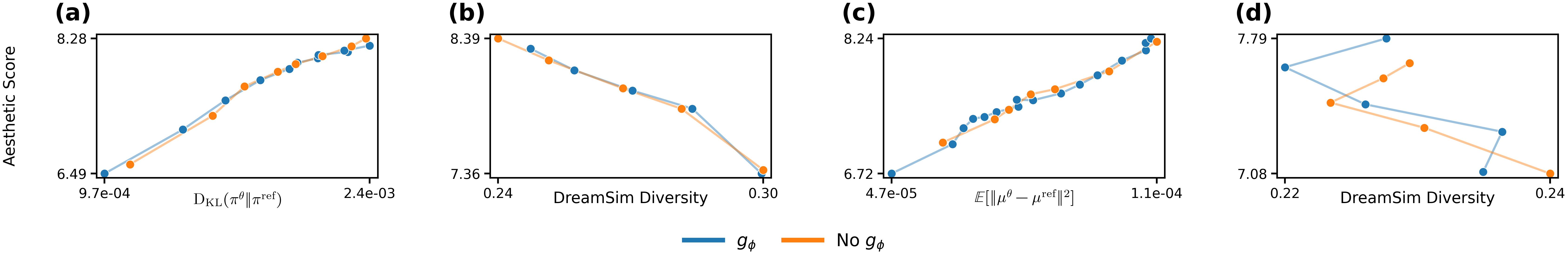}
\caption{\textbf{$g_\phi$ is redundant for training.} Removing $g_\phi$ reduces wall-clock time, while maintaining optimality on the tradeoff between reward and prior/diversity preservation. Averaged across three consecutive random seeds. (a, b) Residual $\nabla$-DB, (c, d) VGG-Flow\protect\footnotemark .}
\label{fig:combined_pareto_analysis}
\end{figure}
\footnotetext{Due to missing checkpoints, panels (c) and (d) are reported from different runs.}

\textbf{\texorpdfstring{Redundancy of Refinement Networks ($g_\phi$).}{Redundancy of Refinement Networks.}}
\label{subsec:refinement_redundant}
Residual $\nabla$-DB and VGG-Flow employ a learnable network $g_\phi$ to compensate for Tweedie approximation errors (i.e., by learning the residual $\nabla V - \nabla r(\hat\vx_0)$).
However, since $g_\phi$ is trained primarily via a terminal boundary loss ($g_\phi(\vx_0) \to 0$) without intermediate supervision, it collapses toward zero due to the spectral bias of neural networks~\citep{rahaman2019spectral}.
Empirical measurements in Fig.~\ref{fig:auxiliary_network_useless} confirm that the learned signal is negligible relative to the Tweedie gradient: $\|g_\phi\| \ll \frac{\tilde\gamma_t}{\alpha} \|\nabla r(\hat \vx_0)\|$.
Consequently, ablations in Fig.~\ref{fig:combined_pareto_analysis} demonstrate that removing $g_\phi$ maintains the Pareto frontier between reward maximization and prior preservation while strictly reducing training time.
This confirms that $g_\phi$ contributes computational overhead without algorithmic benefit.

\textbf{\texorpdfstring{Instability of Backward Loss ($\mathcal{L}_\text{backward}$).}{Instability of Backward Loss.}}
\label{subsec:backward_redundant}
Residual $\nabla$-DB incorporates a backward loss $\mathcal{L}_\text{backward}$ derived from detailed balance conditions.
As detailed in Appendix~\ref{appendix:residual_DB}, this term introduces high-order Jacobian dependencies that are analytically prone to variance explosion.
Results in Fig.~\ref{fig:forward-backward-unetreg} confirm that $\mathcal{L}_\text{backward}$ induces training instability when its weight is non-negligible, whereas the forward gradient alone is sufficient for stable convergence.
This suggests that the backward consistency loss is practically unnecessary, adding volatility rather than value.

\begin{figure}[!t]
\centering
\vspace{-0.5em}
\includegraphics[width=0.65\linewidth]{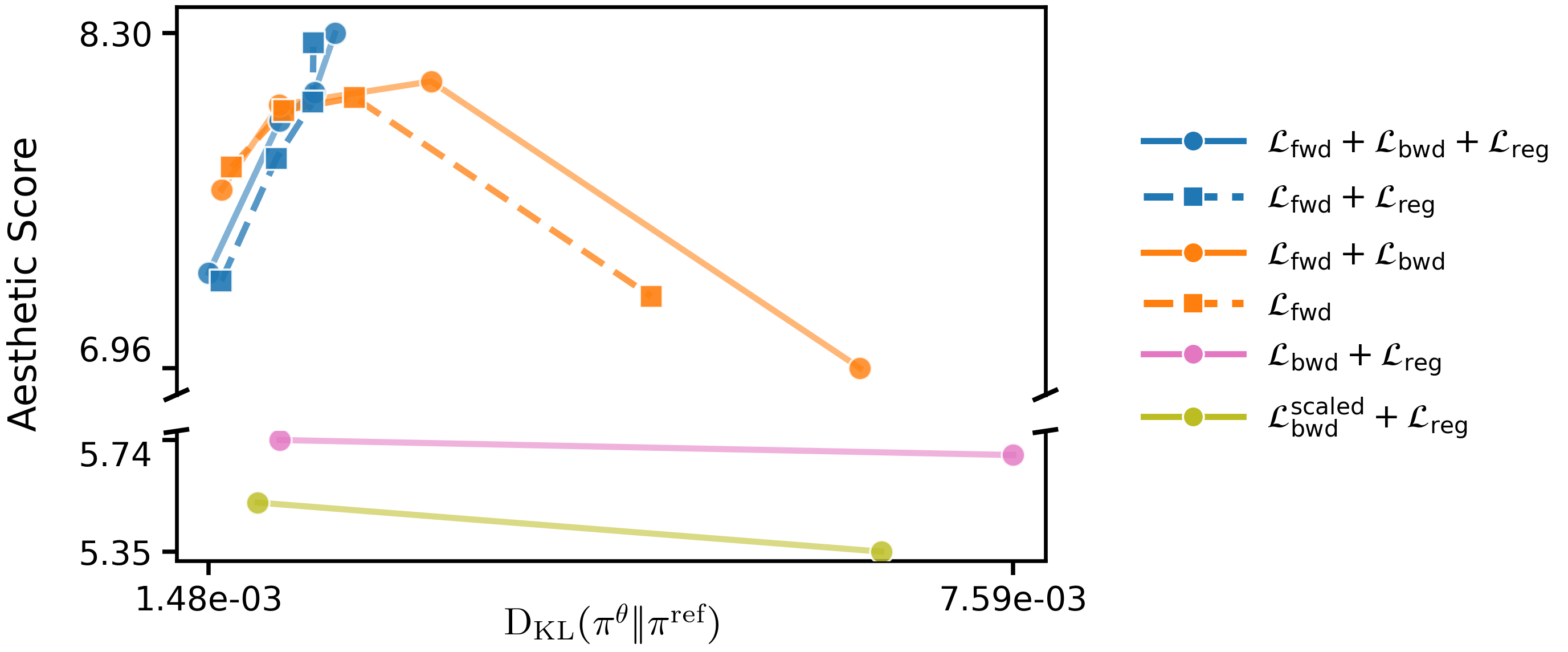}
\caption{\textbf{$\mathcal{L}_\text{backward}$ of Residual $\nabla$-DB does not contribute to effective training.}}
\label{fig:forward-backward-unetreg}
\vspace{-1em}
\end{figure}

\subsection{\texorpdfstring{Revisiting Existing Justifications for Attenuation Coefficients $\tilde\gamma_t$}{Revisiting Existing Justifications for Attenuation Coeffecients}}
Beyond the limitations of value function approximation—which can collapse when the terminal loss drives $g_\phi \approx \mathbf{0}$ (Section~\ref{subsec:refinement_redundant})—the specific attenuation schedules $\tilde\gamma_t$ used in prior work remain only partially justified.

\textbf{Residual $\nabla$-DB.} Motivated by GFlowNet theory, this method uses the forward-looking (FL) trick \cite{pan2023better} to decompose $\nabla_{x_t} \log \tilde F(x_t)$ into a single-step Tweedie estimate and a residual term $g_\phi$. However, this derivation relies on assumptions that are difficult to satisfy for complex image-generation rewards, such as computable intermediate energy or additive energy structure. Even if one interprets the “energy” more broadly as expected future reward, the resulting $\tilde\gamma_t$ schedule still appears heuristic. A simpler interpretation is our RSM perspective; the method attenuates the score-matching objective where the Tweedie estimate is less reliable.

\textbf{VGG-Flow.} VGG-Flow likewise introduces $\tilde\gamma_t$ and decomposes the gradient into a Tweedie term and a learnable correction $g_\phi$, but lacks theoretical justification for the weighting itself. At the same time, the paper explicitly draws inspiration from score distillation \cite{poole2023dreamfusion}, which is compatible with our score-matching interpretation.

\textbf{SQDF.} SQDF, motivated by Soft RL, interprets $\tilde\gamma_t$ as a credit-assignment coefficient, assigning little weight to high-noise regimes under the intuition that these states have lower signal-to-noise ratio and therefore weaker influence on the final sample. However, this intuition is not well aligned with the standard understanding of diffusion models, where low-SNR stages are often responsible for global semantic structure, and early denoising decisions can strongly affect the final reward \cite{choi2022perception}.
As such, monotonically decaying credit assignment may be suboptimal for semantics-aware rewards such as GenEval~\cite{ghosh2023geneval}.
Empirically, their intuition is also inconsistent with the strong performance of Flow-GRPO-Fast~\cite{liu2025flowgrpo}, which trains only on low-to-mid-SNR timesteps, yet outperforms its full-trajectory counterpart.

Overall, these observations suggest that interpreting $\tilde\gamma_t$ primarily as a damping factor for approximation error is more consistent with both prior empirical results and the known behavior of diffusion models.

\subsection{Additional Design Considerations}
\label{subsec:additional_design}
We adopt specific simplifications to isolate the effects of our proposed components.

\textbf{Reward Normalization.}
$\tilde\gamma(t)$ effectively controls the scale of the estimated reward gradient.
Residual $\nabla$-DB and VGG-Flow tune this quantity indirectly, through reward-specific choices of $\alpha$ or $C_2$, to improve stability.
Motivated by the canonical gradient form (Eq.~\eqref{eq:canonical-gradient-form}), we suggest that first-order methods may reduce this hyperparameter sensitivity by explicitly scaling $\tilde\gamma(t)$ with \(\frac{1}{\|\nabla r\|}\).

For zeroth-order methods, the reward-gradient estimator takes the form
$\frac{1}{\sigma_{t_i}} \, \mathbb{E}\!\left[r(\vx_0)\veps_{t_i}\right]$.
This perspective helps clarify why reward normalization can substantially improve optimization.
First, subtracting the group mean acts as a control variate.
Replacing \(r\) with \(r - \hat{\mu}_G\) reduces estimator variance, while also changing the optimization geometry: instead of merely inducing stronger or weaker attraction toward high- and low-reward regions, it produces push--pull dynamics relative to the group average, which is often easier to optimize.
Second, dividing by a scale term helps make the magnitude of the estimator more comparable across reward functions.
In particular, replacing \(r\) with
$
\hat A = \frac{r - \hat{\mu}_G}{\hat{\sigma}_G}
$
makes the scale of \(\mathbb{E}[\hat A \veps_{t_i}]\) roughly comparable across different rewards, and therefore eases hyperparameter tuning.
This same scale-normalization principle also helps explain the empirical success of the Flow-GRPO implementation, which normalizes rewards to be approximately in \([0,1]\).
More broadly, the results of \citet{choi2026rethinking} suggest that dividing by the average magnitude $\|r\|$, or by the sample mean $\|\hat \mu_G\|$ as a practical surrogate, can further stabilize training.

\textbf{Trust Regions.}
Trust-region mechanisms, including the quadratic penalty \(C_2\) and clipping, should be viewed as heuristics rather than essential components.

For example, GRPO-Guard effectively sets \(C_2(t_i)=0\) while remaining stable. This suggests that the RSM objective can sometimes be simplified further without sacrificing performance, although doing so appears to require nontrivial engineering.

A similar picture emerges for clipping. EPG reports that, in some settings, removing clipping does not hurt performance. In contrast, we find that in certain Flow-GRPO settings, naively increasing the clip range \(\xi\) destabilizes training.
Moreover, PCPO observes that standard PPO-style clipping can reduce the effective batch size and thereby harm performance \cite{lee2025pcpo}.
This issue is particularly severe in Flow-SDEs, where the singular behavior near \(t \approx 0\) causes nearly all samples to be clipped, as noted by \citet{wang2025grpoguard}.

Overall, these observations suggest that trust-region design remains an open heuristic space. Further simplification is possible in some regimes, but robust performance across settings likely depends on better-engineered heuristics.

\textbf{Dataset $\mathcal{D}$: Focus on Online Samples.}
While prior work augments training with offline replay buffers~\cite{kang2026sqdf,liu2025resnabladb}, empirical evidence suggests this offers minimal benefit.
SQDF reports negligible difference in performance, and our ablations with Residual $\nabla$-DB (Fig.~\ref{fig:online-better-resnabdb}) show that including past rollouts can actually degrade the Pareto frontier.
Thus, we train exclusively on online samples.

\begin{figure}[!t]
\centering
\vspace{-1em}
\includegraphics[width=0.35\linewidth]{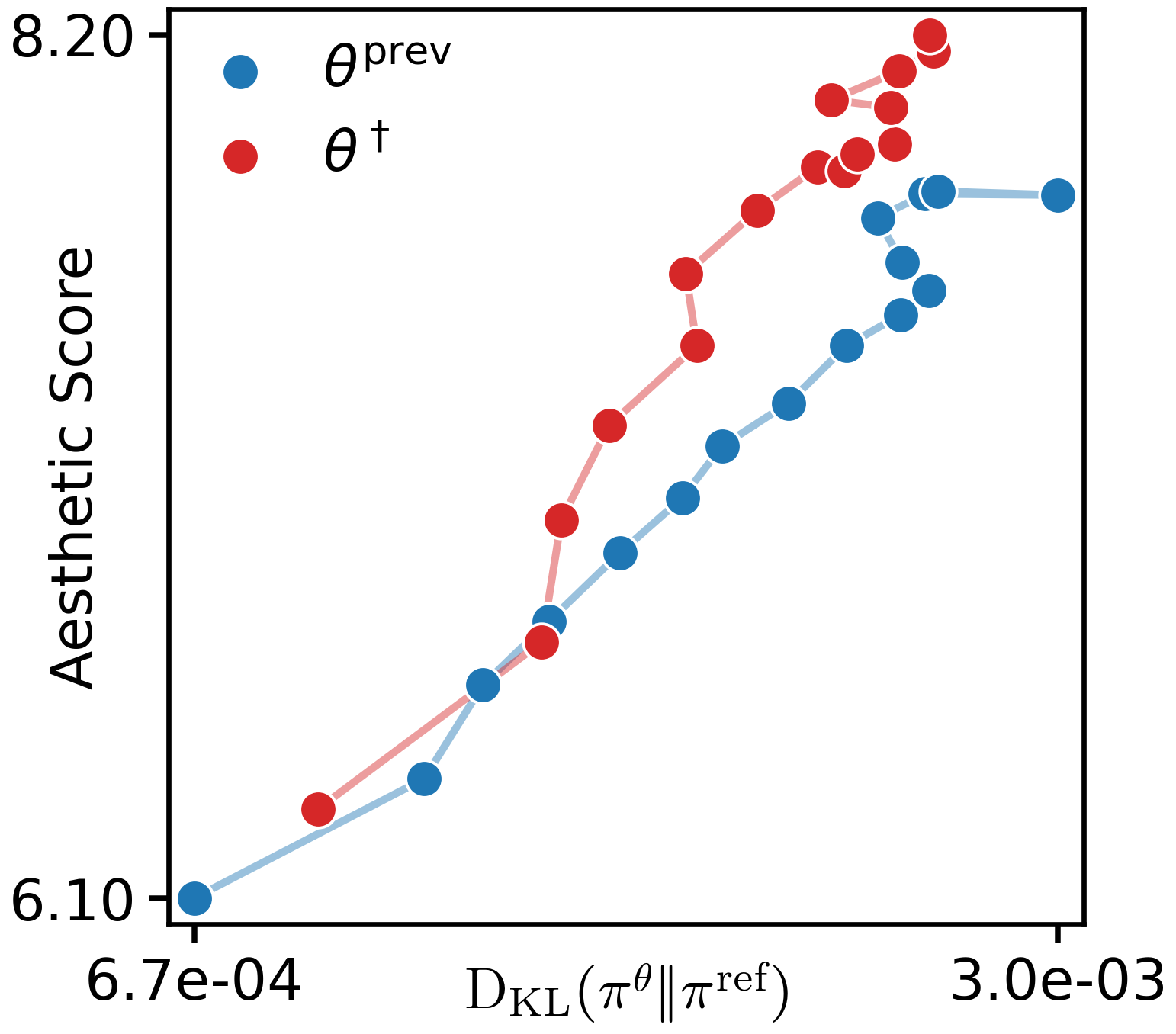}
\caption{\textbf{Online samples suffice.} Including past rollouts (offline buffer) does not improve the Pareto frontier for Residual $\nabla$-DB.}
\vspace{-1.5em}
\label{fig:online-better-resnabdb}
\end{figure}

\textbf{Incorporation of CFG in Training Loss.}
Although we use CFG in all experiments, it is not a necessary component.
Since CFG~\cite{ho2021classifierfree} amplifies the variance of score estimates, it can destabilize training~\cite{choi2026rethinking, xue2025dancegrpo, liu2025flowgrpo}.
In some settings, it may therefore be preferable to exclude CFG from the training loss.

\section{Extension of RSM Framework}
\subsection{Connection with Entropy-regularized Optimal Control Perspective}
\label{appendix:optimal_control}

We provide a summary of the main results from prior work~\cite{quer2024connecting, uehara2024fine}.
%
Consider the soft RL objective defined over path measures:
\begin{align}
 V(\tau) = \max_{u, \nu}\left[ \mathbb{E}_{\sP_{u,\nu}}[r(\vx_0)]- \alpha \mathcal{D}_\text{KL}(\sP_{u,\nu} (\tau) \| \sP_{\text{ref}} (\tau))\right].
\end{align}
Here, $\sP_{\text{ref}}$ denotes the path measure induced by the pretrained diffusion model, while $\sP_{u, \nu}$ denotes the path measure induced by a fine-tuned diffusion model and initial distribution $\nu$.
Note that fine-tuning a diffusion model can be interpreted as modifying the reference dynamics through a drift perturbation.
Such a perturbation admits a natural control interpretation, which enables reformulation of the soft RL objective as an entropy-regularized optimal control problem in continuous time.

To formalize this view, consider a reverse-time SDE with control $u:\sR^d\times\sR\rightarrow\sR^d$,
\begin{align}
    d\vx_t = [f(\vx_t, t)+ u(\vx_t, t)]dt + g(t)dw_t,
    \label{eq:sde_control}
\end{align}
where $f(\vx_t, t)$ denotes the drift coefficient including score function and $g(t)$ is an invertible diffusion coefficient.
Our goal is to characterize the optimal control $u$ that solves the aforementioned entropy-regularized problem.
To obtain an explicit expression of this path-space KL divergence, we invoke Girsanov's theorem.

\textbf{Girsanov's theorem.} Consider the diffusion process with path measure $\sP_{\text{ref}}$,
\begin{align}
    dX_t = f(X_t, t) dt + g(t)dW_t,  \label{eq:sde_base}
\end{align}
where $f(X_t,t)$ denotes the drift coefficient, $g(t)$ is an invertible diffusion coefficient, and $W_t$ denotes the Brownian motion under $\sP_{\text{ref}}$. Let $u(X_t,t)$ be an adapted process satisfying Novikov's condition.
Define the stochastic exponential
\begin{align}
    Y_T=\exp{\left(\int_0^T (g(t)^{-1}u(X_t,t))^\top dW_t - \frac{1}{2}\int_0^T \|g(t)^{-1} u(X_t,t)\|^2 dt \right)}.
\end{align}
As $Y_T$ is a positive martingale with unit expectation, it defines a new probability measure $\sP_u$ via
\begin{equation}
    \frac{d\sP_u}{d\sP_{\text{ref}}}=Y_T.
\end{equation}
Then, the process
\begin{align}
    W_t^u = W_t - \int_0^t g(s)^{-1} u(X_s, s) ds
    \label{eq:new_brownian}
\end{align}
is a Brownian motion under $\sP_u$, and the dynamics of $X_t$ become
\begin{align}
    dX_t = [f(X_t,t) + u(X_t,t)]dt + g(t)dW_t^u
\end{align}
Now, from the definition of path-space KL divergence, we get
\begin{align}
    \gD_{\text{KL}}(\sP_{u,\nu} || \sP_{\text{ref}} ) = \mathbb{E}_{\sP_{u, \nu}} \left[\log \frac{\sP_{u, \nu}}{\sP_{\text{ref}}}\right] = \mathbb{E}_{\sP_{u, \nu}} \left[ \log \frac{\nu}{\nu_{\text{ref}}} + \log \frac{\sP_{u, \nu}(\cdot|\vx_T)}{\sP_{\text{ref}}(\cdot|\vx_T)}  \right]
    \label{eq:decomposed_path_kl}
\end{align}
where the second equality follows from the factorization of the path measure into the initial distribution $\nu_{\text{ref}}=p_T^{\text{ref}}$ and the conditional dynamics. The first term corresponds to the discrepancy in the initial distributions, and the second term measures the deviation of the controlled dynamics from the reference dynamics conditioned on the initial state.
With fixed initial sample $\vx_T$, by the Girsanov theorem, we obtain
\begin{align}
    \mathbb{E}_{u, \nu} \left[\log \frac{\sP_{u, \nu}(\cdot|\vx_T)}{\sP_{\text{ref}}(\cdot|\vx_T)}\right] &=  \mathbb{E}_{u, \nu} \left[\int_0^T \frac{u(\vx_t, t)^\top}{g(t)}d\vw_t - \frac{1}{2} \int_0^T \|\frac{u(\vx_t, t)}{g(t)}\|^2 dt\right]\nonumber\\
    &= \mathbb{E}_{u, \nu} \left[\int_0^T \frac{u(\vx_t, t)^\top}{g(t)}\left(d\vw^u+\frac{u(\vx_t,t)}{g(t)} dt\right) - \frac{1}{2} \int_0^T \|\frac{u(\vx_t, t)}{g(t)}\|^2 dt\right]\nonumber\\
    &=\mathbb{E}_{u, \nu} \left[\frac{1}{2} \int_0^T \|\frac{u(\vx_t, t)}{g(t)}\|^2 dt\right],
\end{align}
where the second equality holds due to Eq.~\eqref{eq:new_brownian} and the third equality holds as $\vw_t^u$ is $\sP_{u, \nu}$-martingale.
In consequence, the soft RL objective function is rewritten as
\begin{align}
    \max_{u, \nu} \mathbb{E}_{\sP_{u, \nu}} [r(\vx_0)] -\alpha \mathbb{E}_{\sP_{u, \nu}}\left[\frac{1}{2}\int_0^T \|\frac{u(\vx_t, t)}{g(t)}\|^2 dt + \log \frac{\nu}{\nu_{\text{ref}}}\right].
\end{align}
Under the perspective of optimal control problem, this is equivalent to
\begin{align}
    \min_{u, \nu} \frac{\alpha}{2} \mathbb{E}_{\sP_{u, \nu}} \underbrace{\left[ \int_0^T \|\frac{u(\vx_t, t)}{g(t)}\|^2 dt \right]}_{\text{running cost}} + \mathbb{E}_{\sP_{u, \nu}} \underbrace{[-r(\vx_0)]}_{\text{terminal cost}} + \mathbb{E}_{\sP_{u, \nu}} \underbrace{\left[\log \frac{\nu}{\nu_{\text{ref}}} \right]}_{\text{initial cost}}.
    \label{eq:oc_softrl}
\end{align}
The initial cost appears because the initial distribution is treated as an optimization variable, which is not standard in classical stochastic optimal control formulation.

\textbf{Value function and Optimal control.} \citet{uehara2024fine} has shown that the marginal density induced by Eq.~\eqref{eq:sde_control} with optimal control and optimal initial distribution obtained from Eq.~\eqref{eq:oc_softrl} is 
\begin{align}
    p_t^\star(\vx_t) = \frac{1}{Z_{\text{tar}}} p_t^{\text{ref}}(\vx_t) \exp{\left(\frac{V^\star_t(\vx_t)}{\alpha}\right)}
\end{align}
where the optimal value function satisfies 
\begin{align}
    \exp{\left(\frac{V_t^\star(\vx_t)}{\alpha} \right)} = \mathbb{E}_{\vx_0 \sim p^{\text{ref}}(\cdot \mid \vx_t)} \left[ \exp{\left(\frac{r(\vx_0)}{\alpha}\right)} \right]
\end{align}
and the optimal control is defined by
\begin{align}
    u^\star(\vx_t, t) =  \frac{g^2(t)}{\alpha}\nabla_{\vx}V^\star_t(\vx) \big|_{\vx=\vx_t} = g^2(t) \nabla_{\vx} \log \mathbb{E}_{\vx_0 \sim p^{\text{ref}}(\cdot \mid \vx)} \left[ \exp{\left(\frac{r(\vx_0)}{\alpha}\right)}\right] \Big|_{\vx=\vx_t} \propto
    \textcolor{red}{\mathbf{\Psi}^\star_{t}}(\vx_{t}).
\end{align}
We emphasize that fine-tuning a diffusion model can be interpreted as perturbing the drift term via a control input. In this view, the fine-tuning model effectively learns to predict $\vs_t^{\text{ref}} + u^\star$, which corresponds to the target score function in the RSM framework.
The optimal control perspective further provides an alternative route to estimating the value function using classical tools from optimal control theory, such as adjoint sampling.

\subsection{Reward-weighted Regression} \label{appendix:reward-weighted-regression}
Reward-weighted regression (RWR) provides a closely related regression-based perspective on online reward fine-tuning. Rather than regressing toward the value-guided optimal score in Eq.~\eqref{eq:optimal-decomposition}, RWR reweights the regression target of the pretraining loss according to the reward of generated samples. This structure closely resembles reward-weighted maximum likelihood estimation, a classical formulation of Soft RL, but is not exactly equivalent; hence, it is neither an exact Soft RL objective nor an exact instance of RSM.

This perspective encompasses seemingly different online fine-tuning constructions. DiffusionNFT~\cite{zheng2025diffusionnft} can be decomposed into an advantage-weighted regression term toward the forward velocity target together with an old-policy anchor. Meanwhile, an extension of DRaFT-LV~\cite{clark2024draft} yields a first-order form of online RWR by directly differentiating the reward-weighted regression objective.

In particular, the first-order formulation (DRaFT-LV) makes explicit the Jacobian that propagates the reward-weighted regression target through the generative model. As shown in the following subsections, DiffusionNFT can be viewed as a simplified realization of this structure that locally approximates the Jacobian by a scaled identity matrix, clarifying both its connection to first-order reward fine-tuning and its distinction from exact RSM reductions.

\subsubsection{DiffusionNFT}
\label{appendix:temp-diffusionnft}

We next examine DiffusionNFT~\cite{zheng2025diffusionnft} through the lens of our unified formulation.
Although its objective is written as a preference-weighted combination of two squared velocity-matching losses, its effective structure is easier to interpret after rewriting it as a regression objective with an explicit old-policy anchor.
This reveals that DiffusionNFT simultaneously pulls the current velocity field toward the forward target $\vv_t^{\mathrm{fwd}}$ according to the clipped advantage, while regularizing it toward the reference velocity field $\vv_t^{\theta^\dagger}$.

\begin{equation}
\begin{aligned}
\label{eq:diffusionnft:orig}
\mathcal{L}_{\mathrm{NFT}}(\theta)
=
\mathbb{E}_{\vx_0 \sim p^{\theta^\dagger},\;\veps \sim \gN(\mathbf{0},\mathbf{I})}
&[\frac{1+c}{2}\,
\Big\|(1-\beta)\vv_t^{\theta^\dagger} + \beta \vv_t^\theta - \vv_t^{\mathrm{fwd}}\Big\|^2 \\
&+\frac{1-c}{2}\,
\Big\|(1+\beta)\vv_t^{\theta^\dagger} - \beta \vv_t^\theta - \vv_t^{\mathrm{fwd}}\Big\|^2 ],
\end{aligned}
\end{equation}
where $c=\mathrm{clip}_1(\hat A(\vx_0)) \in [-1,1]$ and $\vv_t^{\mathrm{fwd}}=\veps-\vx_0$.
Here, $\theta^\dagger$ denotes the old policy used to sample $\vx_0$, while $\theta$ denotes the current policy being optimized.

Differentiating Eq.~\eqref{eq:diffusionnft:orig} w.r.t.\ $\theta$ and collecting terms yields
\begin{equation}
\label{eq:diffusionnft:grad}
\nabla_\theta \mathcal{L}_{\mathrm{NFT}}(\theta)
=
\mathbb{E}\!\left[
2\beta\Big(\beta \vv_t^\theta + (c-\beta)\vv_t^{\theta^\dagger} - c\,\vv_t^{\mathrm{fwd}}\Big)
\cdot \nabla_\theta \vv_t^\theta
\right].
\end{equation}

Therefore, up to $\theta$-independent constants, Eq.~\eqref{eq:diffusionnft:orig} is equivalent to
\begin{equation}
\label{eq:diffusionnft:target}
\begin{aligned}
    \mathcal{L}(\theta)
&= \mathbb{E} \left[ \beta^2 \left\| \vv_t^\theta - \left( \vv_t^{\theta^\dagger} + \frac{c}{\beta} (\vv_t^\text{fwd} - \vv_t^{\theta^\dagger}) \right) \right\|^2 \right].
\end{aligned}
\end{equation}
This form shows that DiffusionNFT trains $\vv_t^\theta$ toward an advantage-dependent interpolation between the old-policy velocity $\vv_t^{\theta^\dagger}$ and the forward velocity target $\vv_t^{\mathrm{fwd}}$.
To make the two components explicit, we use the identity
\begin{equation}
\|y-((1-\lambda)a+\lambda b)\|^2
=\lambda\|y-b\|^2 +(1-\lambda)\|y-a\|^2+C
\end{equation}
where C is independent of $y$.
Applying the identity to Eq.~\eqref{eq:diffusionnft:target} with $\lambda=\tfrac{c}{\beta}$, we see that up to $\theta$-independent constants,
\begin{equation}
\label{eq:diffusionnft:decomposed}
\mathcal{L}_{\mathrm{NFT}}(\theta)
\;\equiv\;
\mathbb{E}\!\left[
\beta c \big\|\vv_t^\theta-\vv_t^{\mathrm{fwd}}\big\|^2
+
\beta(\beta-c)\big\|\vv_t^\theta-\vv_t^{\theta^\dagger}\big\|^2
\right].
\end{equation}
Thus, DiffusionNFT decomposes into a reward-weighted regression term toward the forward target $\vv_t^{\mathrm{fwd}}=\veps-\vx_0$ and an old-policy anchor toward $\vv_t^{\theta^\dagger}$.
The relative strength of these two terms is controlled by the clipped advantage $c$ and the coefficient $\beta$.

\textbf{Connection to Reward-weighted MLE.}
Eq.~\eqref{eq:diffusionnft:decomposed} shows that DiffusionNFT consists of a reward-weighted regression term toward the forward target and an old-policy anchor.
In parallel, Soft RL admits a reward-weighted MLE formulation, whose objective reduces to a reward-weighted denoising regression loss.
This suggests viewing RWR as a practical approximation to reward-weighted MLE~\cite{uehara2024understandingreinforcementlearningbasedfinetuning}.
From this perspective, DiffusionNFT can be viewed as a practical approximation to Soft RL, and is therefore closely related to RSM. We hypothesize that this approximation may also be the source of its weaker optimization stability. This may explain why DiffusionNFT relies on EMA policy updates in practice, since EMA updates serve as a well-established remedies with favorable convergence guarantees in RL literature~\cite{kakadelangford2002,schulmantrpo2015}.

\subsubsection{DRaFT-LV}
\label{appendix:draft-lv-rwr}

We next show that DRaFT-LV~\cite{clark2024draft} suggests a first-order analogue of online reward-weighted regression. The contrast between DiffusionNFT and DRaFT-LV parallels the distinction between policy gradients and DRaFT.

DRaFT-LV is introduced as a low-variance extension of DRaFT-1.
After generating an image $\vx_0$, it forward-noises $\vx_0$ multiple times and differentiates the reward through a one-step denoising prediction.
In the notation of \citet{clark2024draft}, DRaFT-LV fixes $K=1$, so the reward gradient takes the form
\begin{equation}
\nabla_\theta r(\hat\vx_0)
=
(\nabla_\theta \hat\vx_0)^\top
\nabla_{\hat\vx_0} r(\hat\vx_0)
=
-\frac{\sigma_1}{\alpha_1}
(\nabla_\theta \veps_1^\theta)^\top
\nabla_{\hat\vx_0} r(\hat\vx_0).
\label{eq:draft-lv-original-gradient}
\end{equation}
Multiple independently noised copies of the same generated $\vx_0$ are used to average this gradient and reduce its variance.

However, this original construction is specialized to the final denoising step.
In particular, it evaluates the reward through a one-step Tweedie prediction and contains no explicit reference-model regularization.
We consider a natural extension that retains DRaFT-LV's forward-noising construction while allowing arbitrary timesteps:
\begin{equation}
    \vx_{t_i}
    =
    a_{t_i}\vx_0+b_{t_i}\veps,
    \qquad
    \veps\sim\mathcal N(\mathbf 0,\mathbf I),
    \qquad
    t_i\sim p(t),
\label{eq:draft-lv-generalized-noising}
\end{equation}
where, importantly, $\vx_0$ is \emph{not} detached from the computation graph.
Consequently, the terminal reward gradient can be propagated to the forward-noised state,
\begin{equation}
    \nabla_{\vx_{t_i}} r(\vx_0),
\end{equation}
rather than evaluating the reward gradient at a local Tweedie prediction.
This avoids the target-point approximation underlying local first-order estimators.

Combining this first-order reward direction with reference-model regularization yields the generalized update
\begin{equation}
\begin{aligned}
    g
    =
    2(\nabla_\theta \vv_{t_i}^\theta)^\top
    \left[
        (\vv_{t_i}^\theta-\vv_{t_i}^{\mathrm{ref}})
        +
        \frac{1}{\alpha_{\mathrm{KL}}(t_i)}
        \nabla_{\vx_{t_i}} r(\vx_0)
    \right],
\end{aligned}
\label{eq:draft-lv-generalized-gradient}
\end{equation}
where $\alpha_{\mathrm{KL}}(t_i)$ is a timestep-dependent coefficient controlling the relative strength of reward guidance and reference regularization.
We treat this coefficient as a design choice rather than directly inheriting the potentially ill-conditioned scaling induced by the forward-noising schedule.

Eq.~\eqref{eq:draft-lv-generalized-gradient} is the first-order counterpart of the generalized online-RWR gradient.
Online RWR uses the reward-weighted forward-regression direction
\begin{equation}
    r(\vx_0)
    \big(
        \vv_{t_i}^{\theta}-\vv_{t_i}^{\mathrm{fwd}}
    \big),
\end{equation}
whereas the first-order construction replaces this zeroth-order direction with
\begin{equation}
    -\frac{1}{\alpha_{\mathrm{KL}}(t_i)}
    \nabla_{\vx_{t_i}}r(\vx_0).
\end{equation}
Thus, the zeroth-order comparison Flow-GRPO versus online RWR has a direct first-order analogue: DRaFT estimates reward guidance by differentiating through the reverse sampling trajectory, while the generalized DRaFT-LV construction obtains the reward direction through forward noising of on-policy samples.

This perspective also clarifies the relationship to the common Jacobian approximation used in score-distillation-style methods.
Replacing the exact pullback Jacobian relating $\vx_{t_i}$ and $\vx_0$ by an identity-like approximation,
$\mathbf J\approx\mathbf I$, converts the exact first-order reward direction into a local terminal-image reward-gradient surrogate.
This Jacobian approximation is therefore the first-order analogue of the forward-noising simplification underlying online RWR methods.

\subsection{Inference-Time Alignment}
\label{appendix:inference_time_scaling}

While the Reward Score Matching framework primarily addresses weight updates during training, its core principles are isomorphic to inference-time alignment strategies.
DNO~\cite{tang2025inferencetime} steers diffusion models without parameter fine-tuning by directly optimizing the intermediate noise variables $\veps_t$ to maximize a reward.
To handle non-differentiable rewards, DNO employs the \emph{Hybrid2} method, estimating the gradient of the reward with respect to the noise via zeroth-order perturbations.

We demonstrate that this approach is functionally identical to applying the zeroth-order value guidance estimator to the noise space.
Recall that zeroth-order methods update model parameters $\theta$ using the Monte Carlo estimator:
\begin{equation}
    \frac{\sigma_{t_i}}{\alpha \Omega(t_i)} 
    \mathbb{E}_{\vx_{t_{i-1:0}}\sim p_{t_{i:1}}^\star,\veps_{t_{i:1}} \overset{\mathrm{iid}}{\sim} \mathcal{N}(\mathbf{0}, \mathbf{I})}[r(\vx_0)\veps_{t_i}].
\end{equation}
In parallel, DNO updates the latent variable $\vz$ using the estimator (Eqs.~(20, 21) in \citet{tang2025inferencetime}):
\begin{equation}
\mathbb{E}_{\veps}\Big[{\frac{1}{\sigma}}\big(r(D_\theta(\vz+\sigma\veps)) - r(D_\theta(\vz))\big)\big(D_\theta(\vz+\sigma\veps) - D_\theta(\vz)\big)\Big] \cdot \nabla_\vz D_\theta(\vz),
\end{equation}
where $D_\theta$ denotes the denoising network and $\sigma$ denotes the degree of perturbation (not to be confused with variance schedule $\sigma_{t_i}$).
Note that the update above is approximately equal to (Eqs.~(18, 19) in \citet{tang2025inferencetime}):
\begin{equation}
\mathbb{E}_{\veps}\Big[{\frac{1}{\sigma}}\big(r(\vx + \sigma\veps)-r(\vx) \big) \veps\Big] \cdot \nabla_\vz D_\theta(\vz)
\end{equation}
up to a scaling factor.
This simplified form reveals a striking structural equivalence: both methods rely on a reward-weighted noise term ($\approx r \cdot \veps$) to estimate the descent direction.
This implies that training techniques, such as branching and signal-aware scheduling, can be repurposed to improve inference-time alignment.

\subsection{AWM}
Advantage Weighted Matching (AWM)~\cite{xue2025advantageweightedmatchingaligning} also connects diffusion alignment to score matching. However, its objective is fundamentally different from ours. AWM proposes a new advantage-weighted matching objective that brings RL post-training closer to the pretraining loss. By contrast, our goal is not to reinterpret RL fine-tuning as pretraining, but to show that a broad class of \emph{existing} reward-based fine-tuning methods already admit a common Reward Score Matching (RSM) form.

This distinction is central. AWM introduces a new training direction, whereas our contribution is a unifying analysis of prior methods. In particular, we show that methods derived from seemingly different perspectives reduce to a shared RSM structure, and that their primary differences are algorithmic---most notably, how value guidance is estimated and incorporated. To the best of our knowledge, this is the first work to provide such a unification of existing reward-based fine-tuning methods. Accordingly, AWM is best understood as an orthogonal and complementary direction, not a precursor to the perspective developed in this paper.

\section{Qualitative Results}
\label{appendix:qualitative-results}

This section presents qualitative comparisons between baseline methods and our improved variants from the \textbf{Validation Across Settings} experiments (Section~\ref{subsec:end-to-end-design}). The baselines—Residual $\nabla$-DB, VGG-Flow, PCPO, and TempFlow-GRPO—correspond to Figs.~\ref{fig:qualitative-resnabdb}–\ref{fig:qualitative-tempflowgrpo}, respectively. Additional experimental details are provided in Appendix~\ref{appendix:end-to-end-design-details} and summarized in Table~\ref{tab:experiment-master}.

Examples are not cherry-picked to highlight severe reward hacking in the baselines, nor do we claim that our methods exhibit emergent safeguards against reward hacking. Instead, these results serve to illustrate that the observed reward improvements are not achieved through more aggressive exploitation of the reward function. For Fig.~\ref{fig:qualitative-tempflowgrpo}, we report only our method’s outputs, as the corresponding baseline checkpoints trained on GenEval are not publicly available and were not rerun.

\begin{figure}[!ht]
\centering
\includegraphics[width=\linewidth]{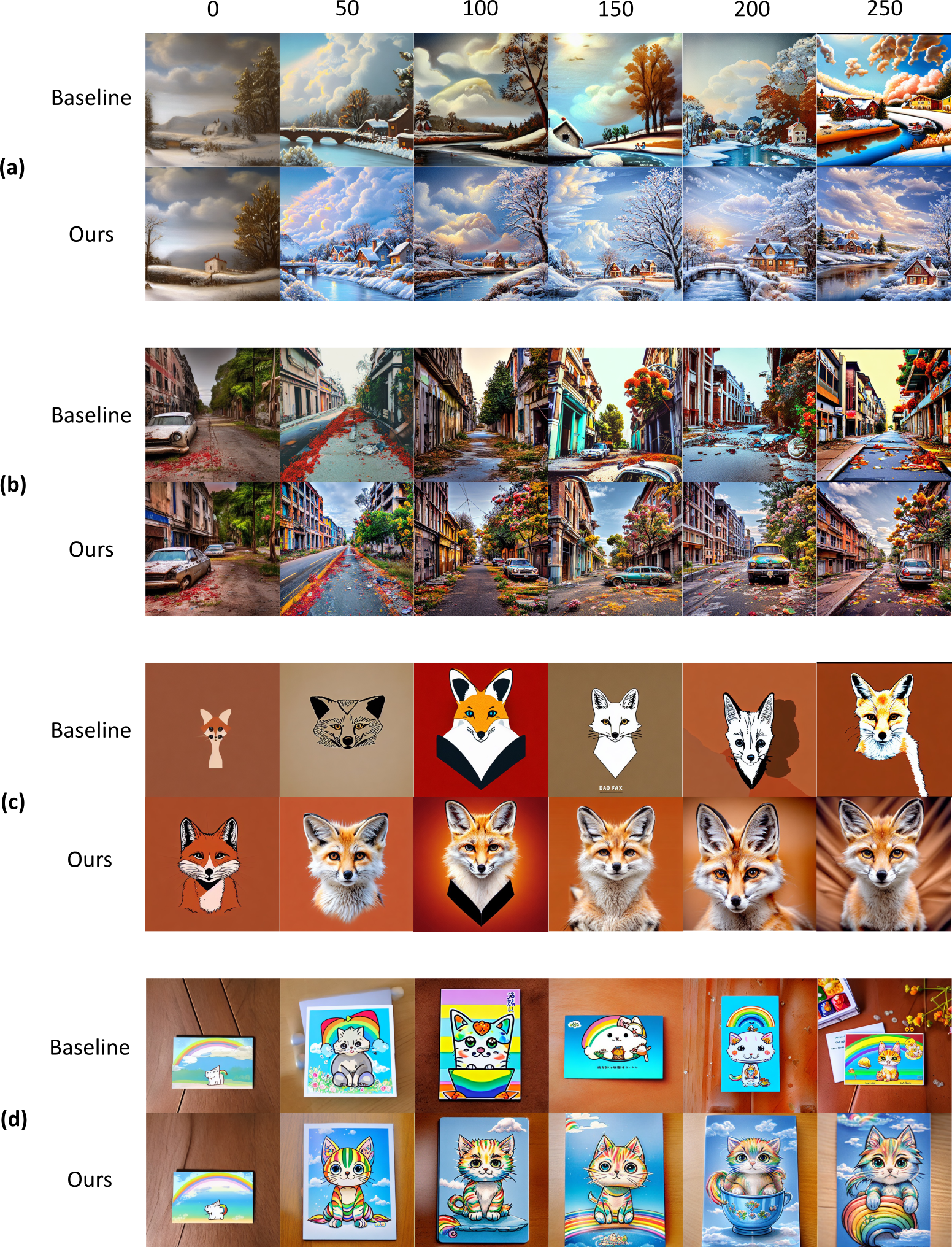}
\caption{\textbf{Qualitative comparisons on the First-order, SD1.5 Validation Setting.} Images are shown at checkpoints after 0, 50, 100, 150, 200, 250 training epochs. Prompts: {\em (a) A painting depicting a snowy winter scene featuring a river, a small house on a hill, and a dreamy cloudy sky; (b) abandoned city with ruined buildings, long deserted streets, cars aged by time, trees, flowers, scattered leaves, empty street, vibrant colors;
(c) portrait of a desert fox, professional, sleek, modern, minimalist, graphic, line art, simple background;
(d) A postcard with cute rainbow kitten in front of blue sky in the chibi-style of Studio Ghibli.}
}
\label{fig:qualitative-resnabdb}
\end{figure}

\begin{figure}[!ht]
\centering
\includegraphics[width=0.95\linewidth]{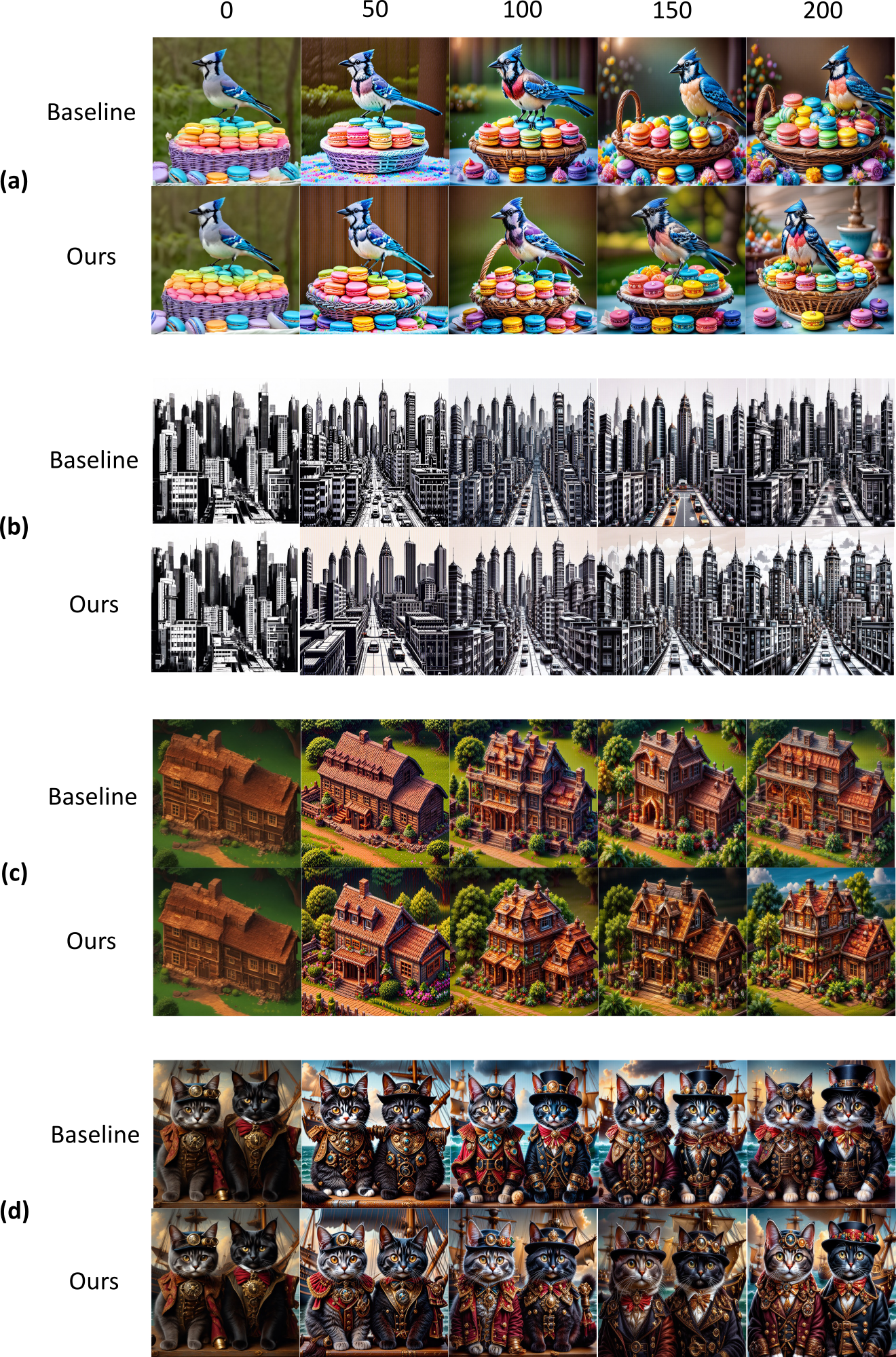}
\caption{\textbf{Qualitative comparisons on the First-order, SD3.5-M Validation Setting.} Images are shown at checkpoints after 0, 50, 100, 150, 200 training epochs. Prompts: {\em (a) A blue jay standing on a large basket of rainbow macarons; (b) an illustration of monochrome cityscape vector graphic;(c) isometric style farmhouse from RPG game, unreal engine, vibrant, beautiful, crisp detailed, ultradetailed, intricate; (d) Two cats, one grey and one black are wearing steampunk attire and standing in front of a ship in a heavily detailed painting.}
}
\label{fig:qualitative-vggflow}
\end{figure}

\begin{figure}[!ht]
\centering
\includegraphics[width=\linewidth]{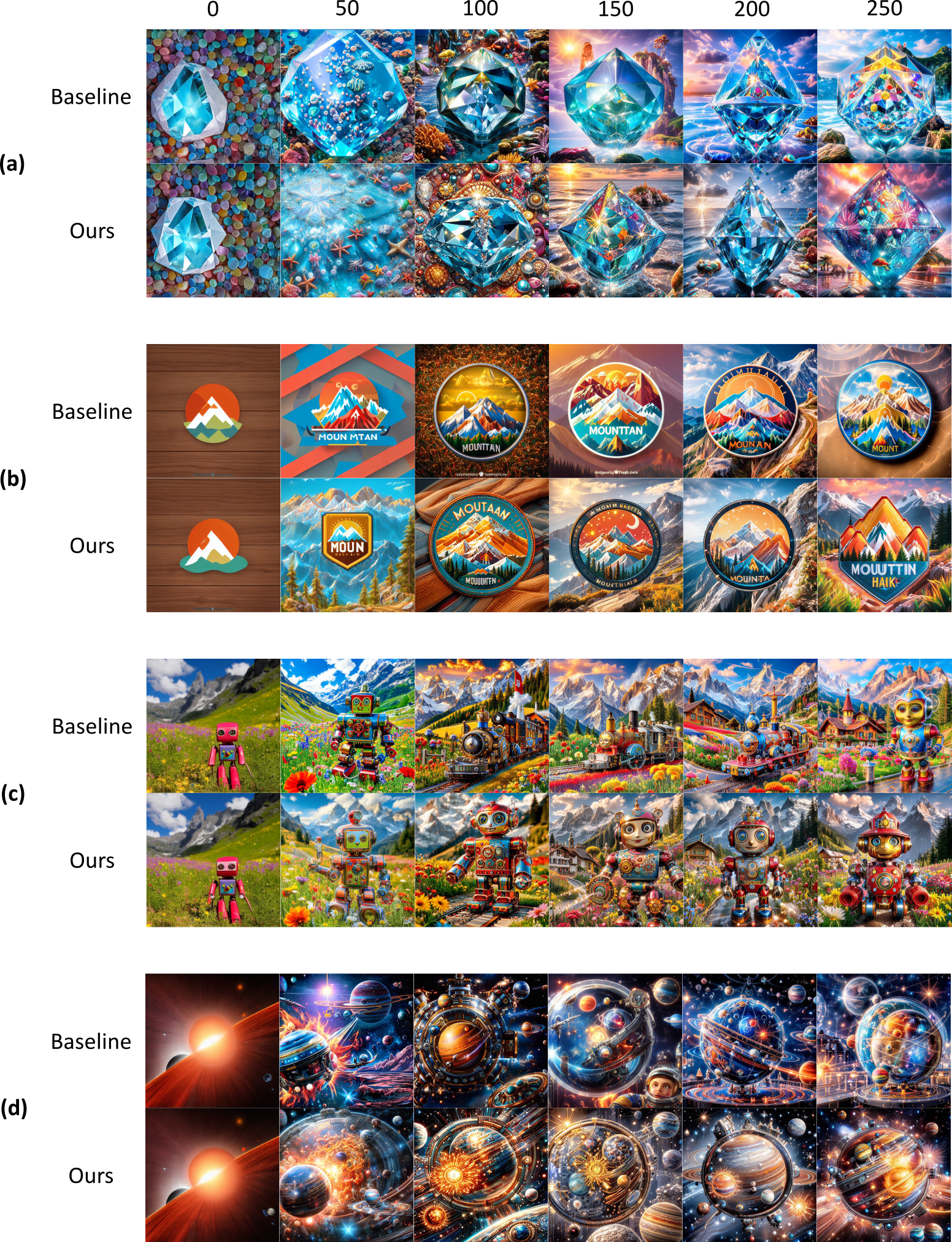}
\caption{\textbf{Qualitative comparisons on the Zeroth-order, SD1.5 Validation Setting.} Images are shown at checkpoints after 0, 50, 100, 150, 200, 250 training epochs. Prompts: {\em (a) A photograph of a giant diamond gem in the ocean, featuring vibrant colors and detailed textures;
(b) logo of mountain, hike, modern, colorful, rounded, 2d concept;
(c) A colorful tin toy robot runs a steam engine on a path near a beautiful flower meadow in the Swiss Alps with a mountain panorama in the background;
(d) A mechanical planet amidst a space with exploding stars.}
}
\label{fig:qualitative-pcpo}
\end{figure}

\begin{figure}[!ht]
\centering
\includegraphics[width=\linewidth]{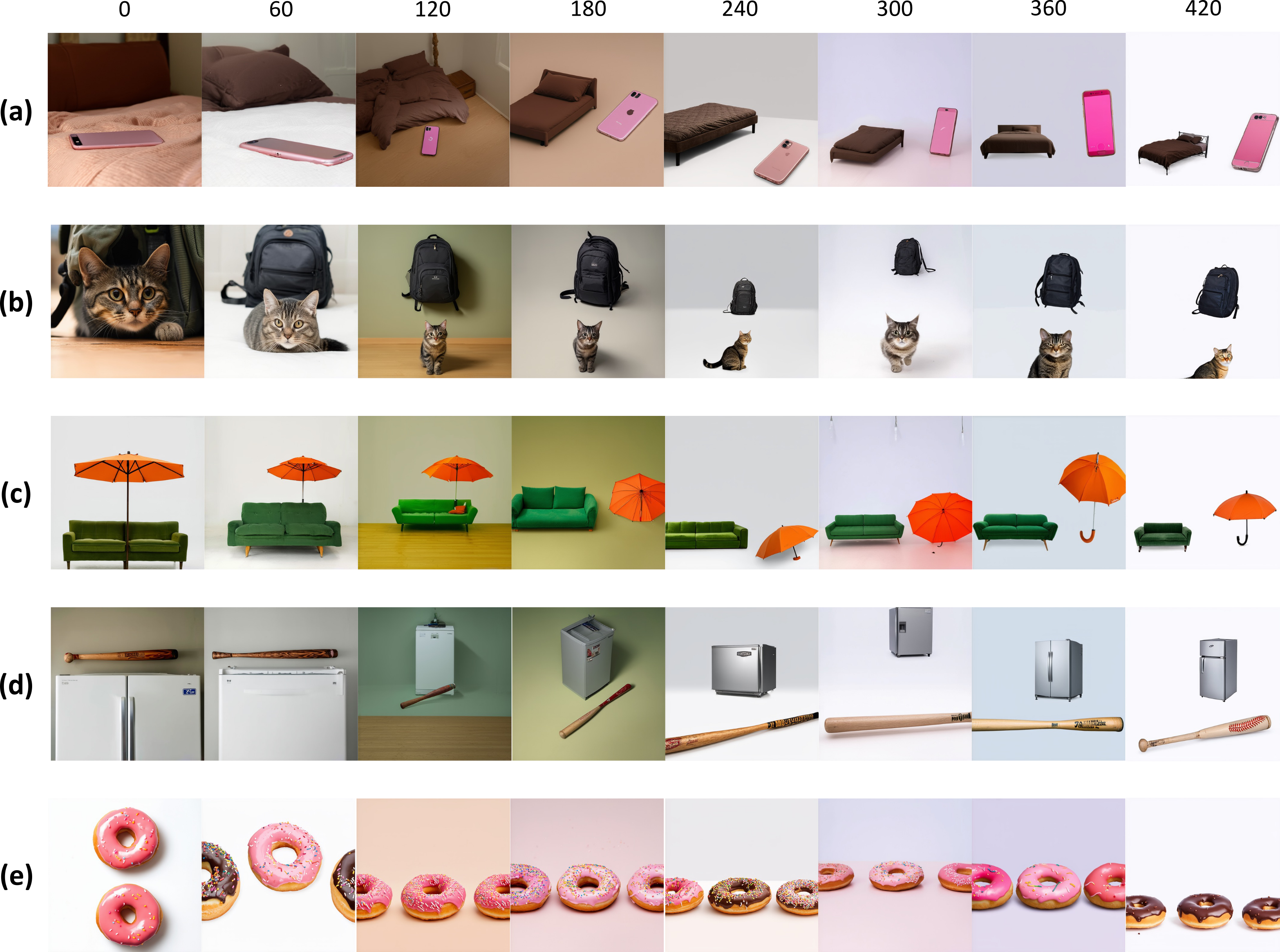}
\caption{\textbf{Qualitative results on the Zeroth-order, SD3.5-M Validation Setting.} Images are shown at checkpoints after 0, 60, 120, 180, 240, 300, 360, 420 training epochs. Prompts:  {\em (a) a photo of a brown bed and a pink cell phone;
(b) a photo of a cat below a backpack;
(c) a photo of a green couch and an orange umbrella;
(d) a photo of a refrigerator above a baseball bat;
(e) a photo of three donuts.}
}
\label{fig:qualitative-tempflowgrpo}
\end{figure}

\end{document}